\documentclass[11 pt]{article}
\usepackage{style}
\usepackage{caption}
\usepackage{subcaption}
\newcommand{\N}{\mathbb{N}}

\newcommand{\E}{\mathbb{E}}
\newcommand{\PR}{\mathbb{P}}
\newcommand{\R}{\mathbb{R}}
\newcommand{\TV}{\mathrm{TV}}

\newcommand{\sS}{\mathcal{S}}
\newcommand{\sA}{\mathcal{A}}
\newcommand{\sZ}{\mathcal{Z}}
\newcommand{\sN}{\mathcal{N}}
\newcommand{\sP}{\mathcal{P}}

\newcommand{\sM}{\mathcal{M}}

\newcommand{\sI}{\mathcal{I}}
\newcommand{\diag}{\textup{diag}}

\newcommand{\khop}{{\color{black}\kappa}}

\newcommand{\nik}{{\color{black}N_i^{\khop}}}

\newcommand{\nikc}{{\color{black}N_i^{\khop_c}}}
\newcommand{\nikg}{{\color{black}N_i^{\khop_G}}}
\newcommand{\agentk}{i_0}
\newcommand{\nkkr}{{\color{black}N_{\agentk}^{\khop_r}}}
\newcommand{\nikr}{{\color{black}N_i^{\khop_r}}}

\newcommand{\nkkc}{{\color{black}N_{\agentk}^{\khop_c}}}
\newcommand{\Mkc}{{\color{black}\sM_{\nkkc}}}
\newcommand{\ukkr}{{\color{black}U_{\agentk}^{\khop_r}}}
\newcommand{\uikr}{{\color{black}U_i^{\khop_r}}}

\newcommand{\supphi}{\psi}
\newcommand{\supepoch}{K}
\newcommand{\supr}{\tilde r}
\newcommand{\regret}{{\color{black}\text{Avg-Nash-Regret}}}

\newcommand{\norm}[1]{\left\lVert#1\right\rVert}
\newcommand{\abs}[1]{\left\lvert#1\right\rvert}
\newcommand{\Sp}[1]{\left(#1\right)}
\newcommand{\Mp}[1]{\left[#1\right]}
\newcommand{\Bp}[1]{\left\{#1\right\}}
\DeclareMathOperator*{\argmax}{arg\,max}

\newcommand{\ri}{(\romannumeral1)}
\newcommand{\rii}{(\romannumeral2)}
\newcommand{\riii}{(\romannumeral3)}
\newcommand{\osi}[1]{\overset{\ri}{#1}}
\newcommand{\osii}[1]{\overset{\rii}{#1}}
\newcommand{\osiii}[1]{\overset{\riii}{#1}}
\title{\LARGE\bfseries Convergence Rates for Localized Actor-Critic in\\ Networked Markov Potential Games}
\author{
	Zhaoyi Zhou\textsuperscript{1}, Zaiwei Chen\textsuperscript{2,$*$}, Yiheng Lin\textsuperscript{2,$\dagger$}, and Adam Wierman\textsuperscript{2,$\ddagger$}\\
	{\textsuperscript{1}Tsinghua University, \href{mailto:zhouzhao20@mails.tsinghua.edu.cn}{\textit{zhouzhao20@mails.tsinghua.edu.cn}}}\\
	{
		\textsuperscript{2}Caltech, \href{mailto:zchen458@caltech.edu}{$^*$\textit{zchen458@caltech.edu}},
		\href{mailto:yihengl@caltech.edu}{$^\dagger$\textit{yihengl@caltech.edu}},
		\href{mailto:adamw@caltech.edu}{$^\ddagger$\textit{adamw@caltech.edu}}
	}\\
}
\date{\vspace{-0.4 in}}
\begin{document}
\maketitle

\begin{abstract}
	We introduce a class of networked Markov potential games in which agents are associated with nodes in a network. Each agent has its own local potential function, and the reward of each agent depends only on the states and actions of the agents within a neighborhood.  In this context, we propose a localized actor-critic algorithm.  The algorithm is scalable since each agent uses only local information and does not need access to the global state.  Further, the algorithm overcomes the curse of dimensionality through the use of function approximation.  Our main results provide finite-sample guarantees up to a localization error and a function approximation error. Specifically, we achieve an $\tilde{\mathcal{O}}(\tilde{\epsilon}^{-4})$ sample complexity measured by the averaged Nash regret. This is the first finite-sample bound for multi-agent competitive games that does not depend on the number of agents.
\end{abstract}

\section{Introduction}\label{sec:intro}
Large-scale systems where agents interact competitively with each other have received significant attention recently, motivated by applications in power systems \citep{shi2022stability}, EV charging \citep{lee2022systems}, and board games \citep{silver2017mastering}, etc. Controlling such systems can be challenging due to the scale of the system, uncertainty about the model, communication constraints, and the interaction between agents. Inspired by the recent success of reinforcement learning (RL), there is an increasing interest in applying RL methods to environments with multi-agent interactions.  However, in multi-agent RL (MARL), the analysis of the system behavior becomes challenging due to the time-varying nature of the environment faced by each agent, which results from the (time-varying) competitive decisions of other agents. As a result, the theoretical analysis of MARL, especially in the competitive setting, is still limited, especially when it comes to large-scale systems. 

The results of MARL in competitive settings to this point have tended to focus on games with a small number of players, e.g., $2$-player zero-sum stochastic games \citep{littman1994markov}, or games with special structure, e.g., Markov potential games (MPGs) \citep{fox2022independent}.  MPGs in particular provide a setting in which the challenges of large-scale systems can be studied. The intuition behind an MPG parallels that of classical (one-shot) potential games. Specifically, the existence of a potential function guarantees that agents can converge to a global equilibrium even when using greedy localized updates. MPGs have wide-ranging applications including variants of congestion games \citep{leo2021convergeMPG,fox2022independent}, medium access control \citep{macua2018learning}, and the stochastic lake game \citep{dechert2006stochastic}. However, the existing theoretical results for MPGs rely on the assumption that a centralized global state exists and can be observed by each individual agent. Such an assumption rules out applications in many large-scale systems including transportation networks \citep{zhang2016control} and social networks \citep{application_chakrabarti2008epidemic}, where the global state space can be exponentially large in the number of agents and/or each agent can only observe its own local state.

A promising approach for the design of local and scalable MARL algorithms in competitive settings is to exploit the network structure of practical applications to design algorithms with sample complexity that only depends on the \emph{ local} properties of the network instead of the \emph{global} state. This approach has recently been successful in the case of cooperative MARL.  For example, \citet{qu2019scalableMARL,lin2021multi,zhang2022global} provides a scalable localized algorithm with a sample complexity that does not depend on the number of agents. However, to this point, local algorithms that exploit network structure do not exist in the competitive MARL setting. Thus, we ask: \emph{Can we design a scalable and local algorithm with finite-time bounds for networked MARL with competitive agents?}

\subsection{Main Contributions}
We address the question above by introducing a class of networked Markov potential games (NMPGs) as the networked counterpart of classical MPGs. Importantly, NMPGs represent a broader class of games than MPGs and draw focus to algorithm design that uses only local information. 

We design a localized actor-critic algorithm that is a combination of independent policy gradient and localized TD$(\lambda)$ with linear function approximation. Notably, our algorithm is \textit{model-free}, uses only \textit{local information}, and  successfully incorporates \textit{function approximation}. This avoids both the need for communication of the global state and the so-called ``curse of dimensionality'' in MARL. 

Our main results provide a finite-sample bound on the averaged Nash regret for our proposed algorithm, which implies an $\tilde{\mathcal{O}}(\tilde{\epsilon}^{-4})$ sample complexity (where $\tilde{\epsilon}$ is the accuracy) up to an approximation error of using local information and a function approximation error. To our knowledge, we are the first to develop a localized algorithm in competitive MARL settings with provable performance guarantees that do not depend on the number of agents. 

Our results are enabled by a novel analysis of the critic in our localized actor-critic framework.  In particular, we propose a localized cost evaluation problem, a new MARL setting to investigate the performance of a local algorithm under a fixed policy. As a critical part of the proof, we propose a novel concept called a ``sub-chain'' that connects local algorithms to their global counterparts, enabling performance bounds via bounds on the gap between the two.

\subsection{Related Work}
\paragraph{Markov Potential Games} Our work adds to the literature on MPGs in MARL. Analytic results for non-cooperative MARL are challenging to obtain because agents learn in a non-stationary environment as other agents update their policies.  As a result, existing analysis has focused on special cases like $2$-player stochastic games \citep{littman1994markov}, adversarial team Markov games \citep{kalogiannis2022efficiently}, and MPGs \citep{fox2022independent}.  The case of MPGs has received considerable attention recently because the potential games are broadly applicable \citep{leo2021convergeMPG} and the existence of potential functions enables provable guarantees \citep{zhang2021gradientStochasticGame,pmlr-v162-ding22b,fox2022independent,zhang2022logBarrierSoftmax}.  While these papers provide algorithms with provable convergence guarantees, they assume that all agents share a common global state and can observe the global state to decide local actions.  An important open question is understanding how to learn in settings where global information is not available.  Our work studies the MARL setting where each agent has its own local state and can only decide local actions based on the local states.

\paragraph{MARL in Networked Systems} The Markov decision process (MDP) model we study is inspired by a series of works on Networked MARL \citep{qu2019scalableMARL,lin2021multi,zhang2022global}, where RL agents are located on a network. In such models, the local state transition of an agent is affected by its own local state/action and its direct neighbors' local states. Networked MARL is applicable to a wide range of applications, including communication networks \citep{application_communication}, social networks \citep{application_chakrabarti2008epidemic}, and traffic networks \citep{zhang2016control}. Compared with general MARL, the additional structure of networked MARL enables us to establish a critical exponential decay property on the local $Q$-functions, which leads to the design of localized  actor-critic algorithms \citep{qu2019scalableMARL,lin2021multi}. All prior works on networked MARL study the case when agents cooperatively maximize the sum of all local rewards. In contrast, our work studies a non-cooperative NMPG in which each agent has its own objective. 

Another approach to study MARL problems is to use mean-field control (MFC) \citep{gu2021mean, mondal2022on, mondal2022can}. The major difference between the mean-field setting and our setting is that mean-field MARL focuses on homogeneous agents, while we allow each agent to have different transition probabilities and local policies.

\paragraph{Finite-Sample Analysis of TD-Learning Variants} TD-learning and its variants are widely used for policy evaluation in RL, which plays a critical role in most policy-space algorithms. The asymptotic analysis of TD-learning dates back to \cite{tsitsiklis1994asynchronous,jaakkola1994convergence}, while finite-sample convergence bounds have received attention in the last decade. In TD-learning, function approximation is a useful technique to reduce the dimension of learning parameters at the cost of incurring an approximation error that depends on the function class. Recently, many breakthroughs are made on finite-sample error bounds for TD-learning with function approximation \citep{bhandari2018finite,Srikant2019FiniteTimeEB,dalal2018finite,yu2009convergence}. Meanwhile, in multi-agent settings, localized TD-learning is crucial for limiting communication and the need for global information \citep{lin2021multi}. Our work provides a novel finite-sample error bound for localized TD-learning with function approximation.

\section{Problem Description}
\label{sec:setting}

\paragraph{Network Structure} We study MARL in the context of networked multi-agent Markov games. Specifically, we consider a setting with $n$ agents that are associated with an undirected graph $\mathcal{G} = (\mathcal{N},\mathcal{E})$, where  $\mathcal{N}=\{1,2,\ldots,n\}$ is the set of nodes and $\mathcal{E}\subseteq \mathcal{N}\times\mathcal{N}$ is the set of edges. We denote by $\text{dist}(i,j)$ the graph distance between agents $i$ and $j$. The local state space and local action space of agent $i$ are denoted by $\mathcal{S}_i$ and $\mathcal{A}_i$, respectively, which are both finite sets. The global state is denoted as $s = (s_1,\ldots,s_n)\in  \mathcal{S}:= \prod_{i=1}^n\mathcal{S}_i$ and the global action is defined similarly.
For any subset $I\subseteq \sN$, we use $s_I$ to denote the joint state of the agents in $I$ and use $\sS_I:=\prod_{i\in I} \sS_i$ to denote the joint state space of agents in $I$. Similarly, we define $a_I$ and $\sA_I$ as the joint action and joint action space of the agents in $I$. Denote $\mu\in \Delta(\sS)$ as the initial state distribution, where $\Delta(\mathcal{S})$ denotes the $|\mathcal{S}|$-dimensional probability simplex.

\paragraph{Transition Probabilities} At time $t\geq 0$, given current state $s(t)$ and action $a(t)$, for each agent $i\in\mathcal{N}$, its successor state $s_i(t+1)$ is independently generated according to the following transition probability, which is only dependent on its neighbors' states and its own action: 
\begin{align*}
	\sP(s(t+1)\,|\, s(t),a(t)) = \prod_{i=1}^n \sP_i(s_i(t+1)\,|\, s_{\mathcal{N}_i}(t),a_i(t)),
\end{align*}
where $\mathcal{N}_i=\{i\}\cup \{j\in\mathcal{N}\mid (i,j)\in\mathcal{E}\}$ denotes the neighborhood of $i$, including $i$ itself. In addition, given an arbitrary integer $\khop\geq 0$, we use $\nik$ to denote the $\khop$-hop neighborhood of $i$, i.e., $\nik=\{i\}\cup \{j\in\mathcal{N}\mid \text{dist}(i,j)\leq \kappa\}$,
and use $-\nik=\sN/\nik$ to denote the set of agents that are not in $\nik$. We use $U_i^{\kappa}=\nik/\{i\}$ to denote the agents in the $\kappa$-hop neighborhood of $i$, excluding $i$ itself.

\begin{remark}
	We require that each agent's transition probability depends only on the states of its neighbors and its own action, which is common in networked MARL literature \citep{qu2019scalableMARL, zhang2022global}. Intuitively, it implies that the impact from far-away agents on the network is ``negligible'', which eventually leads to the exponential decay property (cf. Lemma \ref{le:truncated_Q}).
\end{remark}

\paragraph{Reward Function} Each agent $i\in \sN$ is associated with a deterministic reward function $r_i:\mathcal{S}\times\mathcal{A}\mapsto [0,1]$. The interval $[0, 1]$ is chosen without loss of generality over the set of bounded reward functions. In general, agent $i$'s reward depends on the global state and the global action. Due to the network structure, we assume that there exists a non-negative integer $\kappa_r$ such that the reward function of each agent depends only on the states and the actions of other agents within its $\kappa_r$-hop neighborhood, i.e., $r_i(s,a)=r_i(s_{\mathcal{N}_i^{\kappa_r}},a_{\mathcal{N}_i^{\kappa_r}})$ for all $i$. This makes intuitive sense as we expect the dependence between two agents to weaken as their graph distance grows.

\paragraph{Policy}
In this work, we consider stationary policies \citep{zhang2021multi}. Specifically, each agent $i\in \sN$ is associated with a localized policy $\xi_i:\mathcal{S}_i\mapsto\Delta(\mathcal{A}_i) $. Given a subset $I\subseteq\mathcal{N}$, we define $\xi_I:\mathcal{S}_I\mapsto\Delta(\mathcal{A}_I)$ as the joint policy of agents in $I$. Note that $\xi_I(a_I\mid s_I)=\prod_{i\in I}\xi_i(a_i\mid s_i)$. 
We use $\Xi_i$ to denote agent $i$'s local policy space, and $\Xi_I$ to denote the joint policy space of agents in $I$. When $I=\mathcal{N}$, we omit the subscript and just write $\xi$ for $\xi_\mathcal{N}$ (and $\Xi$ for $\Xi_\mathcal{N}$).
Throughout, we also use $\xi=(\xi_1,\xi_2,\cdots,\xi_n)$ to highlight the local policy components. In this work, we will frequently work with softmax policies, which are defined as
\begin{equation}
	\xi_i^{\theta_i}(a_i|s_i)=\frac{\exp(\theta_{i,s_i,a_i})}{\sum_{a_i'\in\sA_i}\exp(\theta_{i,s_i,a_i'})},\;\forall\;i,s_i,a_i,
\end{equation}
where $\xi_i^{\theta_i}$ stands for agent $i$'s local policy parametrized by the weight vector $\theta_i\in \R^{\abs{\sS_i} \abs{\sA_i}}$.
We denote $\theta=(\theta_1,\theta_2,\cdots,\theta_n)$ as the parameter of a global policy $\xi^\theta$.

\paragraph{Value Function}
Given a global policy $\xi$ and an agent $i$, we define agent $i$'s $Q$-function $Q_i^\xi\in\mathbb{R}^{|\mathcal{S}||\mathcal{A}|}$ as
\begin{align*}
	Q_i^{\xi}(s,a) 
	= \sum_{t=0}^{\infty} \gamma^t\E_{\xi} \left[ r_i(s(t),a(t)) \,\middle|\,s(0)=s,a(0)=a\right]
\end{align*}
for all $(s,a)$, where $\gamma\in (0,1) $ is the discount factor, and $\E_{\xi}[\,\cdot\,]$ is taken w.r.t. the randomness in the (stochastic) policy $\xi$ and the  transition probabilities.  With $Q_i^\xi$ defined above, the averaged $Q$-function $\overline Q_i^{\xi}\in\mathbb{R}^{|\mathcal{S}||\mathcal{A}_i|}$ and the value function $V_i^{\xi}\in\mathbb{R}^{|\mathcal{S}|}$ of agent $i$ are defined as $\overline Q_i^{\xi}(s,a_i)=\E_{a_{-i}\sim \xi_{-i}(\cdot|s_{-i})}[Q_i^{\xi}(s,a_i,a_{-i})]$ for all $(s,a_i)$ and $V_i^{\xi}(s)=\E_{a_i\sim \xi_i(\cdot|s_i)}[\overline Q_i^{\xi}(s,a_i)]$ for all $s$,
where we use $s_{-i}$, $a_{-i}$, and $\xi_{-i}$ to denote the joint state, the joint action, and the joint policy of the agents in $\sN/\{i\}$, respectively.  With the initial state distribution $\mu$, we define  $J_i(\xi)=\E_{s\sim\mu}[V_i^{\xi}(s)]$. Finally, we define the advantage function of agent $i$ as $A_i^{\xi}(s,a)=Q_i^{\xi}(s,a)-V_i^{\xi}(s)$ for all $(s,a)$, and the averaged advantage function of agent $i$ as $\overline A_i^{\xi}(s,a_i)=\overline Q_i^{\xi}(s,a_i)-V_i^{\xi}(s)$ for all $(s,a_i)$. When the policy uses softmax parameterization with parameter $\theta$, we may abuse the policy parameter $\theta$ to represent the policy $\xi$ for simplicity. For example, we may write $J_i(\theta)$ for $J_i(\xi^{\theta})$.

\paragraph{Discounted State Visitation Distribution}
Given a policy $\xi$ and an initial state $s'$, we define the \textit{discounted state visitation distribution} as $d^{\xi}_{s'}(s)=(1-\gamma)\sum_{t=0}^{\infty}\gamma^t {\Pr}^{\xi}[s(t)=s\ |\ s(0)=s']$
for all $s\in\mathcal{S}$,
where $\Pr^{\xi}[s(t)=s\ |\ s(0)=s']$ denotes the probability that $s(t)=s$ given that the initial state is $s'$ and the global policy is $\xi$. We use $d^{\xi}(s):=\E_{s'\sim \mu}[d^{\xi}_{s'}(s)]$ to represent the discounted state visitation distribution when the initial state distribution is $\mu$. 

\section{Networked MPGs}\label{sec:MLPG}
Our focus is a class of networked multi-agent Markov games that we named NMPGs, which is defined in the following.

\begin{definition}
	\label{def:local_MPG} A multi-agent Markov game is called a $\kappa_G$-NMPG (where $\kappa_G$ is a non-negative integer) if
	there exists a set of local potential functions $\{\Phi_i \}_{i\in \sN}$, where $\Phi_i:\Xi\rightarrow \R$ for all $i\in \sN$, such that the following equality holds for any $i\in \sN$, $j\in \mathcal{N}_i^{\kappa_G}$, $\xi_j,\xi_j'\in\Xi_j$, and $\xi_{-j}\in \Xi_{-j}$:
	\begin{equation}
		\label{eq:approx_MPG}
		J_j(\xi_j',\xi_{-j})-J_j({\xi_j,\xi_{-j}}) = \Phi_i(\xi_j',\xi_{-j})-\Phi_i(\xi_j,\xi_{-j}).  
	\end{equation}
\end{definition}

Definition \ref{def:local_MPG} states that when agent $j$ changes its local policy, the change in its objective function $J_j(\cdot,\xi_{-j})$ can be measured by the change of local potential functions from any agent in its $\kappa_G$-hop neighborhood. The non-negative integer $\kappa_G$ is determined by the networked MPG setting and reflects the extent to which the networked MPG is relaxed from an MPG. Recall that in the definition of a standard MPG \citep{leo2021convergeMPG}, there exists a (global) potential function $\Phi$ such that Eq. (\ref{eq:approx_MPG}) holds with $\Phi_i$ being replaced by $\Phi$ for all $i$. Therefore, an MPG is always an NMPG (by choosing $\Phi_i=\Phi$ for all $i$), and hence NMPG represents a strictly broader class of games. More discussions are given in Appendix F.1, and a concrete example of an NMPG is presented in Section \ref{subsec:examples}.  

Due to the boundedness of the reward function and Eq. (\ref{eq:approx_MPG}), the local potential functions are uniformly bounded from above and below, i.e., there exist $\Phi_{\min},\Phi_{\max}>0$ such that
$\Phi_i(\xi)\in [\Phi_{\min},\Phi_{\max}]$ for all $i\in\mathcal{N}$ and $\xi\in\Xi$. See Appendix F.6 for more details.

Unlike in single-agent RL or cooperative MARL, the optimal policy is not well-defined in the competitive setting, and thus our goal is to design algorithms that learn Nash equilibria of NMPGs. We next introduce the concepts of Nash equilibrium, Nash gap, and averaged Nash regret.
\begin{definition}
	A global policy $\xi$ is a Nash equilibrium if $J_i(\xi_i,\xi_{-i})\geq J_i(\xi_i',\xi_{-i})$ for all $\xi_i'\in \Xi_i$ and $i\in\mathcal{N}$.
\end{definition}
To measure the performance of a policy by its ``distance'' to a Nash equilibrium, we use the Nash gap.

\begin{definition}
	\label{def:nash_regret}
	Given a global policy $\xi$, agent $i$'s Nash gap and the global Nash gap are defined as
	\begin{align*}
		\text{NE-Gap}_i(\xi):=\;&\max_{\xi_i'} J_i(\xi_i',\xi_{-i})-J_i(\xi_i,\xi_{-i}),\\
		\text{NE-Gap}(\xi):=\;&\max_{i\in \sN} \text{NE-Gap}_i(\xi).
	\end{align*}
\end{definition}
With $\text{NE-Gap}(\cdot)$ defined above, given $\hat{\epsilon}>0$, we say that a policy $\xi$ is an $\hat{\epsilon}$-approximate Nash equilibrium if $\text{NE-Gap}(\xi)\leq \hat{\epsilon}$. When using a softmax policy with parameter $\theta$, we may abuse the notation to denote $\text{NE-Gap}_i(\theta)$ for $\text{NE-Gap}_i(\xi^\theta)$ and also $\text{NE-Gap}(\theta)$ for $\text{NE-Gap}(\xi^\theta)$.

While Definition \ref{def:nash_regret} enables us to measure the performance of a single policy, in MARL, most algorithms iterate over a sequence of policies. To measure the performance of a sequence of policies, we use the averaged Nash regret, which is defined in the following. 
\begin{definition}\label{def:avg_NR}
	Given a sequence of $M$ policies $\{\xi(0),\xi(1),\ldots,\xi(M-1)\}$, the averaged Nash regret of agent $i$ and the global averaged Nash regret are defined as
	\begin{align*}
		\regret_i(M)&=\frac{1}{M} \sum_{m=0}^{M-1} \text{NE-Gap}_i(\xi(m)),\\
		\regret(M)&=\max_{i\in\sN} \regret_i(M).
	\end{align*}
\end{definition}

Note that a similar concept  called ``Nash Regret'' was previously introduced in \cite{pmlr-v162-ding22b}, and is defined as 
\begin{align}\label{eq:NR}
	\text{Nash-Regret}(M)=\frac{1}{M} \sum_{m=0}^{M-1} \max_{i\in \sN}\text{NE-Gap}_i(\xi(m)).
\end{align}
By using Jensen's inequality and the fact that the maximum of a set of positive real numbers is less than the summation, we easily have $\regret(M)=\Theta (\text{Nash-Regret}(M))$. See Appendix F.4 for the proof. As a result, $\regret(M)$ and $\text{Nash-Regret}(M)$ have the same rate of convergence (up to a multiplicative constant that depends on the number of agents).

\subsection{An Example of NMPGs}\label{subsec:examples}
To illustrate the model, we present an extension of classical congestion games \citep{roughgarden2004bounding} and distributed welfare games \citep{marden2013distributed}. In this example, $n$ agents are located on a traffic network $\mathcal{T} = (\mathcal{V}, \zeta)$, where $\mathcal{V}$ denotes the set of nodes and $\zeta$ denotes the set of directed edges with self-loops\footnote{Note that the traffic network and the communication network $\mathcal{G}$ may be different.}. The objective of each agent $i$ is to commute from its start node $h_i$ to its destination $d_i$. In this example, the local state $s_i(t)$ of agent $i$ at time $t$ is its current location (a node $v \in \mathcal{V}$). By choosing a directed edge $(v, u) \in \zeta$ as its local action $a_i(t)$ at time $t$, agent $i$ will transit to state $s_i(t+1)=u$ at time $t+1$.  Without the loss of generality, we assume an agent will stay at the same node after it arrives at its destination. 

The reward of agent $i$ is defined as $r_i(t)=0$ if $s_i(t) = d_i$, $r_i(t)=-\Bar{\epsilon}$ if $s_i(t+1) = s_i(t)$, and $r_i(t)=-\Bar{\epsilon} - N(a_i(t), t)$ otherwise, where $\Bar{\epsilon} > 0$ is a constant and $N(e, t)$ denote the number of agents that chooses edge $e$ at time $t$. The reward is designed so that the agent incurs a time cost of $\Bar{\epsilon}$ for every step spent on its trip and a congestion cost of $N(a_i(t), t)$ depending on the traffic on the edge it travels through. The congestion cost is avoided if the agent chooses to wait at its current location (i.e., $s_i(t+1) = s_i(t)$). Each agent's goal is to maximize its expected discounted cumulative reward $\mathbb{E}\left[\sum_{t=0}^\infty \gamma^t r_i(t)\right]$.

To see that this congestion game fits in our NMPG framework, consider the following communication network $\mathcal{G}$: agents $i$ and $j$ are neighbors if and only if there exists a global policy $\xi$ such that $\sum_{t = 0}^\infty \Pr(s_i(t) = s_j(t), s_i(t)\not=d_i, s_j(t)\not=d_j) > 0$. Under this communication network, the transition kernel is completely local because the next state of any agent $i$ is decided completely locally and the local reward of agent $i$ is a function that depends on the $1$-hop local states and actions $(s_{\mathcal{N}_i^1}, a_{\mathcal{N}_i^1})$. 
We provide more discussion of this example and numerical simulations using it in Appendix A.

\section{Algorithm Design}\label{sec:alg}

We now present a novel algorithm for solving NMPGs.  Our approach uses a combination of independent policy gradient (IPG) with localized TD-learning to form a localized actor-critic framework.

\subsection{Actor: Independent Policy Gradient}
Suppose that the agents have complete knowledge about the underlying model (e.g., reward function and transition dynamics). Then a popular approach for solving MPGs is to use IPG, which is presented in Algorithm \ref{alg:exact_IPG} \citep{leo2021convergeMPG,zhang2021gradientStochasticGame,pmlr-v162-ding22b,fox2022independent,zhang2022logBarrierSoftmax}. 
\begin{algorithm}
	\caption{Independent Policy Gradient  }
	\label{alg:exact_IPG}
	\begin{algorithmic}[1]
		\STATE \textbf{Input:} Initialization $\theta_i(0)=0$, $\forall\; i\in \sN$.
		\FOR{$m=0,1,2,\cdots,M-1$}
		\STATE  $\theta_i(m+1)=\theta_i(m)+\beta\nabla_{\theta_i} J_i(\theta(m))$ for all $i\in\mathcal{N}$
		\ENDFOR
	\end{algorithmic}
\end{algorithm}

In each round of Algorithm \ref{alg:exact_IPG}, each agent simultaneously updates its policy by implementing gradient ascent (in the policy space) w.r.t. their own objective function (cf. Algorithm \ref{alg:exact_IPG} Line 3). Notably, to carry out Algorthm \ref{alg:exact_IPG}, each agent only needs to know its own policy. While Algorithm \ref{alg:exact_IPG} is promising, it is not a model-free algorithm as computing the gradient requires knowledge of the underlying MDP model.  This motivates the design of a critic to help estimate the gradient.

\subsection{Critic: Localized TD$(\lambda)$ with Linear Function Approximation}\label{subsec:critic_algorithm}
To motivate the design of the critic, we first present an explicit expression of the policy gradient of agent $i$ \citep{sutton1999PGT}:
\begin{align}
	\nabla_{\theta_i}J_i(\theta) =\sum_{t=0}^{\infty}\gamma^t \E_{\xi^{\theta}}\big[\nabla_{\theta_i}\log \xi_i^{\theta_i}(a_i(t)|s_i(t))\overline Q_{i}^{\theta}(s(t),a_i(t))\big]. \label{eq:PGT}
\end{align}
Similar versions of policy gradient theorems under different multi-agent settings were previously developed in \cite{zhang2022logBarrierSoftmax,mao2022Decent-MARL}. For completeness, we present a proof of Eq. (\ref{eq:PGT}) in Appendix F.2.

In view of Eq. (\ref{eq:PGT}), to estimate $\nabla_{\theta_i}J_i(\theta)$, the key is to construct an estimate of the averaged $Q$-function $\overline Q_{i}^{\theta}$. However, directly estimating the averaged $Q$-function of agent $i$ requires information about the global state, incurring long-distance communication. 
To localize the algorithm, we introduce a hyper-parameter $\kappa_c\in \N$, and for each agent, we learn an approximation of the averaged $Q$-function (which we refer to as the $\kappa_c$-truncated averaged $Q$-function) using only information in its $\kappa_c$-hop neighborhood. 

\paragraph{Truncated Averaged $Q$-functions}
Given the non-negative integer $\kappa_c$, agent $i\in \sN$, and a global policy parameter $\theta$, we define $\mathcal Q_i^{\theta,\kappa_c}$ as the class of $\kappa_c$-truncated averaged $Q$-functions w.r.t. $\overline Q_i^{\theta}$. Specifically,
\begin{align*}
	&\mathcal Q_i^{\theta,\kappa_c}=\left\{\overline Q_i^{\theta,\kappa_c} \in\mathbb{R}^{|\mathcal{S}_{\nikc}||\mathcal{A}_i|}\;\middle|\;\exists\, u_i\in \Delta(\mathcal{S}_{-\nikc})\text{ s.t. }\right.\\
	&\left.\overline Q_i^{\theta,\kappa_c}(s_{\nikc},a_{i}) =\mathbb{E}_{s_{-\nikc}\sim u_i}\left[\overline Q_i^{\theta}(s_{\nikc},s_{-\nikc},a_i)\right],\forall\;(s_{\nikc},a_{i})\in \mathcal{S}_{\nikc}\times \mathcal{A}_i\right\}.
\end{align*}
Note that when $\kappa_c \geq \max_{i,j}\text{dist}(i,j)$, there is essentially no truncation, i.e., any element in $\mathcal Q_i^{\theta,\kappa_c}$ is equal to $\overline Q_i^{\theta}$.
When $\kappa_c<\max_{i,j}\text{dist}(i,j)$, we have the following \emph{exponential-decay property}. See Appendix F.3 for the proof.

\begin{lemma}\label{le:truncated_Q}
	For any $\kappa_c\in \N$, agent $i$, and global policy parameter $\theta$, it holds  that 
	{\begin{align}
			\sup_{\overline Q_{i}^{\theta,\kappa_c}\in \mathcal Q_i^{\theta,\kappa_c}}\max_{s,a_i}\abs{\overline Q_{i}^{\theta,\kappa_c}(s_{\nikc},a_{i})-\overline Q_i^{\theta}(s,a_i)}\leq \frac{2\min\left( \gamma^{\kappa_c-\kappa_r+1},1\right)}{1-\gamma} .
	\end{align}}
\end{lemma}
In view of Lemma \ref{le:truncated_Q}, the $\kappa_c$-truncated averaged $Q$-function approximates the averaged $Q$-function (at a geometric rate) as $\kappa_c$ increases. Therefore, it is enough for the critic to estimate an arbitrary $\kappa_c$-truncated averaged $Q$-function within the class $\mathcal Q_i^{\theta,\kappa_c}$. It is worth noting that the use of truncated $Q$-functions and the exponential-decay property have been widely exploited in the cooperative MARL literature
for communication and dimension reduction in recent years \citep{qu2019scalableMARL,gu2022mean,lin2021multi}.  In this work, we show how to use such an approach in a non-cooperative setting for the first time. 

\paragraph{Linear Function Approximation} While using the $\kappa_c$-truncated $Q$-functions enables us to overcome the computational bottleneck as the number of agents increases, there is still the challenge due to the curse of dimensionality. To further reduce the parameter dimension, we use linear function approximation. To be specific, for each $i\in\mathcal{N}$, let $\phi_i: \sS_{\mathcal{N}_i^{\kappa_c}}\times \sA_i\rightarrow \R^{d_i}$ be a feature mapping of agent $i$. Then, with weight vector $w_i\in\mathbb{R}^{d_i}$, we consider approximating the $\kappa$-truncated $Q$-functions using $\hat Q_{i}(s_{\mathcal{N}_i^{\kappa_c}},a_i,w_i)=\langle\phi_i(s_{\mathcal{N}_i^{\kappa_c}},a_i),w_i \rangle$ for all $(s_{\mathcal{N}_i^\kappa},a_i)$. Let $\tilde\phi_i(s,a_i)=\phi_i(s_{\nikc},a_i)$ for any $i\in \sN$, $s\in \sS$, and $a_i\in \sA_i$. That is, given an agent $i$, for each pair $(s,a_i)$ of global state and local action, we look at the states of agents in agent $i$'s $\kappa_c$-hop neighborhood  (i.e., $s_{N_i^{\kappa_c}}$) and agent $i$'s action (i.e., $a_i$) and assign the vector $\phi(s_{N_i^{\kappa_c}}, a_i)$ to $\tilde{\phi}_i(s,a_i)$. Then agent $i$'s feature matrix $\Omega_i$ is defined to be an $|\mathcal{S}||\mathcal{A}_i|\times d_i$ matrix with its $(s,a_i)$-th row being $\tilde\phi_i^\top(s,a_i)$, where $(s,a_i)\in\mathcal{S}\times \mathcal{A}_i$. 

We propose a novel policy evaluation algorithm called localized TD$(\lambda)$ with linear function approximation, which is presented in Algorithm \ref{alg:LPES}. The algorithm can be viewed as an extension of the classical TD$(\lambda)$ with linear function approximation \citep{tsitsiklis1997analysis} to the case where we estimate the $\kappa_c$-truncated averaged $Q$-functions using local information.

\begin{algorithm}[H]
	\caption{Localized TD($\lambda$) with Linear Function Approximation}
	\label{alg:LPES}
	\begin{algorithmic}[1]
		\STATE \textbf{Input}: Target policy $\xi^\theta$,
		positive integers $K$ and $\kappa_c\geq \kappa_r$, initializations $w_i(0)=0$ for all $i$, step size $\alpha>0$, $\lambda\in [0,1)$, and $\epsilon>0$.
		\STATE Construct $\epsilon$-exploration policy $\hat \xi_i(a_i|s_i)=(1-\epsilon)\xi^{\theta_i}_i(a_i|s_i)+\epsilon /\abs{\sA_i}$, for all $i,a_i$, and $s_i$.
		\label{line:critic_sample_beg}
		\STATE The agents use the joint policy $\hat\xi=(\hat{\xi}_1,\hat{\xi}_2,\cdots,\hat{\xi}_n)$ to collect a sequence of samples $\tau=\{(s(t),a(t),r(t))\}_{0\leq t\leq K}$, where $r(t)=(r_1(t),r_2(t),\cdots,r_n(t))$. 
		\FOR {$i=1,2,\cdots,n$}
		\STATE $\tau|_{(i,\kappa_c)}:=\{(s_{\mathcal{N}_i^{\kappa_c}}(t),a_i(t),r_i(t))\}_{0\leq t\leq K}$
		\FOR{$t=0,1,\cdots, K-1$}
		\STATE $\delta_{i}(t) = \phi_i(s_{\mathcal{N}_i^{\kappa_c}}(t),a_i(t))^\top w_i(t)
		-r_i(t) -\gamma
		\phi_i(s_{\mathcal{N}_i^{\kappa_c}}(t+1),a_i(t+1))^\top  w_i(t)$
		\STATE $w_i(t+1)=w_i(t)-\alpha\delta_{i}(t)\zeta_{i}^{\kappa_c}(t)$ 
		\STATE $\zeta_{i}^{\kappa_c}(t+1) =(\gamma\lambda)\zeta_{i}^{\kappa_c}(t)+\phi_i(s_{\mathcal{N}_i^{\kappa_c}}(t+1),a_i(t+1))$
		\ENDFOR
		\label{line:critic_upd_end}
		\ENDFOR
		
		\STATE \textbf{Return} $\{w_i(K)\}_{i\in\mathcal{N}}$.
	\end{algorithmic}
\end{algorithm}

Note from Algorithm \ref{alg:LPES} Line $2$ that we use $\epsilon$-exploration policies to ensure exploration in localized TD$(\lambda)$. Denote the set of all $\epsilon$-exploration policies by $\Xi^\epsilon$. Importantly, agent $i$ requires only the states and the actions of the agents in its $\kappa_c$-hop neighborhood to carry out the algorithm, where $\kappa_c$ can be viewed as a tunable parameter that trades off the communication effort and the accuracy. In particular, the larger $\kappa_c$ is, the closer the $\kappa_c$-truncated averaged $Q$-function is to the true averaged $Q$-function, albeit at a cost of requiring more communication among agents.

\subsection{Localized Actor-Critic}
Combining IPG with localized TD($\lambda$), we arrive at a localized actor-critic algorithm for solving NMPGs, which is presented in Algorithm \ref{alg:approx_IPG}. 

\begin{algorithm*}[h]
	\caption{Localized Actor-Critic}\label{alg:approx_IPG}
	\begin{algorithmic}[1]
		\STATE \textbf{Input}: Non-negative integers $M$, $T$, $K$, $H$, $\kappa_c\geq \kappa_r$, and a positive real number $\epsilon>0$, initializations $\theta_i(0)=0$ for all $ i$, and $\Delta_{i}^0(m)=0$ for all $i$ and $m$.
		\FOR{$m=0,1,2,\cdots,M-1$} 
		\STATE  All agents simultaneously execute localized TD$(\lambda)$ with linear function approximation (with iteration number $K$) to estimate a $\kappa_c$-truncated averaged $Q$-function and output weight vectors $\{w_i^m\}_{i\in\sN}$. \hfill $\vartriangleright$ Critic Update
		\FOR{$t=0,1,\cdots,T-1$}
		\STATE The agents use the joint policy $\xi^{\theta(m)}=(\xi_1^{\theta_1(m)},\xi_2^{\theta_2(m)},\cdots,\xi_n^{\theta_n(m)})$ to collect a sequence of samples $\{(s^t(k),a^t(k))\}_{0\leq k\leq H-1}$
		\STATE $\eta_i^t(m)=\sum_{k=0}^{H-1}\gamma^k \nabla_{\theta_i}\log \xi_i^{\theta_i(m)}(a_i^t(k)|s_i^t(k)) \phi_i(s_{\mathcal{N}_i^{\kappa_c}}^t(k),a_i^t(k))^\top  w_i^m  $
		\STATE $\Delta_i^{t+1}(m) =\frac{t }{t+1}\Delta_i^t(m)+\frac{1}{t+1}\eta_i^t(m)$
		\ENDFOR
		\STATE  $\theta_i(m+1)=\theta_i(m)+\beta \Delta_i^T(m)$\hfill $\vartriangleright$ Actor Update
		\ENDFOR
	\end{algorithmic}
\end{algorithm*}

The algorithm consists of three major steps. First, in Algorithm \ref{alg:approx_IPG} Line $3$, each agent calls localized TD$(\lambda)$ with linear function approximation for policy evaluation and outputs a weight vector $w_i^m$ for all $i\in\mathcal{N}$. Then, in Algorithm \ref{alg:approx_IPG} Lines $4$ -- $8$, each agent uses the averaged $Q$-function estimate to iteratively construct an estimate of the independent policy gradient. Specifically, since the independent policy gradient is an expected discounted sum of the averaged $Q$-functions (cf. Eq. (\ref{eq:PGT})), we essentially construct an estimator $\Delta_i^T(m)$ (cf. Algorithm \ref{alg:approx_IPG} Line $8$) of it by taking average of total $T$ samples $\{\eta_i^t(m)\}_{0\leq t\leq T-1}$ (cf. Algorithm \ref{alg:approx_IPG} Line $6$). Finally, in Algorithm \ref{alg:approx_IPG} Line $9$, using the estimated gradient, each agent implements an approximate version of the IPG algorithm presented in Algorithm \ref{alg:exact_IPG}.

Compared with Algorithm \ref{alg:exact_IPG}, Algorithm \ref{alg:approx_IPG} has the following strengths: (1) the algorithm is model-free, (2) due to the use of truncated $Q$-functions, each agent only requires information from its $\kappa_c$-hop neighborhood to carry out the algorithm, which eliminates long-distance communication along the network, and (3) the algorithm, to some extent, overcomes the curse of dimensionality thanks to the use of linear function approximation.

\section{Algorithm Analysis}\label{sec:analysis}

We next present the main results of the paper.  We formally state our assumptions in Section \ref{subsec:assumptions} and then present convergence bounds for Algorithms \ref{alg:exact_IPG}, \ref{alg:LPES}, and \ref{alg:approx_IPG} in Section \ref{subsec:theorems}.  A proof sketch of our main theorems is given in Section \ref{subsec:proof_idea}.

\subsection{Assumptions}\label{subsec:assumptions}
We make the following assumptions.
\begin{assumption}\label{assump:pot_exp_decay}
	There exists a decreasing function $\nu:\N\rightarrow \R^+$ such that:
	\begin{align}       \abs{\Phi_i(\theta_{\nik},\theta_{-\nik}')-\Phi_i(\theta_{\nik},\theta_{-\nik})} \leq \nu(\kappa)\max_{j\in-\nik}\norm{\theta_{j}'-\theta_{j}},\;\forall\;\kappa\in \N,
	\end{align}
	where $\Phi_i(\theta)$ is the short-hand notation for $\Phi_i(\xi^\theta)$.
\end{assumption}

Assumption \ref{assump:pot_exp_decay} captures the idea that, for each agent, its potential function is less impacted by the agents far away, and can be viewed as a generalization of the decay property of the $Q$-functions in the existing literature to the networked MPG setting \citep{qu2019scalableMARL, lin2021multi, zhang2022global}. In the extreme case where $\kappa$ exceeds the diameter $ \max_{i,j}\text{dist}(i,j)$ of the network, we have $\nu(\kappa)=0$. Note that this assumption is automatically satisfied for our illustrative example in Section \ref{subsec:examples}, where changing the policy of an agent will only affect its direct neighbors. In Appendix F.5, we show that this assumption is also satisfied when each local potential function admits a stage-wise representation \citep{zhang2022logBarrierSoftmax}.

\begin{assumption}\label{assump:explore}
	It holds that $\inf_{\theta} \min_{s\in\sS} d^\theta(s)>0$, where we recall that $d^\theta$ is the discounted state visitation distribution under a softmax policy $\xi^\theta$
\end{assumption}
Assumption \ref{assump:explore} states that every state can be visited with positive probability under any policy, which easily holds when the initial state distribution $\mu(\cdot)$ is supported on the entire state space. This  assumption is standard and has been used in, e.g.,  \cite{zhang2022logBarrierSoftmax,agarwal2021theory,pmlr-v119-mei20b}. Under Assumption \ref{assump:explore}, we define
$D=1/\inf_{\theta} \min_{s\in\sS}d^{\theta}(s)$, which is finite.

\begin{assumption}\label{assump:MC}
	There exists a joint policy $\xi$ such that the Markov chain $\{s(t)\}$ induced by $\xi$ is uniformly ergodic.
\end{assumption}

Under Assumption \ref{assump:MC}, \cite[Lemma 4]{zhang2022global} implies a uniform exploration property for the Markov chain $\{(s(t),a(t))\}$ induced by any policy with entries bounded away from zero, which includes $\epsilon$-exploration policy. Therefore, for any $\hat{\xi}\in\Xi^\epsilon$, the Markov chain $\{(s(t),a(t))\}$ induced by $\hat{\xi}$ has a unique stationary distribution, denoted by $\overline{\pi}^{\hat{\xi}}\in\Delta(\sS\times\sA)$, which satisfies $\pi_{\min}:=\inf_{\hat \xi\in \Xi^{\epsilon}}\min_{i\in \sN} \min_{s_{\mathcal{N}_i^{\kappa_c}},a_i} \overline \pi^{\hat \xi} (s_{\mathcal{N}_i^{\kappa_c}},a_i)>0$. 

While Assumption \ref{assump:explore}, to some extent, already ensures uniform exploration of our policy class, we further impose Assumption \ref{assump:MC} to deal with the Markovian sampling in Algorithm \ref{alg:approx_IPG}. This type of assumption is standard in the existing literature even for the single-agent setting \citep{Srikant2019FiniteTimeEB,tsitsiklis1997analysis}.

\begin{assumption}\label{assump:lin_indep}
	For all $i\in \sN$, the feature mapping is normalized so that $\max_{i,s,a_i}\|\Tilde{\phi}_i(s,a_i)\|\leq 1$.
	In addition, the feature matrix $\Omega_i$ (the row vectors of which are $\{\Tilde{\phi}_i^\top (s,a_i)\}_{(s,a_i)\in\mathcal{S}\times \mathcal{A}_i}$) has linearly independent columns.
\end{assumption}
Assumption \ref{assump:lin_indep} is indeed without loss of generality because neither disregarding dependent features nor performing feature normalization changes the approximation power of the function class \citep{bertsekas1996neuro}. 

To state our last assumption, let $D^{\hat \xi}\in\mathbb{R}^{|\mathcal{S}||\mathcal{A}|\times |\mathcal{S}||\mathcal{A}|}$ be the diagonal matrix with diagonal entries $ \{\overline \pi^{\hat \xi}(s,a)\}_{(s,a)\in \sS\times \sA} $. Since $D^{\hat{\xi}}$ has strictly positive diagonal entries under Assumption \ref{assump:MC} and the feature matrix $\Omega_i$ has linearly independent columns for all $i$, we have $\underline{\lambda}:=\min_{i\in \sN} \inf_{\hat \xi\in \Xi^{\epsilon}}\lambda_{\min}(\Omega_iD^{\hat{\xi}} \Omega_i)>0$, where $\lambda_{\min}(\cdot)$ returns the smallest eigenvalue of a positive definite matrix.
For any $i\in\sN$ and $\theta\in\mathbb{R}^{|\mathcal{S}||\mathcal{A}|}$, let $c_i(\theta):= \min_s \sum_{a_i^*\in{\argmax}_{a_i}\overline Q_i^{\theta}(s,a_i)} \xi_i^{\theta_i}(a_i^*|s_i)$.
\begin{assumption}\label{assump:c_inf}
	$c:=\inf_{m\geq 0}\min_{1\leq i\leq N} c_i(\theta(m))>0$,  where $\{\theta(m)\}_{m\geq 0}$ are policy parameters encountered from the algorithm trajectory (cf. Algorithm \ref{alg:approx_IPG}).
\end{assumption}
The inequality stated in Assumption \ref{assump:c_inf} is called a non-uniform Łojasiewicz inequality \citep{zhang2022logBarrierSoftmax,pmlr-v119-mei20b}, which is used to connect the NE-Gap with the gradient of the objective function through gradient domination. This assumption automatically holds in the existing literature when the policy gradient is exact \citep{zhang2022logBarrierSoftmax}. However, for Algorithm \ref{alg:approx_IPG}, due to the more challenging model-free setup and the presence of noise in sampling, $c$ is not necessarily strictly positive, which motivates Assumption \ref{assump:c_inf} as a means for analytical tractability. Further relaxing this assumption is our immediate future direction.
One approach for removing Assumption \ref{assump:c_inf} is to regularize the problem (e.g., using log-barrier regularization like in \cite{zhang2021gradientStochasticGame}), which prevents the policy generated by IPG from being deterministic, albeit at a cost of introducing an asymptotic bias due to regularization.

\subsection{Results}\label{subsec:theorems}
We are now ready to present our main results. We first present the averaged Nash-regret bound of the IPG algorithm (cf. Algorithm \ref{alg:exact_IPG}) as a warm-up, then we present the finite-sample bound of Algorithm \ref{alg:approx_IPG}, which involves a critic error. Finally, we present a concise bound of the critic estimation error when using our localized TD$(\lambda)$ with linear function approximation. Given an arbitrary integer $\kappa$, let $n(\kappa) := \max_{i\in \mathcal{N}} |\nik|$ be the size of the largest $\kappa$-hop neighborhood.
\begin{theorem}\label{thm:tab_softmax_Nash_regret}
	Consider $\{\theta_i(m)\}_{0\leq m\leq M-1}$ generated by Algorithm \ref{alg:exact_IPG}. 
	Suppose that Assumptions 
	\ref{assump:pot_exp_decay}, \ref{assump:explore}, and
	\ref{assump:c_inf} are satisfied, and the step size $\beta=\frac{(1-\gamma)^3}{6n(\kappa_G)}$. Then, 
	\begin{align}
		\regret(M)
		\leq\; & \mathcal{O}\Sp{\frac{D}{c}\sqrt{\frac{\max_{j\in\sN} |\sA_j|n(\kappa_G)(\Phi_{\max}-\Phi_{\min})}{(1-\gamma)^3 M}}}\nonumber \\
		& + \mathcal{O}\Sp{\frac{D\sqrt{\max_{j\in\sN} |\sA_j|\nu(\kappa_G)}}{c(1-\gamma)}} \label{eq:IPG_bound}.
	\end{align}
\end{theorem}
The first term on the right-hand side of Eq. (\ref{eq:IPG_bound}) goes to zero at a rate of $\mathcal{O}(M^{-1/2})$, which matches with the existing convergence rate of IPG for solving MPGs \citep{zhang2022logBarrierSoftmax}. Note that, unlike in existing results, the total number of agents $n$ does not appear in the bound. Instead, we have $n(\kappa_G)$, which captures the impact of network structure. The second term on the right-hand side of Eq. (\ref{eq:IPG_bound}) arises because of the relaxation from MPG to NMPG (see Definition \ref{def:local_MPG}), which decreases with $\kappa_G$, and vanishes when $\kappa_G\geq \max_{i,j}\text{dist}(i,j)$.  

We next move on to study Algorithm \ref{alg:approx_IPG}. 
\begin{theorem}\label{thm:main}
	Consider $\{\theta_i(m)\}_{0\leq m\leq M-1}$ generated by Algorithm \ref{alg:approx_IPG}. 
	Suppose that Assumptions 
	\ref{assump:pot_exp_decay} -- \ref{assump:c_inf} are satisfied, and
	$\beta= \frac{(1-\gamma)^3}{24n(\kappa_G)}$. Then,  
	\begin{align}
		\E\left[\regret(M)\right] 
		\leq & \frac{\sqrt{\max_{j\in\sN} |\sA_j|}D }{c}\bigg\{ \mathcal{O}\Sp{\frac{\sqrt{n(\kappa_G)(\Phi_{\max}-\Phi_{\min})}}{(1-\gamma)^{1.5}M^{1/4}}} \nonumber\\
		&+\mathcal{O}\Sp{\frac{\sqrt{\nu(\kappa_G)}}{1-\gamma}}+\mathcal{O}\Sp{\frac{\sqrt{n(\kappa_G)}[1+(1-\gamma)\epsilon_{\text{critic}}]}{(1-\gamma)^2M^{1/4}}} \nonumber\\
		&+\mathcal{O}\Sp{\frac{ \sqrt{n(\kappa_G)}\epsilon_{\text{critic}}^{1/2}}{(1-\gamma)^{1.5}}}+\mathcal{O}\Sp{\frac{ \sqrt{n(\kappa_G)}\gamma^{H/2}}{(1-\gamma)^2}}\bigg\},\label{eq:bound:appro}
	\end{align}
	where $\epsilon_{\text{critic}}$ stands for the critic estimation error in policy evaluation:
	\begin{align*}   \epsilon_{\text{critic}}=\sup_{\theta,i}\mathbb{E}^{1/2}\left[\sup_{s,a_i}\left|\overline Q_{i}^{\theta}(s,a_i)-\phi_i(s_{\mathcal{N}_i^{\kappa_c}},a_i)^\top  w_i^\theta\right|^2\right].
	\end{align*}
\end{theorem}
The first two terms on the right-hand side of Eq. (\ref{eq:bound:appro}) are analogous to the two terms on the right-hand side of the IPG error bounds presented in Theorem \ref{thm:tab_softmax_Nash_regret}. The last $4$ terms are approximation errors for the independent policy gradient, which (in the order as they appear in the bound) consist of a localization error, an error incurred by using a finite sum (Algorithm \ref{alg:approx_IPG} Line $6$) to approximate an infinite sum (cf. Eq. (\ref{eq:PGT})), a critic error, and an error incurred by using a finite average (Algorithm \ref{alg:approx_IPG} Lines $4$ -- $8$) to approximate an expectation (cf. Eq. (\ref{eq:PGT})).

To establish an overall sample complexity bound of Algorithm \ref{alg:approx_IPG}, we need to specify how the critic error decays as a function of the number of iterations in localized TD$(\lambda)$ with linear function approximation, which is presented in the following.  
\begin{theorem}\label{thm:critic_short}
	Consider $\{w_i(K)\}_{i\in\mathcal{N}}$ generated by Algorithm \ref{alg:LPES}. Suppose that Assumption \ref{assump:MC} is satisfied. Then, with appropriately chosen step size $\alpha$ (see Appendix D for the explicit requirements) and large enough $K$, we have 
	\begin{align}
		\epsilon_{\text{critic}} 
		\leq \; & \mathcal{O}(1-(1-\gamma)\underline{\lambda}\alpha)^{\frac{K}{2}}
		+\mathcal{O}\left[\frac{\alpha \log(1/\alpha)}{(1-\gamma)\underline{\lambda}}\right]^{1/2} +\mathcal{O}\left(\frac{\epsilon_{\text{app}}}{\pi_{\min}(1-\gamma)}\right)\nonumber\\
		&+ \mathcal{O}\left(\frac{\gamma^{\kappa_c-\kappa_r}}{1-\gamma}\right) +\mathcal{O}\left(\frac{n\epsilon}{(1-\gamma)^2}\right),\label{eq:bound:critic}
	\end{align}
	where 
	$\epsilon_{\text{app}}$ stands for the function approximation error. See Appendix D for the explicit definition.
\end{theorem}
The first two terms on the right-hand side of Eq. (\ref{eq:bound:critic}) represent the convergence bias (which has geometric convergence rate) and the variance (which decreases with the step size $\alpha$), and their behaviors agree with existing results on stochastic approximation \citep{Srikant2019FiniteTimeEB,chen2019nonlinearSA}. The third term arises from using linear function approximation and vanishes in the tabular setting where we use a complete basis. The fourth term represents the error between the averaged $Q$-function and the $\kappa_c$-truncated averaged $Q$-function, which is introduced to overcome the scalability issue when the number of agents increases. Note that the fourth term decays exponentially with the choice of $\kappa_c$, and vanishes when $\kappa_c$ is greater than the diameter (i.e., $\max_{i,j}\text{dist}(i,j)$) of the network. The last term arises because of using $\epsilon$-exploration behavior policies to ensure sufficient exploration.

Combining Theorem \ref{thm:main} and Theorem \ref{thm:critic_short} leads to the following sample complexity bound. 

\begin{corollary}\label{co:sample_complexity}
	To achieve $\mathbb{E}[\regret(M)]\leq \tilde{\epsilon}+\mathcal{E}_{\text{EX}}+\mathcal{E}_{\text{FA}}+\mathcal{E}_{\text{LO}}$,
	the sample complexity is $\tilde{\mathcal{O}}(\tilde{\epsilon}^{-4})$, where $\mathcal{E}_{\text{EX}}$ stands for the induced error from exploration (cf. the last term on the right-hand side of Eq. (\ref{eq:bound:critic})), $\mathcal{E}_{\text{FA}}$ stands for the function approximation error (cf. the third term on the right-hand side of Eq. (\ref{eq:bound:critic})), and $\mathcal{E}_{\text{LO}}$ stands for the induced error from localization (cf. the summation of the second last term on the right-hand side of Eq. (\ref{eq:bound:critic}) and the third term on the right-hand side of Eq. (\ref{eq:bound:appro})).
\end{corollary}

In Corollary \ref{co:sample_complexity} The presence of $\mathcal{E}_{\text{EX}}+\mathcal{E}_{\text{FA}}+\mathcal{E}_{\text{LO}}$ are due to the fundamental limit of the problem, such as the approximation power of function class, using truncated averaged $Q$-functions to approximate global averaged $Q$-functions, and using ``soft'' policies to ensure exploration.

In single-agent RL, popular algorithms such as $Q$-learning and natural actor-critic are known to achieve $\tilde{\mathcal{O}}(\tilde{\epsilon}^{-2})$ sample complexity \citep{qu2020finite,lan2022policy}. While we study the more challenging setting of using localized algorithms to solve MARL problems, it is an interesting direction to investigate whether there is a fundamental gap.
In addition, while Localized Actor-Critic (cf. Algorithm \ref{alg:approx_IPG}) is an independent learning algorithm, our theoretical results require all agents to follow the same learning dynamics, which suggests some implicit coordination among the agents. Although this is common in the existing literature \citep{leo2021convergeMPG, pmlr-v162-ding22b, zhang2022logBarrierSoftmax}, developing completely independent learning dynamics is an interesting future direction.

\subsection{Proof Sketch}\label{subsec:proof_idea}
\paragraph{Analysis of the Actor}
At a high level, we use a Lyapunov approach to analyze the policy update, where the potential function is a natural choice of the Lyapunov function. The key is to bound $\Phi_i(\theta(m+1))-\Phi_i(\theta(m))$, $i\in\mathcal{N}$, in each iteration using the gradient of objective function $J_i(\cdot)$, which is related to NE-Gap of agent $i$ through the non-uniform Łojasiewicz inequality \citep{zhang2022logBarrierSoftmax,pmlr-v119-mei20b}. To exploit the network structure and to remove the raw dependence on the total number of agents in the NMPG setting, instead of directly bounding $\Phi_i(\theta(m+1))-\Phi_i(\theta(m))$, we perform the following decomposition:
\begin{align*} 
	\Phi_i(\theta(m+1))-\Phi_i(\theta(m) 
	=\;&\underbrace{\left[ \Phi_i(\theta_{\nikg}(m+1),\theta_{-\nikg}(m)) - \Phi_i(\theta(m)\right]}_{(a)} \\ 
	&+ \underbrace{\left[\Phi_i(\theta(m+1))-\Phi_i(\theta_{\nikg}(m+1),\theta_{-\nikg}(m))\right]}_{(b)}.
\end{align*}
The term $(a)$ captures the policy change of the agents inside the $\kappa_G$-hop neighborhood of agent $i$, and the first step of bounding it is to use the smoothness property of the potential function, which is similar to that of \cite{zhang2022logBarrierSoftmax}. However, unlike existing analysis of IPG, we also need to bound the error in approximating the gradient, which can be decomposed into three error terms:
\begin{enumerate}
	\item[$e_1$:] error due to estimating the averaged $Q$-function, which is exactly  the critic error;
	\item[$e_2$:] error due to the  randomness in the trajectory sampling (see Algorithm \ref{alg:approx_IPG} Lines $4$ -- $8$), which has zero mean;
	\item[$e_3$:] error resulted from truncating the sample trajectory at horizon $H$ (see Algorithm \ref{alg:approx_IPG} Lines $6$), which decays exponentially with $H$.
\end{enumerate}

Term $(b)$ results from the policy change of agents outside the $\kappa_G$-hop neighborhood of agent $i$, and is a decreasing function of $\kappa_G$ (cf. Assumption \ref{assump:pot_exp_decay}). 

\paragraph{Analysis of the Critic}
The critic is designed to perform policy evaluation of a softmax policy $\xi^\theta$ using localized TD$(\lambda)$ with linear function approximation. Similar to \cite{chen2019nonlinearSA,Srikant2019FiniteTimeEB}, we formulate localized TD($\lambda$) as a stochastic approximation algorithm and again use a Lyapunov approach to establish the finite-sample bound of the difference between $w_i(K)$ and $w_i^\theta$, where $w_i^\theta$ is the solution to a properly defined projected Bellman equation associated with agent $i$.

The challenge lies in bounding the difference between the $Q$-function associated with the weight vector $w_i^\theta$ (denoted by $Q(w_i^\theta)$) and the true averaged $Q$-function $\overline{Q}_i^\theta$ of policy $\xi^{\theta}$, which we decompose into a function approximation error, an error due to using $\epsilon$-exploration policy, and an error due to truncating the averaged $Q$-function at its $\kappa_c$-hop neighborhood, and bound them separately. 
To achieve that, we develop a novel approach 
involving the construction of a ``sub-chain'', which is an auxiliary Markov chain with state space $\sS_{\mathcal{N}_i^{\kappa_c}}\times \sA_i$. 
See Appendix D for more details.

\section{Conclusion}
We study MARL in the context of MPGs and introduce a networked structure that allows agents to learn equilibria using local information.  In particular, we develop a localized actor-critic framework for minimizing the averaged Nash regret of NMPGs. Importantly, the algorithm is scalable and uses function approximation. We provide finite-sample convergence bounds to theoretically support our proposed algorithm and conduct numerical simulations to demonstrate its empirical effectiveness. 

An immediate future direction is to investigate whether there is a fundamental gap in the convergence rates between localized MARL algorithms and single-agent RL algorithms.  It is also interesting to see if localized algorithms (with provable guarantees) can be designed to solve other classes of games beyond NMPGs.

\bibliographystyle{apalike}
\bibliography{references}

\begin{center}
	{\LARGE\bfseries Appendices}
\end{center}

\appendix

\section{Markov Congestion Game Example}\label{ap:simulations}
In this section, we provide the detailed settings and proofs of the Markov congestion game example in \Cref{subsec:examples}. Since our Markov congestion game is an extension of the classic potential game, it is natural to conjecture that it is a MPG because every single time step is a one-shot potential game. However, this intuition does not hold in general (see counterexamples in \citet{leo2021convergeMPG}). To show that the Markov congestion game is actually a MPG, we need the critical observation that the transition probability of each agent in this congestion game is completely local, i.e., the next local state of agent is determined by its current local state and local action. In \Cref{thm:potential-game-to-MDP}, we show that completely local transition probability is a sufficient condition for a networked Markov game that is potential at every single step to be an NMPG. 

\begin{theorem}\label{thm:potential-game-to-MDP}
	If the networked MDP satisfies that for an arbitrary fixed global state, the one-round game is a potential game, i.e., there exists a global potential function $\phi: \mathcal{S} \times \mathcal{A} \to \mathbb{R}$ such that
	\[r_i(s_i,a_i,s_{-i},a_{-i})-r_i(s_i',a_i',s_{-i},a_{-i})=\phi(s_i,a_i,s_{-i},a_{-i})-\phi(s_i',a_i',s_{-i},a_{-i}),\]
	and the transition probability of each agent is completely local, i.e., $s_i(t+1) \sim P_i(\cdot \mid s_i(t), a_i(t))$, then the networked MDP is an Markov potential game, i.e., there exists a potential function $\Phi: \Pi \times \mathcal{S} \to \mathbb{R}$ such that
	\[V_i^{\xi_i, \xi_{-i}}(s) - V_i^{\xi_i', \xi_{-i}}(s) = \Phi^{\xi_i, \xi_{-i}}(s) - \Phi^{\xi_i', \xi_{-i}}(s).\]
\end{theorem}

We defer the proof of \Cref{thm:potential-game-to-MDP} to \Cref{appendix:example-proof}. Note that the Markov congestion game satisfies the assumptions of \Cref{thm:potential-game-to-MDP} because at each time step $t$, we can set the potential function $\phi$ as
\[\phi(s, a) = - \frac{1}{2} \sum_{e \in \zeta} N(e, t) \left(N(e, t) - 1\right),\]
where $N(e, t)$ (the number of agents that choose edge $e$) is decided by the global state/action pair $(s, a)$. Since MPG is a special case of NMPG, we know the Markov congestion game is an NMPG.

\subsection{Simulation Results}\label{appendix:example-simulation}

In the numerical simulation, we consider a problem instance with 12 agents moving from 4 different start nodes to the same destination node on an acyclic graph (see \Cref{fig:multi-bridge-congestion-game-illustration}). Specifically, for $i \in \{1, 2, \ldots, 12\}$, agent $i$ travels from start node $b_{\lceil i/3\rceil}$ to the destination $d$. The local state space of each agent $i$ is the set of all possible locations $i$ can visit. For example, the local state space of agent $4$ (starts from $b_2$ and goes to $d$) contains 4 locations $\{b_2, c_1, c_2, d\}$. Since the maximum out degree of each node in this example is $2$, the local action space of agent $i$ contains 3 actions $\{0, 1, 2\}$, where $0$ means ``wait for one step at the current node'' and $1$($2$) means ``go through the first(second) out edge''. When a node only has one out edge (e.g., node $c_1$), local action $2$ is treated as action $1$. We set the time elapse cost $\epsilon = 0.5$ in this simulation.

\begin{figure}[h]
	\centering
	\begin{tikzpicture}
		\filldraw[black] (0,-1) circle (2pt) node[anchor=east]{$b_1$};
		\filldraw[black] (0,-2) circle (2pt) node[anchor=east]{$b_2$};
		\filldraw[black] (0,-4) circle (2pt) node[anchor=east]{$b_3$};
		\filldraw[black] (0,-5) circle (2pt) node[anchor=east]{$b_4$};
		
		\filldraw[black] (3,-2) circle (2pt) node[anchor=south]{$c_1$};
		\filldraw[black] (3,-3) circle (2pt) node[anchor=south]{$c_2$};
		\filldraw[black] (3,-4) circle (2pt) node[anchor=south]{$c_3$};
		
		\filldraw[black] (6,-3) circle (2pt) node[anchor=west]{$d$};
		
		\draw[-latex] (0, -1) -- (3, - 2);
		\draw[-latex] (0, -2) -- (3, - 2);
		\draw[-latex] (0, -2) -- (3, - 3);
		\draw[-latex] (0, -4) -- (3, - 3);
		\draw[-latex] (0, -4) -- (3, - 4);
		\draw[-latex] (0, -5) -- (3, - 4);
		
		\draw[-latex] (3, -2) -- (6, -3);
		\draw[-latex] (3, -3) -- (6, -3);
		\draw[-latex] (3, -4) -- (6, -3);
	\end{tikzpicture}
	\caption{Illustration of the simulation setting. On this directed acyclic graph, each agent $i \in \{1, 2, \ldots, 12\}$ travels from start node $b_{\lceil i/3\rceil}$ to the same destination $d$. The policy of each agent is a mapping from its local state (its current location) to the distributions of local actions (pass through an outgoing edge or wait for one time step).}
	\label{fig:multi-bridge-congestion-game-illustration}
\end{figure}
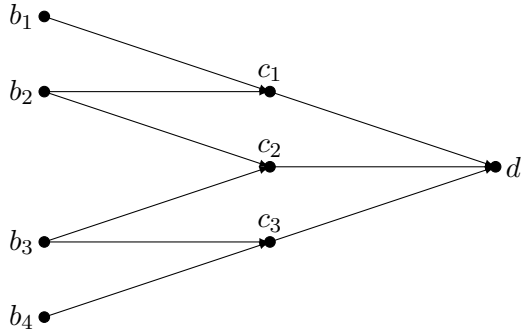

We simulate our Localized Actor-Critic algorithm (\Cref{alg:exact_IPG}). For the policy evaluation subroutine, we adopt a localized TD$(0)$ with linear function approximation and $\kappa_c = 1$. For each agent $i$, we use the hot encoding of the 1-hop state $s_{\mathcal{N}_i^1}$ and the local action $a_i$ as the feature vector $\phi_i(s_{\mathcal{N}_i^1}, a_i)$. During training, we evaluate the Nash gap $\text{NE-Gap}_i(\xi)$ by fixing the local policies of all agents except $i$ and use a policy gradient to learn the best response of agent $i$ under other agents' policy profile.  \textcolor{black}{We list the detailed choice of hyperparameters in Appendix \ref{subsec:supp_hyperparam} and present the code in \url{https://github.com/yihenglin97/Networked_MPG_Traffic_Game}.} 
We plot the Nash regret curve and the average Nash regret curve in \Cref{fig:multi-bridge-congestion-game-illustration}. The results show that both the Nash regret and the averaged Nash regret converge to zero as the learning progresses.

\begin{figure}[h]
	\centering
	\begin{subfigure}{0.5\textwidth}
		\centering
		\includegraphics[width=.9\linewidth]{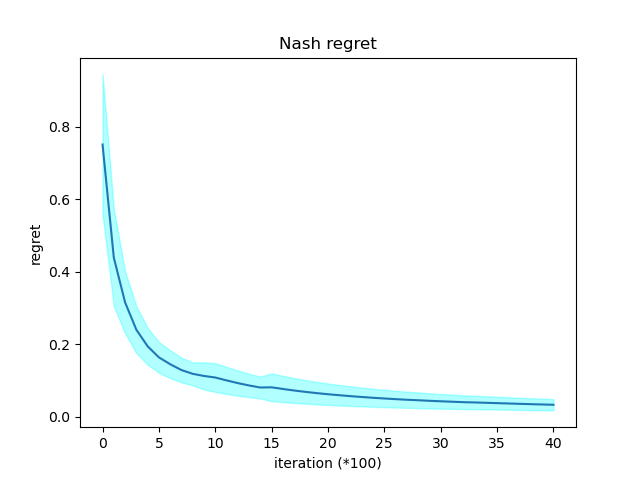}
	\end{subfigure}
	\begin{subfigure}{0.5\textwidth}
		\centering
		\includegraphics[width=.9\linewidth]{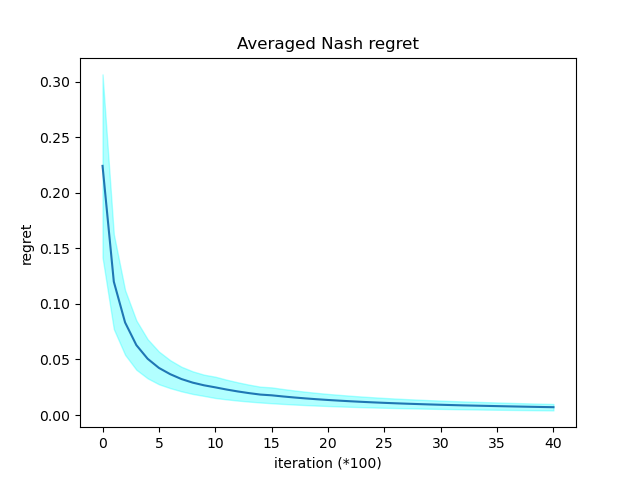}
	\end{subfigure}
	\caption{Nash regret (left) and averaged Nash regret (right) of Localized Actor-Critic (Algorithm \ref{alg:exact_IPG}). The blue curve represents the mean and the cyan area represents the standard deviation over 10 random initial policies.}
	\label{fig:test}
\end{figure}

\subsection{Proof of Theorem \ref{thm:potential-game-to-MDP}}\label{appendix:example-proof}
Under the completely local transition probability, the distribution of the local state of agent $i$ at time $t$ depends only on the initial state $s_i(0)$. We denote this local distribution as $d_i^{\xi_i}[t](\cdot \mid s_i(0))$ and the joint distribution as $d_{\mathcal{N}}^\xi[t](\cdot \mid s(0))$. We define the potential function as
\[\Phi^\xi(s) = \sum_{t = 0}^\infty \gamma^t \sum_{s' \in \mathcal{S}, a' \in \mathcal{A}} d_{\mathcal{N}}^\xi[t]((s', a') \mid s(0) = s) \cdot \phi(s', a').\]
The potential function can be rewritten as
\begin{align*}
	&\Phi^{\xi_i, \xi_{-i}}(s)\\
	={}& \sum_{t = 0}^\infty \gamma^t \sum_{s', a'} d_{-i}^{\xi_{-i}}[t]((s'_{-i}, a'_{-i}) \mid s_{-i}(0) = s_{-i}) \cdot d_{i}^{\xi_i}[t]((s'_{i}, a'_{i}) \mid s_{i}(0) = s_{i}) \cdot \phi(s', a'_i, a'_{-i})\\
	={}& \sum_{t = 0}^\infty \gamma^t \sum_{s'_{-i}, a'_{-i}} d_{-i}^{\xi_{-i}}[t]((s'_{-i}, a'_{-i}) \mid s_{-i}(0) = s_{-i}) \\
	&\times \sum_{s'_i, a'_i} d_{i}^{\xi_i}[t]((s'_{i}, a'_{i}) \mid s_{i}(0) = s_{i}) \cdot \phi(s', a'_i, a'_{-i}).
\end{align*}
Note that
\begin{align*}
	&\sum_{s'_i, a'_i} \left(d_{i}^{\xi_i}[t]((s'_{i}, a'_{i}) \mid s_{i}(0) = s_{i}) - d_{i}^{\xi_i'}[t]((s'_{i}, a'_{i}) \mid s_{i}(0) = s_{i})\right) \cdot \phi(s', a'_i, a'_{-i})\\
	={}& \sum_{s'_i, a'_i} \left(d_{i}^{\xi_i}[t]((s'_{i}, a'_{i}) \mid s_{i}(0) = s_{i}) - d_{i}^{\xi_i'}[t]((s'_{i}, a'_{i}) \mid s_{i}(0) = s_{i})\right) \cdot r_i(s', a'_i, a'_{-i}).
\end{align*}
Therefore, we have
\[\Phi^{\xi_i, \xi_{-i}}(s) - \Phi^{\xi_i', \xi_{-i}}(s) = V_i^{\xi_i, \xi_{-i}}(s) - V_i^{\xi_i', \xi_{-i}}(s).\]

\subsection{Choice of Hyperparameters}\label{subsec:supp_hyperparam}
In numerical experiments, we use the following hyperparameters: Localized TD($\lambda$) with Linear Function Approximation (\Cref{alg:LPES}): $\lambda = 0, K=10, \alpha = 0.001, \epsilon = 0$. Localized Actor-Critic (\Cref{alg:approx_IPG}): $M = 4000, T = 1, H = 15, K = 10, \kappa_c = 1, \beta = 0.001$. We construct the feature $\phi_i(s_{\mathcal{N}_i^{\kappa_c}},a_i)$ by concatenating the one-hot encodings for $s_j, j \in \mathcal{N}_i^{\kappa_c}$ and $a_i$.

\section{Proof of Theorem  \ref{thm:tab_softmax_Nash_regret}}
For any $\kappa\leq \kappa_G$ and $\beta\leq 1/L(\kappa)$, we have by Lemma \ref{lem:opt_step} that
\begin{align*}
	\Phi_i(\theta(M))-\Phi_i(\theta(0)) 
	= \;&\sum_{m=0}^{M-1}\left[\Phi_i(\theta(m+1))-\Phi_i(\theta(m))\right]\\
	\geq \; &\sum_{m=0}^{M-1}\left[\frac{\beta}{2} \norm{\nabla_{\theta_{\nik}}\Phi_i(\theta(m))}^2-\frac{\sqrt 2 \nu(\kappa)\beta}{(1-\gamma)^2}\right].
\end{align*}
It follows that
\begin{align*}
	\frac{1}{M}\sum_{m=0}^{M-1} \norm{\nabla_{\theta_{\nik}}\Phi_i(\theta(m))}^2 
	\leq\;& \frac{2}{M\beta}[\Phi_i(\theta(M))-\Phi_i(\theta(0))] +\frac{2\sqrt 2 \nu(\kappa)}{M(1-\gamma)^2}\\
	\leq\;& \frac{2(\Phi_{\max}-\Phi_{\min})}{M\beta} +\frac{2\sqrt 2 \nu(\kappa)}{(1-\gamma)^2},
\end{align*}
where the last line follows from Lemma \ref{lem:bound_pot}. Note that
\begin{align*}
	\norm{\nabla_{\theta_{\nik}}\Phi_i(\theta(m))}^2 =\;&\norm{\nabla_{\theta_{\nik}}J_i(\theta(m))}^2\\
	\geq \;&\norm{\nabla_{\theta_i}J_i(\theta(m))}^2\\
	\geq\; & \frac{c(\theta(m))^2}{\max_{j\in\sN}|\sA_j|D(\theta(m))^2} \text{NE-Gap}_i(\theta(m))^2 \tag{Lemma \ref{lem:NE_bound_grad_Phi}}\\
	\geq\; &\frac{c^2\text{NE-Gap}_i(\theta(m))^2}{\max_{j\in\sN} |\sA_j|D^2}.\tag{$c=\inf_\theta\min_{i}c_i(\theta)$ and $D=\sup_\theta D(\theta)$}
\end{align*}
Therefore, we have
\begin{align*}
	\frac{1}{M}\sum_{m=0}^{M-1}\text{NE-Gap}_i(\theta(m))^2
	\leq \frac{\max_{j\in\sN} |\sA_j|D^2}{c^2}\left( \frac{2(\Phi_{\max}-\Phi_{\min})}{M\beta} +\frac{2\sqrt 2 \nu(\kappa)}{(1-\gamma)^2}\right).
\end{align*}
It then follows from the definition of $\regret_i(M)$ that
\begin{align*}
	\regret_i(M)=\;&\frac{1}{M}\sum_{m=0}^{M-1}\text{NE-Gap}_i(\theta(m))\\
	\leq \;&\left(\frac{1}{M}\sum_{m=0}^{M-1}\text{NE-Gap}_i(\theta(m))^2\right)^{1/2}\tag{Jensen's inequality}\\
	\leq \;&\frac{\max_{j\in\sN} |\sA_j|^{1/2}D}{c}\left( \frac{2(\Phi_{\max}-\Phi_{\min})}{M\beta} +\frac{2\sqrt 2 \nu(\kappa)}{(1-\gamma)^2}\right)^{1/2}\\
	\leq \;&\frac{2\max_{j\in\sN} |\sA_j|^{1/2}D}{c}\left( \frac{(\Phi_{\max}-\Phi_{\min})^{1/2}}{M^{1/2}\beta^{1/2}} +\frac{ \nu(\kappa)^{1/2}}{(1-\gamma)}\right),
\end{align*}
where the last line follows from $\sqrt{a+b}\leq \sqrt{a}+\sqrt{b}$ for any $a,b\geq 0$. Since the RHS of the previous inequality is not a function of $i$, choosing $\kappa=\kappa_G$ and $\beta=1/L(\kappa_G)=\frac{(1-\gamma)^3}{6n(\kappa_G)}$, and we complete the proof.

\subsection{Supporting Lemmas}
\begin{lemma}\label{lem:l1_smooth_v}
	The following inequality holds for all  $\theta, \theta'$, and $i\in\mathcal{N}$:
	\begin{align*}
		\norm{\nabla_{\theta_i}J_i(\theta)-\nabla_{\theta_i'}J_i(\theta')}_1 \leq \frac{6}{(1-\gamma)^3}\sum_{j\in \sN}\norm{\theta_{j}-\theta'_{j}}.
	\end{align*}
\end{lemma}

\begin{proof}[Proof of Lemma \ref{lem:l1_smooth_v}]
	Using Lemma \ref{cor:softmax_v_grad}, and we have 
	\begin{align*}
		&\norm{\nabla_{\theta_i}J_i(\theta)-\nabla_{\theta_i'}J_i(\theta')}_1 \\
		=\;&\frac{1}{1-\gamma}\sum_{s_i,a_i}\abs{\sum_{s_{-i}}\left( d^\theta(s_i,s_{-i})\xi_i^{\theta_i}(a_i|s_i)\overline A_{i}^\theta(s,a_i) - d^{\theta'}(s_i,s_{-i})\xi_i^{\theta_i'}(a_i|s_i)\overline A_{i}^{\theta'}(s,a_i) \right)} \\
		=\;& \frac{1}{1-\gamma}\sum_{s_i,a_i}\abs{\sum_{s_{-i},a_{-i}}\left( d^\theta(s)\xi^{\theta}(a|s) A_{i}^\theta(s,a) - d^{\theta'}(s)\xi^{\theta'}(a|s) A_{i}^{\theta'}(s,a) \right)} \\
		\leq \;& \frac{1}{1-\gamma}\sum_{s,a}\abs{ d^\theta(s)\xi^{\theta}(a|s) A_{i}^\theta(s,a) - d^{\theta'}(s)\xi^{\theta'}(a|s) A_{i}^{\theta'}(s,a)} \\
		\leq \;& \frac{1}{1-\gamma}\bigg(\sum_{s,a}\abs{ d^\theta(s)\xi^{\theta}(a|s)-  d^{\theta'}(s)\xi^{\theta'}(a|s)}\abs{A_{i}^\theta(s,a)} \\
		&+\sum_{s,a} d^{\theta'}(s)\xi^{\theta'}(a|s) \abs{A_{i}^\theta(s,a)-A_{i}^{\theta'}(s,a)}\bigg) \\
		\leq  \;& \frac{1}{1-\gamma}\left(\sum_{s,a} \frac{1}{1-\gamma}\abs{ d^\theta(s)\xi^{\theta}(a|s)-  d^{\theta'}(s)\xi^{\theta'}(a|s)} +\max_{s,a} \abs{A_{i}^\theta(s,a)-A_{i}^{\theta'}(s,a)}\right).
	\end{align*}
	Lemma $32$ and Corollary $35$ of \cite{zhang2022logBarrierSoftmax} imply
	\begin{align*}
		\abs{A_{i}^\theta(s,a)-A_{i}^{\theta'}(s,a)} &\leq \frac{2}{(1-\gamma)^2} \max_s\sum_{a\in \sA}\abs{\xi^{\theta}(a|s)-\xi^{\theta'}(a|s)} \\
		\frac{1}{1-\gamma}\sum_{s,a}\abs{ d^\theta(s)\xi^{\theta}(a|s)} &\leq \frac{1}{(1-\gamma)^2} \max_s \sum_{a\in \sA}\abs{\xi^{\theta}(a|s)-\xi^{\theta'}(a|s)}.
	\end{align*}
	Therefore, we have
	\begin{align*}
		\norm{\nabla_{\theta_i}J_i(\theta)-\nabla_{\theta_i}J_i(\theta')}_1 
		\leq \;& \frac{3}{(1-\gamma)^3} \max_s \sum_{a\in \sA}\abs{\xi^{\theta}(a|s)-\xi^{\theta'}(a|s)} \\
		= \;&\frac{3}{(1-\gamma)^3} \max_s \sum_{i\in \sN}\sum_{a_i\in \sA_i}\abs{\xi_i^{\theta_i}(a_i|s_i)-\xi_i^{\theta_i'}(a_i|s_i)} \tag{Lemma \ref{lem:prod_pol_diff_general}}\\
		\leq  \;& \frac{6}{(1-\gamma)^3} \sum_{i\in \sN} \norm{\theta_i-\theta_i'}.
	\end{align*}
	where the last line follows from \cite[Corollary 37]{zhang2022logBarrierSoftmax}.
\end{proof}

\begin{lemma}\label{lem:value_smooth} 
	The following inequality holds for all $\kappa\leq \kappa_G$, $\theta=(\theta_{\nik},\theta_{-\nik})$, and $\theta'=(\theta_{\nik}',\theta_{-\nik})$:
	\begin{align*}
		\norm{\nabla_{\theta_{\nik}}\Phi_i(\theta)-\nabla_{\theta_{\nik}}\Phi_i(\theta')} \leq L(\kappa)\norm{\theta_{\nik}-\theta'_{\nik}},
	\end{align*}
	where $L(\kappa)=\frac{6n(\kappa)}{(1-\gamma)^3}$.
\end{lemma}

\begin{proof}[Proof of Lemma \ref{lem:value_smooth}]
	Using the definition of NMPG (cf. Definition \ref{def:local_MPG}) and we have
	\begin{align*}
		&\norm{\nabla_{\theta_{\nik}}\Phi_i({\theta_\nik},\theta_{-\nik})-\nabla_{\theta_{\nik}}\Phi_i(\theta'_\nik,\theta_{-\nik})}^2 \\
		=\; & \sum_{j\in \nik}\norm{\nabla_{\theta_j}J_j({\theta_\nik},\theta_{-\nik})-\nabla_{\theta_j}J_j(\theta'_\nik,\theta_{-\nik})}^2 \\
		\leq\; & \sum_{j\in \nik}\norm{\nabla_{\theta_j}J_j({\theta_\nik},\theta_{-\nik})-\nabla_{\theta_j}J_j(\theta'_\nik,\theta_{-\nik})}_1^2 \\
		\leq \;& \sum_{j\in \nik} \left(\frac{6}{(1-\gamma)^3}\sum_{j'\in \nik}\norm{\theta_{j'}-\theta'_{j'}}\right)^2 \tag{Lemma \ref{lem:l1_smooth_v}}\\
		=\; & \frac{36\abs{\nik}}{(1-\gamma)^6} \left(\sum_{j'\in \nik}\norm{\theta_{j'}-\theta'_{j'}}\right)^2 \\
		\leq  \;& \frac{36\abs{\nik}^2}{(1-\gamma)^6}\sum_{j'\in \nik}\norm{\theta_{j'}-\theta'_{j'}}^2 \tag{Cauchy-Schwarz inequality}\\
		\leq\; & \frac{36n(\kappa)^2}{(1-\gamma)^6}\norm{\theta_{\nik}-\theta'_{\nik}}^2,
	\end{align*}
	where the last line follows from the definition of $n(\kappa)$.
	Taking square root on both sides of the previous inequality and we have the desired result.
\end{proof}

\begin{lemma}\label{lem:opt_step}
	Consider $\{\theta_i(m)\}_{0\leq m\leq M}$ generated by Algorithm \ref{alg:exact_IPG}. Suppose that  $\kappa\leq \kappa_G$ and $\beta\leq\frac{1}{L(\kappa)}$ , where $L(\kappa)$ is defined in Lemma \ref{lem:value_smooth}. Then we have for any $i\in \sN$ and $m\geq 0$ that
	\begin{align*}
		\Phi_i(\theta(m+1))-\Phi_i(\theta(m)) \geq \frac{\beta}{2} \norm{\nabla_{\theta_{\nik}}\Phi_i(\theta(m))}^2-\frac{\sqrt 2 \nu(\kappa)\beta}{(1-\gamma)^2}.
	\end{align*}
\end{lemma}

\begin{proof}[Proof of Lemma \ref{lem:opt_step}]

	For any $i\in\mathcal{N}$ and $m\geq 0$, we have
	\begin{align}
		&\Phi_i(\theta(m+1))-\Phi_i(\theta(m) \nonumber\\
		=\;&\left[\Phi_i(\theta(m+1))-\Phi_i(\theta_{\nik}(m+1),\theta_{-\nik}(m))\right] \nonumber\\
		&+ \left[ \Phi_i(\theta_{\nik}(m+1),\theta_{-\nik}(m)) - \Phi_i(\theta(m))\right]\label{eq1:decomposition}.
	\end{align}
	To bound the first term on the RHS of Eq. (\ref{eq1:decomposition}), using Assumption \ref{assump:pot_exp_decay} and the update equation in Algorithm \ref{alg:exact_IPG} Line $3$ and we have
	\begin{align*}
		\Phi_i(\theta(m+1))-\Phi_i(\theta_{\nik}(m+1),\theta_{-\nik}(m))\geq \;&-\nu(\kappa)\max_{j\in-\nik}\norm{\theta_j(m+1)-\theta_j(m)}\\
		= \;&-\nu(\kappa)\beta\max_{j\in-\nik}\norm{\nabla_{\theta_j}J_j(\theta(m))}\\
		\geq  \;&-\frac{\sqrt{2}\nu(\kappa)\beta}{(1-\gamma)^2},
	\end{align*}
	where the last line follows from Lemma \ref{lem:bound_grad}.

	To bound the second term on the RHS of the previous inequality, note that the smoothness property (cf. Lemma \ref{lem:value_smooth}) implies that
	\begin{align*}
		&\Phi_i(\theta_{\nik}(m+1),\theta_{-\nik}(m)) - \Phi_i(\theta(m)) \\
		\geq\;& \langle\nabla_{\theta_{\nik}}\Phi_i(\theta(m)),\theta_{\nik}(m+1)-\theta_{\nik}(m)\rangle - \frac{L(\kappa)}{2} \norm{\theta_{\nik}(m+1)-\theta_{\nik}(m)}^2 \\
		= \;&\beta\langle\nabla_{\theta_{\nik}}\Phi_i(\theta(m)),\nabla_{\theta_{\nik}}J_i(\theta(m))\rangle-\frac{L(\kappa)\beta^2}{2} \norm{\nabla_{\theta_j}J_i(\theta(m))}^2\\
		= \;&\left(\beta-\frac{L(\kappa)\beta^2}{2}\right) \norm{\nabla_{\theta_{\nik}}\Phi_i(\theta(m))}^2\tag{Lemma \ref{lem:NMPG}}\\
		\geq \;&\frac{\beta}{2} \norm{\nabla_{\theta_{\nik}}\Phi_i(\theta(m))}^2,
	\end{align*}
	where the last line follows from $\beta\leq 1/L(\kappa)$.
	
	Using the previous two bounds in Eq. (\ref{eq1:decomposition}) and we have
	\begin{align*}
		\Phi_i(\theta(m+1))-\Phi_i(\theta(m)\geq \;&\frac{\beta}{2} \norm{\nabla_{\theta_{\nik}}\Phi_i(\theta(m))}^2-\frac{\sqrt{2}\nu(\kappa)\beta}{(1-\gamma)^2},
	\end{align*}
	which completes the proof.
\end{proof}

Define $D(\theta)=1/\min_{s}d^\theta(s)$ and
\begin{align*}
	c_i(\theta)= \min_s \sum_{a_i^*\in\argmax_{a_i}\overline Q_i^{\theta}(s,a_i) } \xi_i^{\theta_i}(a_i^*|s_i),\quad  \forall\; i\in \sN.
\end{align*}

\begin{lemma}[Lemma $1$ of \cite{zhang2022logBarrierSoftmax} ]\label{lem:NE_bound_grad_Phi}
	It holds for all $i\in\mathcal{N}$ that
	\begin{align*}
		\text{NE-Gap}_i(\theta)\leq \frac{\sqrt{|\sA_i|}D(\theta)}{c(\theta)}\norm{\nabla_{\theta_i}J_i(\theta)}.
	\end{align*}
\end{lemma}

\section{Analysis of Critic}\label{sec:local_policy_eval}

In this section, we generalize the localized TD($\lambda$) algorithm (cf. Algorithm \ref{alg:LPES}) to \emph{localized stochastic approximation} and analyze its performance. The localized stochastic approximation problem can be of independent interest. 

\subsection{Localized Stochastic Approximation}

To make this section self-contained, we first introduce the problem setting of localized stochastic approximation, and then propose the generalized TD($\lambda$) algorithm. After that, we state assumptions and the main result of the localized stochastic approximation problem. 

A localized stochastic approximation problem consists of an infinite-horizon, multi-agent Markov chain $\sM=(\sN, \mathcal E, \sZ, \PR, \supr, \gamma, \mu')$ and a fixed agent $\agentk\in \sN$. Here $\sN=\{1,2,\cdots,n\}$ is the set of agents, associated with an undirected graph $\mathcal G=(\mathcal N,\mathcal E)$.
$\sZ=\prod_{i\in\sN} \sZ_i$ is the global state space, where $\sZ_i$ is the local state space of agent $i$.

At time $t\geq 0$, given current state $z(t)\in \sZ$, for each $i\in\mathcal{N}$, the next individual state $z_i(t+1)$ is independently generated and is only dependent on its neighbors' states and its own action:
\begin{align}\label{eq:loc_z_dynamics}
	\PR(z(t+1)\mid z(t)) = \prod_{i=1}^n \PR_i(z_i(t+1)\mid z_{N_i}(t)).
\end{align}
$\gamma$ is the discount factor. $\mu'$ is the initial state distribution.  $\supr:\sZ_{\nkkr}\rightarrow [0,1]$ is the reward function. 

Before the learning stage, the learner can observe partial information of a trajectory sampled from the Markov chain:
\begin{align*}
	\tau_{\kappa_c}=(z_{\nkkc}(0),\supr(0),\cdots,z_{\nkkc}(\supepoch)).
\end{align*}
Here $\supepoch$ is the horizon of the sampled trajectory, $\kappa_c>\kappa_r$ measures the observability of the learner. 

The goal is to estimate agent $\agentk$'s cost function $C(z)=\sum_{t=0}^\infty \gamma^t \E[\supr(z_{\nkkr}(t)) | z(0)=z]$. 

The localized stochastic approximation problem has the following two applications in localized stochastic approximation, depending on the choice of $z_i$'s.
\begin{itemize}
	\item Estimate local $Q$-function $Q_i(s,a)$. Let $z_j=(s_j,a_j), \forall j\in \sN$.
	\item Estimate averaged $Q$-function $\bar Q_i(s,a_i)$. Let $z_j=s_j,\forall j\neq k$, $z_{\agentk}=(s_{\agentk},a_{\agentk})$. Notice that $\{s(t),a_{\agentk}(t)\}$ forms an induced Markov chain of agent $\agentk$'s averaged MDP \citep{zhang2021gradientStochasticGame}.
\end{itemize}

\paragraph{Additional notation}
We introduce some other notations that will be used in this section.

We denote by $\pi_t\in \Delta(\sZ)$ the state distribution at time $t$ and denote by $\overline \pi\in \Delta(\sZ)$ the stationary state distribution.

We use the notation $\pi_{\kappa_c,t}(z_{\nkkc})$ to represent the marginal probability of state $z_{\nkkc}$ at time $t$
and use $\overline{\pi}_{\kappa_c}(z_{\nkkc})$ to represent the marginal probability of state $z_{\nkkc}$ under stationary distribution.

For any set $\mathcal{X}$ and two distributions $\pi_1,\pi_2\in \Delta(\mathcal X)$, define 
\begin{align}
	\TV(\pi_1,\pi_2)=\max_{A\subseteq X}\abs{\pi_1(A)-\pi_2(A)}.
\end{align}
to be the total variation distance between $\pi_1$ and $\pi_2$. The total variation distance has the following property \citep{levin2017markov}:
\begin{align}
	\TV(\pi_1,\pi_2)=\frac{1}{2}\norm{\pi_1-\pi_2}_1\leq 1.
\end{align}

Given a Markov chain with state space $\mathcal X=\mathcal X_1\times \cdots\times \mathcal X_l$ and transition probability $\Gamma$, for any set $I\subseteq \mathcal U=\{1,2,\cdots,l\}$ and any $x_I'\in \mathcal X_I$, any $x\in \mathcal X$, let
\begin{align}
	\Gamma_I(x_I'|x):=\sum_{x_{\mathcal U/I}'}\Gamma(x_I',x_{\mathcal U/I}'|x)
\end{align}
be the marginal transition probability of $x_I'$ given previous global state $x$.

\subsection{Generalized TD($\lambda$) algorithm}
Now we design a generalized version of Algorithm \ref{alg:LPES} to make it compatible with more classical policy evaluation methods. We consider approximating the cost function $C$ with function class $\hat C: \sZ_{\nkkc}\times \R^d\rightarrow \R$, where $\kappa_c>\kappa_r$. That is, $C(z)$ is approximated by $\hat C(z_{\nkkc},w)$, where $w\in \R^d$ is the parameter. Notice that we allow arbitrary function approximation class, and $\hat C$ only depends on states of agents in $\kappa_c$-hop neighborhood. 

Furthermore, we denote by $\supphi:\sZ_{\nkkc}\rightarrow \R^d$ the feature vector. Assume $\norm{\supphi(z_{\nkkc})}\leq 1$ without loss of generality. We introduce the feature vector for compatibility with linear function approximation. Nevertheless, we emphasize that our algorithm still allows general function approximation class by choosing $\lambda=0$. See Algorithm \ref{alg:recursive_estimator} for the complete algorithm.

We point out that Algorithm \ref{alg:recursive_estimator} reduces to Algorithm \ref{alg:LPES} if we choose $t_0=0$ and let $F$ be the temporal difference $\delta_{\agentk}(t)$ in Algorithm \ref{alg:LPES}. Besides, many other classical single agent policy evaluation algorithm, such as LSTD \citep{boyan1999least}, are a special case of Algorithm \ref{alg:recursive_estimator}.

\begin{algorithm}
	\caption{Generalized TD($\lambda$)}\label{alg:recursive_estimator}
	\begin{algorithmic}[1]
		\STATE \textbf{Input}: 
		$\tau_{\kappa_c}$ 
		\STATE \textbf{Parameter}: $\lambda\in[0,1)$, $t_0\geq 0$, function $F:(\sZ_{\nkkc})^{t_0+2}\times \R^d\rightarrow \R$. \;
		
		\textbf{Initialization}: $w(0):=0$, $\zeta(0):=\supphi(z_{\nkkc}(0))$.
		
		\FOR{$t=t_0,t_0+1,\cdots, \supepoch-1$}
		\STATE $X(t):=(z_{\nkkc}(t-t_0), \cdots,z_{\nkkc}(t+1))$.
		\STATE Update parameter $w(t+1)=w(t)+\alpha F(X(t),w(t))\zeta(t) $.
		\STATE Update eligibility vector $\zeta(t+1)=\lambda \zeta(t)+\supphi(z_{\nkkc}(t+1))$.
		\ENDFOR
		\STATE \textbf{Return} $w(\supepoch)$
	\end{algorithmic}
\end{algorithm}

\subsection{Convergence result}
To make the section self-contatined, we restate assumptions needed for localized stochastic approximation. We point out that all assumptions below can be satisfied by localized TD($\lambda$) with linear function approximation under the assumptions in the main text. 
\begin{assumption}\label{assump:chain_irreducible}
	$\PR$ is aperiodic and irreducible. 
\end{assumption}

Assumption \ref{assump:chain_irreducible} guarantees the existence and uniqueness of stationary distribution. In addition, Assumption \ref{assump:chain_irreducible} ensures that there exists $c'>1$ and $\rho'\in (0,1)$ such that 
\begin{align}\label{eq:mixing_chain}
	\TV(\pi_t,\overline \pi)\leq c'\rho'^t. 
\end{align}

We further define the stationary distribution of $\zeta(t)$ \citep{tsitsiklis1997analysis}. Consider a stationary Markov process $\{z(t)\}$ ($-\infty<t<\infty$), in which the state distribution at each time $t$ is the stationary distribution. Let 
\begin{align*}
	\overline\zeta(t)=\sum_{k=-\infty}^t \lambda^{t-k}\supphi(z_{\nkkc}(k)),
\end{align*}
where $\{z(k)\}$ is sampled from the stationary Markov process. $\overline\zeta(t)$ is well-defined, and its distribution is invariant of $t$. Thus we can use the distribution of $\overline \zeta(t)$ under stationary Markov process as the stationary distribution of $\zeta(t)$. We use $\overline \E[\cdot]$ to represent the expected value of a formula, given that $\{z(t)\}$ follows the stationary Markov process and $\zeta(t)$ is sampled from the defined stationary distribution.

Let $G(\zeta(t),X,w)=\zeta(t) F(X,w)$ and denote 
\begin{align*}
	\overline G(w):=\overline\E[\zeta(t) F(X,w)].
\end{align*}

\begin{assumption}\label{assump:chen_assump}
	
	\begin{enumerate}
		\item     There exists $L_1>1$ such that 
		\begin{align*}
			&\abs{F(x,w_1)-F(x,w_2)}\leq L_1\norm{w_1-w_2}, \forall w_1,w_2,x \\
			&\abs{F(x,0)}\leq L_1, \forall x
		\end{align*}
		\item $\overline G(w)$ has a unique zero $w^*$. In addition, there exists $c_0>0$ such that 
		\begin{align*}
			(w-w^*)^\top \overline G(w)\leq -c_0\|w-w^*\|^2, \forall w\in \R^d.
		\end{align*}
	\end{enumerate}
\end{assumption}
Point 1. ensures that our updating term is Lipschitz, while point 2. guarantees the existence of negative drift, which is crucial in Lyapunov analysis. Both assumptions are standard in non-linear stochastic approximation problems \citep{chen2019nonlinearSA}.

\begin{assumption}\label{assump:hatC_lipschitz}
	$\hat C(z_{\nkkc},w)$ is $L_2$-Lipschitz with respect to $w$, i.e., 
	\begin{align*}
		\abs{\hat C(z_{\mathcal{N}_i^\kappa},w_1)-\hat C(z_{\mathcal{N}_i^\kappa},w_2)}\leq L_2\|w_1-w_2\|
	\end{align*}
	for all $z_{\mathcal{N}_i^\kappa}\in \mathcal Z_{\mathcal{N}_i^\kappa},w_1,w_2\in \R^d$.
\end{assumption}

Given the assumptions above, we can show a geometric mixing rate of $G(\zeta(t),X(t), w)$ to $\overline{G}(w)$ w.r.t. $t$. To be concrete, there exists $c_g(c',\rho',\lambda,t_0)>0$ and $\rho_g(\rho',\lambda)\in (0,1)$, such that for any $t\geq t_0$, we have
\begin{align}
	&\norm{\E[G(\zeta(t),X(t),w)]-\overline G(w)} 
	\leq L_1(\|w\|+1)c_g(c',\rho',\lambda,t_0)[\rho_g(\rho',\lambda)]^t.
	\label{eq:G_decay}
\end{align}
This can be viewed as a generalized result of \cite[Lemma 6.7]{bertsekas1996neuro}, and we defer the proof to Appendix \ref{subsubsec:decay_G}. In order to state conditions on stepsize, we
introduce the concept of mixing time of function $G$.

\begin{definition}[Mixing time of function $G$]\label{def:G_mix}
	The mixing time of function $G$ with precision $\delta$, for any $\delta>0$, is defined as 
	\begin{align*}
		t_{\delta}':=\min\left\{t\geq 1\;\middle | \;\norm{\E[G(\zeta_t,X(t),w)]-\bar G(w)}\leq \delta L_1(\norm{w}+1)\right\}.    
	\end{align*}
\end{definition}
Eq. (\ref{eq:mixing_chain}) ensures that
\begin{align*}
	t_{\delta}'=O(\log(1/\delta))
\end{align*}
for any $\delta>0$,
so $\lim_{\delta\rightarrow 0}\delta t_{\delta}'=0$.

The performance of Algorithm \ref{alg:recursive_estimator} is related to its globalized version on state space $\sZ_{\nkkc}$. To illustrate this idea, we define the concept ``sub-chain''. 

Construct transition probability $\overline{\PR}$ on state space $\sZ_{\nkkc}$ satisfying
\begin{align}
	\overline{\PR}(z_{\nkkc}'|z_{\nkkc})
	=\sum_{z_{-\nkkc}}\frac{\overline \pi(z_{\nkkc},z_{-\nkkc})}{\overline \pi_{\nkkc}(z_{\nkkc})} \PR_{\nkkc}(z_{\nkkc}'|(z_{\nkkc},z_{-\nkkc})).
\end{align}
Here 
\begin{align*}
	\PR_{\nkkc}(z_{\nkkc}'|(z_{\nkkc},z_{-\nkkc})=\sum_{z_{-\nkkc}'} \PR(z_{\nkkc}',z_{-\nkkc}'|(z_{\nkkc},z_{-\nkkc}).
\end{align*}
Notice that $\overline \pi_{\nkkc}(z_{\nkkc})>0$ due to Assumption \ref{assump:chain_irreducible}.

\begin{definition}\label{def:sub_chain}
	Let $\mathcal{E}_{\nkkc}\subseteq \nkkc\times \nkkc$ denote the edges with two end points in $\nkkc$. Let $\mu_{\nkkc}'$ denote the marginal initial state distribution of state space $\sZ_{\nkkc}$. 
	
	Then
	Markov chain 
	$\Mkc=(\nkkc,\mathcal{E}_{\nkkc}, \sZ_{\nkkc},\overline{\PR},\supr, \gamma,\mu'_{\nkkc})$ is called the \emph{sub-chain} of Markov chain $(\sN, \mathcal E, \sZ, \PR, \supr, \gamma, \mu')$ with respect to agents $\nkkc$. 
\end{definition}

We denote by $\tilde C: \sZ_{\nkkc}\rightarrow \R$ the sub-chain's cost function. 

The following concept is critical for the reduction of a localized algorithm to its globalized version. 
\begin{definition}\label{def:asymp_err_var} 
	The reduction error is defined as
	\begin{align}\label{eq:asymp_error}
		\epsilon_{red}:=\sup_{z_{\nkkc}\in \sZ_{\nkkc}}\abs{\hat C(z_{\nkkc},w^*)-\tilde C(z_{\nkkc})}.
	\end{align}
\end{definition}
The reduction error is dependent on the actual algorithm used, including update rule, choice of function approximation class.

We now present our main theorem of localized stochastic approximation. The proof sketch is given in subsection \ref{subsec:LCE_sketch}, and the detailed proofs are deferred to section \ref{sec:proof_LCE}.
\begin{theorem}\label{thm:chen_main_var}
	Suppose Assumptions
	\ref{assump:chain_irreducible}, \ref{assump:chen_assump} and
	\ref{assump:hatC_lipschitz}
	are satisfied. 
	Choose stepsize $\alpha$ satisfying $\alpha t_\alpha' \leq \min\Bp{\frac{1}{4L_1'}, \frac{c_0}{114L_1'^2}}$, where $L_1'=\frac{L_1}{1-\lambda}$. 
	Then we have for all $\supepoch\geq t_\alpha'$
	\begin{align*}
		\E\left[\sup_z \abs{\hat C(z_{\nkkc},w(\supepoch))-C(z)}\right] 
		\leq &3\left[L_2^2\left(c_1(1-c_0\alpha)^{\supepoch-t_\alpha'}+c_2\frac{\alpha t_{\alpha}'}{c_0}\right) \right. \\
		&+ \left.\epsilon_{red}^2
		+ \Sp{\frac{\gamma^{\kappa_c-\kappa_r+1}}{1-\gamma}}^2\right].
	\end{align*}
	Here $c_1=(\|w(0)\|+\|w(0)-w^*\|+1)^2$, 
	$c_2=114L_1'^2(\norm{w^*}+1)^2$,
	$\epsilon_{red}$ is defined in Definition \ref{def:asymp_err_var}. 
\end{theorem}
Notice that $t_{\delta}'=O(\log \frac{1}{\delta})$. So $\alpha t_\alpha' \rightarrow 0$ when $\alpha \rightarrow 0$. Thus the conditions for $\alpha$ in Theorem \ref{thm:chen_main_var} can be satisfied when $\alpha$ is small enough. 

\subsection{Proof of Theorem \ref{thm:chen_main_var}}\label{subsec:LCE_sketch}

\paragraph{Key Idea: A Reduction Approach}
The key to analyzing the localized stochastic approximation algorithm is to reduce it to a globalized policy evaluation algorithm on the sub-chain with respect to agents $\nkkc$. To be specific, Algorithm \ref{alg:recursive_estimator} itself can be regarded as a globalized policy evaluation algorithm for the sub-chain $\Mkc$.

The most important part of the proof is that the sub-chain has the following properties:
\begin{enumerate}
	\item The ``sub-chain'' is an aperiodic and irreducible Markov chain. 
	\item The local transition probabilities of agents in $N_{\agentk}^{\kappa_c-1}$ are the same for the sub-chain and the original chain. The local transition probabilities in the sub-chain do not need to be independent.
	\item The stationary distribution of sub-chain equals the marginal stationary distribution of $\sZ_{\nkkc}$ in the original chain.
\end{enumerate}

We formulate the properties as the three lemmas below. The proofs are deferred to Appendix \ref{subsec:proof_subchain}.
\begin{lemma}\label{lemma:tilde_P_ir_ap}
	$\overline {\PR}$ is aperiodic and irreducible.
\end{lemma}

\begin{lemma}\label{lem:loc_prob_sub}
	We have for any agent $i\in N_{\agentk}^{\kappa_c-1}$ that
	\begin{align}
		\overline{\PR}_i(z_i'|z_{\nikc})=\PR_i(z_i'|z_{N_i}), \forall z_i'\in \sZ_i, z_{\nkkc}\in \sZ_{\nkkc}.   
	\end{align}
\end{lemma}
Notice that $N_i\subseteq \nkkc$ when $i\in N_{\agentk}^{\kappa_c-1}$.

\begin{lemma}\label{lem:interpret_overline_P}
	Consider the stationary Markov chain $(\sN, \mathcal E, \sZ, \PR, \supr, \gamma, \overline \pi)$. That is, the initial state distribution
	$\mu'=\overline \pi$ is the stationary distribution. 
	
	Then for any $t\in \N$,  $z_{\nkkc},z_{\nkkc}'\in \sZ_{\nkkc}$, we have
	\begin{align*}
		\Pr\left[z_{\nkkc}(t+1)=z_{\nkkc}'|z_{\nkkc}(t)=z_{\nkkc}\right]= \overline{\PR}(z_{\nkkc}'|z_{\nkkc}).
	\end{align*}
\end{lemma}

Now we can discuss the proof of Theorem \ref{thm:chen_main_var}. We can think that Algorithm \ref{alg:recursive_estimator} is executed on the local chain $\Mkc$, with cost function denoted by $\tilde C$. Then $w^*$ can be regarded as the fix point of parameter update in a globalized algorithm. In this regard, we can do the following error decomposition.
\begin{align}
	&\E\left[\sup_z \abs{\hat C(z_{\nkkc},w(t))-C(z)}^2\right]    \nonumber\\
	\leq & \E\left[\sup_z \left(\abs{\hat C(z_{\nkkc},w(t))-\hat C(z_{\nkkc}, w^*)}+ \abs{\hat C(z_{\nkkc},w^*)-\tilde C(z_{\nkkc})} \right.\right. \nonumber \\
	&+\left.\left.\abs{\tilde C(z_{\nkkc})-C(z)}\right)^2\right]\nonumber\\
	\osi{\leq} & 3\left\{\E\left[\sup_z \abs{\hat C(z_{\nkkc},w(t))-\hat C(z_{\nkkc}, w^*)}^2\right] +\E\left[\sup_z \abs{\hat C(z_{\nkkc},w^*)-\tilde C(z_{\nkkc})}^2\right] \right. \nonumber \\
	&+\left. \E\left[\sup_z \abs{\tilde C(z_{\nkkc})-C(z)}^2\right]\right\} \nonumber\\
	= &3\left\{\underbrace{\E\left[\sup_{z_\nkkc} \abs{\hat C(z_{\nkkc},w(t))-\hat C(z_{\nkkc}, w^*)}^2\right]}_{(a)}  +\underbrace{\sup_{z_{\nkkc}} \abs{\hat C(z_{\nkkc},w^*)-\tilde C(z_{\nkkc})}^2}_{(b)} \right.\nonumber \\
	&+\left.\underbrace{\sup_z \abs{\tilde C(z_{\nkkc})-C(z)}^2}_{(c)}\right\}. \label{eq:err_decomp}
\end{align}
Here $\ri$ is by Cauchy-Schwarz inequality. 

To interpret the error terms, 
$(a)$ is related to the convergence of $w(t) $ to the fix point $w^*$, which can be viewed as a globalized non-linear stochastic approximation problem \citep{chen2019nonlinearSA}. 
$(b)$ is the inherent property of the globalized algorithm and the Markov chain, which is defined as $\epsilon_{red}$ in Definition \ref{def:asymp_err_var}. We point out that for some algorithms, such as TD($\lambda$) with linear function approximation, $\epsilon_{red}$ can be bounded by the function approximation error. 
$(c)$ is the difference of the cost function in the sub-chain and in the original chain, which originates from the use of a localized algorithm and decays exponentially with $\kappa_c$. 

\paragraph{Bounding $(a)$} This can be done in two steps. The first step is to analyze the convergence rate of $w(t)$ to the stationary point $w^*$, which can be viewed as a stochastic approximation problem. Thus we can adopt Lyapunov approach to bound the convergence rate, which is a standard method in stochastic approximation \citep{Srikant2019FiniteTimeEB,chen2019nonlinearSA}.

The second step is to combine the Lipschitz assumption of $\hat C$ (Assumption \ref{assump:hatC_lipschitz}) with convergence result of $w(t)$ and derive the bound of $(a)$. 

With the proof idea above, we proceed to bound the convergence rate of $w(t)$. In order to utilize the mixing of function $G(\zeta(t),X(t),w(t))$, we take expectation conditioned on $X(t-t_{\alpha}')$, $\zeta(t-t_\alpha')$ and $w(t-t_\alpha')$. For convenience of notation, we use $\E_{\alpha}[\cdot]=\E[\cdot|X(t-t_{\alpha}'),\zeta(t-t_\alpha'),w(t-t_\alpha')]$. 

For any $t\geq t_{\alpha}'$, we have
\begin{align}
	&\E_{\alpha}\left[\norm{w(t+1)-w^*}^2\right]-\E_{\alpha}\left[\norm{w(t)-w^*}^2\right] \nonumber \\
	=&2\E_{\alpha}\left[(w(t)-w^*)^\top(w(t+1)-w(t))\right]+\E_{\alpha}\left[(w(t+1)-w(t))^2\right] \nonumber \\
	=&2\alpha\underbrace{\E_{\alpha}\left[(w(t)-w^*)^\top \overline G(w(t)\right]}_{(a_1)} \nonumber  \\
	&+2\alpha\underbrace{\E_{\alpha}\left[(w(t)-w^*)^\top(G(\zeta(t),X(t),w(t))- \overline G(w(t))\right]}_{(a_2)} \nonumber \\
	&+ \alpha^2 \underbrace{\E_{\alpha}\left[\norm{G(\zeta(t),X(t),w(t))}^2\right]}_{(a_3)} \label{eq:w_itr}.   
\end{align}
Term $(a_1)$ corresponds to the negative drift, and Assumption \ref{assump:chen_assump} indicates that
\begin{align*}
	(a_1)\leq -c_0\E_{\alpha}[\norm{w(t)-w^*}^2], 
\end{align*}
Before we bound $(a_2)$ and $(a_3)$, we show that function $G$ is Lipschitz. See Appendix \ref{subsubsec:proof_G_Lip} for the proof.
\begin{lemma}\label{lem:G-Lipschitz}
	Let $L_1'=\frac{L_1}{1-\lambda}$. Then we have for all $t,x$,
	\begin{align*}
		&\norm{G(\zeta(t),x,w_1)-G(\zeta(t),x,w_2)}\leq L_1'\norm{w_1-w_2}, \forall w_1,w_2 \\
		&\norm{G(\zeta(t),x,0)}\leq L_1'.
	\end{align*}
\end{lemma}
Besides, Lemma \ref{lem:G-Lipschitz} also implies for any $t,x,w,w_1,w_2$ that 
\begin{align*}
	& \norm{\overline G(w_1)-\overline G(w_2)}\leq \overline \E\Mp{\norm{G(\zeta(t),X(t),w_1)-G(\zeta(t),X(t),w_2)}}\leq L_1'\norm{w_1-w_2} \\
	& \norm{G(\zeta(t),x,w)}\leq  \norm{G(\zeta(t),x,w)-G(\zeta(t),x,0)}+\norm{G(\zeta(t),x,0)}\leq L_1'(\norm{w}+1) \\
	& \norm{\overline G(w)}\leq \overline \E[\norm{G(\zeta(t),X(t),w]}\leq L_1'(\norm{w}+1).
\end{align*}
We apply  Lemma \ref{lem:G-Lipschitz} and bound term $(a_3)$ as
\begin{align*}
	(a_3)=&\E_{\alpha}\left[\norm{G(\zeta(t),X(t),w(t))}^2\right]\\
	\leq &L_1'^2\E_{\alpha}\left[(\norm{w(t)}+1)^2\right] \\
	\leq &L_1'^2\E_{\alpha}\left[(\norm{w(t)-w^*}+\norm{w^*}+1)^2\right] \\
	\leq &2 L_1'^2\E_{\alpha}\left[\norm{w(t)-w^*}^2+(\norm{w^*}+1)^2\right].
\end{align*}

Finally we bound $(a_2)$. Before that, we need to control the difference of $w(t_1)$ and $w(t_2)$ for any $t_1,t_2$ under certain conditions of stepsize. This can be formulated as the lemma below, with proof deferred to Appendix \ref{subsubsec:proof_para_ch}.
\begin{lemma}\label{lem:param_change}
	For any $t_1>t_2\geq 0$, if $\alpha (t_1-t_2)\leq \frac{1}{4L_1'}$, then we have
	\begin{align*}
		\norm{w(t_1)-w(t_2)}
		&\leq 2 L_1'\alpha(t_1-t_2)(\norm{w(t_2)}+1) \\
		&\leq 4L_1'\alpha(t_1-t_2)(\norm{w(t_1)}+1).
	\end{align*}
\end{lemma}

With Lemma \ref{lem:param_change}, we are able to control $(a_2)$. The proof is given in Appendix \ref{subsubsec:proof_a2}.
\begin{lemma}\label{lem:bound_a2}
	Suppose $\alpha t_{\alpha}'\leq \frac{1}{4L_1'}$. Then the following inequality holds for all $t\geq t_{\alpha}'$:
	\begin{align*}
		(a_2)\leq 56L_1'^2 \alpha t_{\alpha}'\E_{\alpha}\Mp{\norm{w(t)-w^*}^2+(\norm{w^*}+1)^2}.
	\end{align*}
\end{lemma}
Now we apply the upper bounds for $(a_1)$, $(a_2)$ and $(a_3)$ to Eq. (\ref{eq:w_itr}), and we get 
\begin{align*}
	&\E_{\alpha}\left[\norm{w(t+1)-w^*}^2\right]-\E_{\alpha}\left[\norm{w(t)-w^*}^2\right] \\
	\leq &2\alpha (a_1)+ 2\alpha (a_2) + \alpha^2 (a_3) \\
	\leq &-2c_0 \alpha \E_{\alpha}[\norm{w(t)-w^*}^2] \\
	&+ 112L_1'^2 \alpha^2 t_{\alpha}'\E_{\alpha}\Mp{\norm{w(t)-w^*}^2+(\norm{w^*}+1)^2} \\
	&+  2L_1'^2 \alpha^2\E_{\alpha}\left[\norm{w(t)-w^*}^2+(\norm{w^*}+1)^2\right] \\
	\leq &-2c_0 \alpha \E_{\alpha}[\norm{w(t)-w^*}^2] \\
	&+ 114L_1'^2 \alpha^2 t_{\alpha}'\E_{\alpha}\Mp{\norm{w(t)-w^*}^2+(\norm{w^*}+1)^2}.
\end{align*}
Rearranging the terms, we derive that
\begin{align*}
	&\E_{\alpha}\left[\norm{w(t+1)-w^*}^2\right] \\
	\leq &\Sp{1-2c_0\alpha +114L_1'^2 \alpha^2 t_{\alpha}'}\E_{\alpha}\left[\norm{w(t)-w^*}^2\right] +114L_1'^2 \alpha^2 t_{\alpha}' (\norm{w^*}+1)^2 \\
	\leq & (1-c_0\alpha)\E_{\alpha}\left[\norm{w(t)-w^*}^2\right] +114L_1'^2 \alpha^2 t_{\alpha}' (\norm{w^*}+1)^2,
\end{align*}
where the last inequality is due to $\alpha t_{\alpha}'\leq \frac{c_0}{114L_1'^2}$. Taking total expectation on both sides, we get 
\begin{align*}
	\E\left[\norm{w(t+1)-w^*}^2\right] 
	\leq  (1-c_0\alpha)\E\left[\norm{w(t)-w^*}^2\right] +114L_1'^2 \alpha^2 t_{\alpha}' (\norm{w^*}+1)^2.
\end{align*}

Repeatedly use the above inequality starting from $t_{\alpha}'$, and we have 
\begin{align*}
	\E\left[\norm{w(t)-w^*}^2\right] \leq & (1-c_0\alpha)^{t-t_{\alpha}'} \E\left[\norm{w(t_{\alpha}')-w^*}^2\right] \\
	&+ 114L_1'^2 \alpha^2 t_{\alpha}' (\norm{w^*}+1)^2 \sum_{k=0}^{t-t_{\alpha}'-1} (1-c_0\alpha)^{k} \\
	\leq & (1-c_0\alpha)^{t-t_{\alpha}'} \E\left[\norm{w(t_{\alpha}')-w^*}^2\right]+\frac{ 114L_1'^2 \alpha t_{\alpha}'(\norm{w^*}+1)^2}{c_0}.
\end{align*}
We can bound $\E\Mp{\norm{w(t_{\alpha}')-w^*}^2}$ by
\begin{align*}
	\E\Mp{\norm{w(t_{\alpha}')-w^*}^2}
	\leq &\E\Mp{\Sp{\norm{w(t_{\alpha}')-w(0)}+\norm{w(0)-w^*}^2}} \\
	\osi{\leq} & \E\Mp{\Sp{\norm{w(0)}+\norm{w(0)-w^*}+1}^2} \\
	= & c_1,
\end{align*}
where $\ri$ is by Lemma \ref{lem:param_change}, with $t_1=t_{\alpha}'$ and $t_2=0$, and by $\alpha t_{\alpha}'\leq \frac{1}{4L_1'}$.

Noticing that $c_2=114L_1'^2(\norm{w^*}+1)^2$, we substitute $t$ with $\supepoch$ and get 
\begin{align}
	\E\left[\norm{w(\supepoch)-w^*}^2\right] 
	\leq c_1(1-c_0\alpha)^{\supepoch-t_{\alpha}'}+c_2\frac{\alpha t_{\alpha}'}{c_0}, \; \forall \supepoch\geq t_{\alpha}'.
	\label{eq:w_converge}
\end{align}

For the second step, we combine Assumption \ref{assump:hatC_lipschitz} with Eq. (\ref{eq:w_converge}) and get the following result, the proof of which is in Appendix \ref{subsubsec:proof_a}.

\begin{lemma}\label{lem:bound_a}
	For $\supepoch > t_\alpha'$, we have
	\begin{align*}
		(a)\leq L_2^2\left(c_1(1-\alpha c_0)^{\supepoch-t_\alpha'}+c_2\frac{\alpha t_{\alpha}'}{c_0}\right).
	\end{align*}
\end{lemma}

\paragraph{Bounding $(b)$}
We have $(b)=\epsilon_{red}^2$ by Definition \ref{def:asymp_err_var}.

\paragraph{Bounding $(c)$} This can be derived by the exponential decay property of the cost function. See Appendix \ref{subsubsec:proof_exp_decay} for the proof.

\begin{lemma}\label{lemma:exp_decay_Q} The cost function of the sub-chain and the original chain differs by 
	\begin{align*}
		\sup_z \abs{\tilde C(z_{\nkkc})-C(z)}\leq \frac{1}{1-\gamma}\gamma^{\kappa_c-\kappa_r+1}.
	\end{align*}
	Thus $(c)\leq \Sp{\frac{\gamma^{\kappa_c-\kappa_r+1}}{1-\gamma}}^2$.
\end{lemma}

Eventually. we combine bounds for $(a)$, $(b)$, $(c)$ (Lemma \ref{lem:bound_a}, Definition \ref{def:asymp_err_var} and Lemma \ref{lemma:exp_decay_Q}) and plug into Eq. (\ref{eq:err_decomp}). Then we complete the proof.

\subsection{Proofs of Technical Lemmas in Appendix \ref{sec:local_policy_eval} }\label{sec:proof_LCE}
\subsubsection{Properties of Sub-Chain}\label{subsec:proof_subchain}
We give the proofs of Lemma \ref{lemma:tilde_P_ir_ap}, \ref{lem:loc_prob_sub}, \ref{lem:interpret_overline_P}. 
\begin{proof}[Proof of Lemma \ref{lemma:tilde_P_ir_ap}]
	For any Markov chain with transition probability $\Gamma$ on some state space $\mathcal{X}$, we write $\Gamma^k(x'|x)=\Pr[x(t+k)=x'|x(t)=x$, for any $t\in \N$, $x,x'\in \mathcal{X}$. 
	
	For any 
	$z_{\nkkc}, z_{\nkkc}'\in \mathcal Z_{\nkkc}$, we randomly pick $z_{-\nkkc}, 
	z_{-\nkkc}'\in \mathcal Z_{-\nkkc}$ and let $z=(z_{\nkkc},z_{-\nkkc})$, 
	$z'=(z_{\nkkc}',z_{-\nkkc}')$. Since $\PR$ is irreducible, there exists $\agentk>0$, such that 
	$\PR^k(z'|z)>0$, so $\PR^k_{\nkkc}(z_{\nkkc}'|z)>0$. $\PR^k_{\nkkc}(z_{\nkkc}'|z)$ represents the marginal probability of $z_{\nkkc}'$ given previous state $z$. By the interpretation of $\overline{\PR}$ given in Lemma \ref{lem:interpret_overline_P},
	\begin{align*}
		\overline{\PR}^k(z_{\nkkc}'|z_{\nkkc})
		&=\sum_{\hat z_{-\nkkc}\in \sZ_{-\nkkc}} 
		\frac{\overline\pi(z_{\nkkc},\hat z_{-\nkkc})}{\overline \pi_{\nkkc}(z_{\nkkc})}
		{\PR}_{\nkkc}^k(z_{\nkkc}'|z_{\nkkc}, \hat z_{-\nkkc}) \\
		&\geq \frac{\overline\pi(z)}{\overline \pi_{\nkkc}(z_{\nkkc})}
		{\PR}_{\nkkc}^k(z_{\nkkc}'|z) \\
		&>0.
	\end{align*}
	Therefore, $\overline{\PR}$ is irreducible.
	
	To show that $\overline{\PR}$ is aperiodic, we assume that $\overline{\PR}$ has period $T\geq 2$. Then for any 
	$\agentk$ not divisible by $T$ and any $z_{\nkkc}$, $\overline{\PR}^k(z_{\nkkc}|z_{\nkkc})=0$. For any 
	$z\in\mathcal Z$, since for  $\agentk$ not divisible by $T$, 
	\begin{align*}
		0=\overline{\PR}^k(z_{\nkkc}|z_{\nkkc})=\sum_{\hat z_{-\nkkc}\in \mathcal Z_{-\nkkc}} 
		\frac{\overline\pi(z_{\nkkc},\hat z_{-\nkkc})}{\pi_{\nkkc}^*(z_{\nkkc})}
		{\PR}_{\nkkc}^k(z_{\nkkc}|(z_{\nkkc}, \hat z_{-\nkkc})),
	\end{align*}
	we have $ {\PR}_{\nkkc}^k(z_{\nkkc}|(z_{\nkkc}, \hat z_{-\nkkc}))=0$ for any 
	$\hat z_{-\nkkc}\in \mathcal Z_{-\nkkc}$. In particular, 
	$\overline{\PR}_{\nkkc}^k(z_{\nkkc}|z)=0$ and thus $\overline{\PR}^k(z|z)=0$. This implies that the 
	period of state $z$ is at least $T\geq 2$, which contradicts the assumption that $\PR$ is aperiodic. Hence 
	$\overline{\PR}$ is aperiodic.
	
	In conclusion, $\overline{\PR}$ is irreducible and aperiodic.
\end{proof}

\begin{proof}[Proof of Lemma \ref{lem:loc_prob_sub}]
	For simplicity, let $I=\nkkc/\{i\}$. Then we have
	\begin{align*}
		\overline{\PR}_i(z_i'|z_{\nkkc})
		=&\sum_{z_I'} \overline{\PR}(z_i',z_I'|z_{\nkkc}) \\
		=&\sum_{z_I'} \sum_{z_{-\nkkc}}\frac{\overline \pi(z_{\nkkc},z_{-\nkkc})}{\overline \pi_{\kappa_c}(z_{\nkkc})}\PR_{\nkkc}(z_{\nkkc}'|(z_{\nkkc},z_{-\nkkc})) \\
		=& \sum_{z_{-\nkkc}}\frac{\overline\pi(z_{\nkkc},z_{-\nkkc})}{\overline\pi_{\kappa_c}(z_{\nkkc})}\sum_{z_I'}\PR_{\nkkc}(z_{\nkkc}'|(z_{\nkkc},z_{-\nkkc})) \\
		\mathop{=}\limits^{(\romannumeral1)}& \sum_{z_{-\nkkc}}\frac{\overline\pi(z_{\nkkc},z_{-\nkkc})}{\overline\pi_{\kappa_c}(z_{\nkkc})}\PR_{i}(z_i'|z_{N_i}) \\
		\mathop{=}\limits^{(\romannumeral2)}&  \PR_{i}(z_i'|z_{N_i}) \sum_{z_{-\nkkc}}\frac{\overline\pi(z_{\nkkc},z_{-\nkkc})}{\overline\pi_{\kappa_c}(z_{\nkkc})} \\
		=& \PR_{i}(z_i'|z_{N_i}).
	\end{align*}
	Here $(\romannumeral1)$ uses the fact that the local transition probability $\PR_i$ only depends on the states of agents in $N_i$, and $(\romannumeral2)$ is because $N_i\cap -\nkkc=\varnothing$.
\end{proof}

\begin{proof}[Proof of Lemma \ref{lem:interpret_overline_P}]
	Since we consider the stationary Markov chain, the state distribution at any time $t$ is equal to the stationary state distribution $\overline \pi$, so we have
	\begin{align*}
		&\Pr\left[z_{\nkkc}(t+1)=z_{\nkkc}'|z_{\nkkc}(t)=z_{\nkkc}\right] \\
		=& \frac{\Pr\left[z_{\nkkc}(t+1)=z_{\nkkc}',z_{\nkkc}(t)=z_{\nkkc}\right]}{\Pr\left[z_{\nkkc}(t)=z_{\nkkc}\right]} \\
		=& \frac{\sum_{z_{-\nkkc}}\Pr\left[z_{\nkkc}(t+1)=z_{\nkkc}',z(t)=(z_{\nkkc},z_{-\nkkc})\right]}{\overline \pi_{\nkkc}(z_{\nkkc})} \\
		=& \frac{\sum_{z_{-\nkkc}}\overline \pi(z_{\nkkc},z_{-\nkkc})\PR_{\nkkc}(z_{\nkkc}'|(z_{\nkkc},z_{-\nkkc}))}{\overline \pi_{\nkkc}(z_{\nkkc})} \\
		=& \overline{\PR}(z_{\nkkc}'|z_{\nkkc}).
	\end{align*}
\end{proof}

\subsubsection{Geometric mixing of the function $G$}
\label{subsubsec:decay_G}
We prove the geometric mixing property of function $G$ (cf. Eq. (\ref{eq:G_decay})), which can be formalized as the lemma below:
\begin{lemma}\label{lemma:chen_lm_2_2}
	There exists $c_g(c',\rho',\lambda,t_0)>0$ and $\rho_g(\rho',\lambda)\in (0,1)\in (0,1)$, such that for any $t\geq t_0$, we have
	\begin{align*}
		\norm{\E[G(\zeta(t),X(t),w)]-\overline G(w)}
		\leq L_1(\|w\|+1)c_g(c',\rho',\lambda,t_0)[\rho_g(\rho',\lambda)]^t.
	\end{align*}
\end{lemma}

To prove Lemma \ref{lemma:chen_lm_2_2}, we need some auxiliary results. Let $Y^{t_1}(t)=(z(t-t_1),\cdots,z(t+1))\in \mathcal Y^{t_1}=\sZ^{t_1+2}$. Denote by $\pi_{Y,t_1,t}\in \Delta(\mathcal Y^{t_1})$ the distribution of $Y^{t_1}(t)$ and by $\overline\pi_{Y,t_1}\in \Delta(\mathcal Y^{t_1})$ the corresponding stationary distribution. 

$\pi_{Y,t_1,t}$ and $\overline\pi_{Y,t_1}$ can be computed by 
\begin{align}
	\pi_{Y,t_1,t}(z_{t-t_1},\cdots,z_{t+1})&=\pi_{t-t_1}(z_{t-t_1})\prod_{i=t-t_1}^{t} \PR(z_{i+1}|z_i) \label{eq:comp_Y} \\
	\overline \pi_{Y}(z_0,\cdots,z_{t_1+2})&=\overline\pi(z_0)\prod_{i=0}^{t+1} \PR(z_{i+1}|z_i). \label{eq:comp_Y_stat} 
\end{align}

The following lemma states the convergence rate of $\pi_{Y,t_1,t}$.
\begin{lemma}\label{lemma:mixing_rate_X_Y}
	For any $t\geq t_1$, $\TV(\pi_{Y,t_1,t},\overline\pi_{Y,t_1})\leq c'\rho'^{t-t_1}$.
\end{lemma}

\begin{proof}[Proof of Lemma \ref{lemma:mixing_rate_X_Y}]
	We have by Eqs. (\ref{eq:comp_Y}) and (\ref{eq:comp_Y_stat})
	\begin{align*}
		\TV(\pi_{Y,t_1,t},\overline\pi_{Y,t_1})
		&=\frac12\sum_{z_{t-t_1},\cdots,z_{t+1}}\abs{\pi_{Y,t_1,t}(z_{t-t_1},\cdots,z_{t+1})-\overline\pi_Y(z_{t-t_1},\cdots,z_{t+1})} \\
		&=\frac12\sum_{z_{t-t_1},\cdots,z_{t+1}}\abs{\pi_{t-t_1}(z_{t-t_1})-\overline\pi(z_{t-t_1})}\prod_{i=t-t_1}^t\left[\PR(z_{i+1}|z_i)\right]| \\
		&= \frac12\sum_{z_{t-t_1}}\abs{\pi_{t-t_1}(z_{t-t_1})-\overline\pi(z_{t-t_1})}\\
		&=\TV(\pi_{t-t_1},\overline\pi) \\
		&\leq c'\rho'^{t-t_1}.
	\end{align*}
\end{proof}

\begin{lemma}\label{lem:one_term_err}
	For $t\geq m$, any $w$ and any $X(t)=(z_{\nkkc}(t-t_0), \cdots,z_{\nkkc}(t+1))$, we have
	\begin{align*}
		&\abs{\E[\supphi(z(t-m))F(X(t),w)]-\overline\E[\supphi(z(t-m))F(X(t),w)]} \\
		\leq & 2L_1(\|w\|+1)c'\rho'^{t-\max\{m,t_0\}}.
	\end{align*}
\end{lemma}
\begin{proof}[Proof of Lemma \ref{lem:one_term_err}]
	Let set $\mathcal S^{t_1}=\{(z(t-t_1), \cdots,z(t+1))\in \mathcal Y^{t_1} \mid (z_{\nkkc}(t-t_1), \cdots,z_{\nkkc}(t+1)=X(t))\}$, for any $t_1\in \N$.
	
	When $m\geq t_0$, we have
	\begin{align*}
		&\norm{\E[\supphi(z(t-m))F(X(t),w)]-\overline\E[\supphi(z(t-m))F(X(t),w)]}  \\
		= & \norm{\sum_{Y^m\in \mathcal S^m} \supphi(z(t-m))F(X(t),w) (\pi_{Y,m,t}(Y^m)-\overline\pi_{Y,m}(Y^m))} 
		\\
		\leq &\sum_{Y^m\in \mathcal S^m} \norm{\supphi(z(t-m))}\abs{F(X(t),w)} \abs{\pi_{Y,m,t}(Y^m)-\overline\pi_{Y,m}(Y^m))} \\
		\leq & L_1(\norm{w}+1) \sum_{Y^m\in \mathcal S^m} \abs{\pi_{Y,m,t}(Y^m)-\overline\pi_{Y,m}(Y^m))} \\
		= & 2L_1(\norm{w}+1) \TV(\pi_{Y,m,t},\overline\pi_{Y,m}) \\
		\leq & 2L_1(\norm{w}+1) c'\rho'^{t-m}.
	\end{align*}
	Here the last inequality is due to Lemma \ref{lemma:mixing_rate_X_Y}.
	
	We can similarly prove the case for $m<t_0$. In that case, $z(t-m)$ is included in $(z(t-t_0),\cdots,z(t+1)$, so we just need to replace $m$ with $t_0$ in the proof above.
\end{proof}

Now we can bound the convergence rate of function $G$.

\begin{proof}[Proof of Lemma \ref{lemma:chen_lm_2_2}]
	We have by definition of function $G$ and $\overline G$ that
	\begin{align*}
		&\norm{\E[G(\zeta(t),X(t),w)]-\overline G(w)} \\
		=& \norm{\sum_{m=0}^t \lambda^m \E[\supphi(z(t-m))F(X(t),w)]-\sum_{m=0}^{\infty} \lambda^m \overline\E[\supphi(z(t-m))F(X(t),w)]} \\
		\leq & \sum_{m=0}^t \lambda^m \norm{\E[\supphi(z(t-m))F(X(t),w)]-\overline\E[\supphi(z(t-m))F(X(t),w)]} \\
		+&\sum_{m=t+1}^{\infty}\lambda^m\norm{ \overline\E[\supphi(z(t-m))F(X(t),w)} .
	\end{align*}
	Notice that Assumption \ref{assump:chen_assump} indicates that
	\begin{align*}
		\norm{\overline\E[\supphi(z(t-m))F(X(t),w)]}\leq L_1(\|w\|+1),
	\end{align*}
	and Lemma \ref{lem:one_term_err} implies that
	\begin{align*}
		&\norm{\E[\supphi(z(t-m))F(X(t),w)]-\overline\E[\supphi(z(t-m))F(X(t),w)]} \\
		\leq& 2L_1(\|w\|+1)c'\rho'^{t-\max\{t_0,m\}}.
	\end{align*}
	
	Plug in the two bounds back, and we have 
	\begin{align}
		&\|\E[G(\zeta(t),X(t),w)]-\overline G(w)\| \nonumber \\
		\leq & \sum_{m=0}^{t_0-1}\lambda^m 2L_1(\|w\|+1)c'\rho'^{t-t_0}
		+\sum_{m=t_0}^{t}\lambda^m 2L_1(\|w\|+1)c'\rho'^{t-m} 
		+\sum_{m=t+1}^{\infty}\lambda^m L_1(\|w\|+1) \nonumber \\
		=& L_1(\|w\|+1)\Mp{2c'\underbrace{\sum_{m=0}^{t_0-1}\lambda^m \rho'^{t-t_0}}_{(a)}+2c'\underbrace{\sum_{m=t_0}^{t}\lambda^m \rho'^{t-m}}_{(b)}+\underbrace{\sum_{m=t+1}^{\infty}\lambda^m}}_{(c)}. \label{eq:G_decomp}
	\end{align}
	where we use Assumption \ref{assump:chen_assump} as well as Lemma \ref{lem:one_term_err}.
	Obviously, $(a)=O(\rho'^t)$ and $(c)=O(\lambda^t)$. 
	For $(b)$, there are three cases:
	\begin{itemize}
		\item $\lambda<\rho'$, then $(b)=O(\rho'^t)$.
		\item $\lambda=\rho'$, then $(b)=(t-t_0+1)\rho'^t$, which also decays exponentially fast with $t$, with any decay rate less than $\rho'$.
		\item $\lambda>\rho'$, then $(b)=\sum_{m'=0}^{t-t_0}\lambda^{t-m'}\rho'^{m}=O(\lambda^t)$.
	\end{itemize}
	In conclusion, $(a)$, $(b)$, and $(c)$ all decays exponentiallly fast with $t$, with decay rates depending only on $\lambda$ or $\rho'$. Let $\rho'(\rho',\lambda)$ be the maximum value among the three decay rates, then there exists some $c_g(c',\rho',\lambda,t_0)>0$, such that
	\begin{align*}
		2c'\sum_{m=0}^{t_0-1}\lambda^m \rho'^{t-t_0}+2c'\sum_{m=t_0}^{t}\lambda^m \rho'^{t-m}+\sum_{m=t+1}^{\infty}\lambda^m\leq  c_g(c',\rho',\lambda,t_0) [\rho_g(\rho',\lambda)]^t. 
	\end{align*}
	Plug into Eq. (\ref{eq:G_decomp}), and we complete the proof.
\end{proof}

\subsubsection{Proof of Lemma \ref{lem:G-Lipschitz}}
\label{subsubsec:proof_G_Lip}
The key is to notice that 
\begin{align*}
	\|\zeta(t)\|&\leq\sum_{k=0}^t \lambda^{t-k}\|\supphi(z(k))\| \\
	&\leq \sum_{k=0}^t \lambda^{t-k} \\
	&<\frac{1}{1-\lambda}.
\end{align*}
Then by Assumption \ref{assump:chen_assump}, we get
\begin{align*}
	&\norm{G(\zeta(t),x,w_1)-G(\zeta(t),x,w_2)} \\
	=&\norm{\zeta(t)[F(x,w_1)-F(x,w_2)]} \\
	\leq & \norm{\zeta(t)} \abs{F(x,w_1)-F(x,w_2)} \\
	\leq &\frac{1}{1-\lambda}\cdot L_1\norm{w_1-w_2} \\
	=& L_1'\norm{w_1-w_2}.
\end{align*}
Furthermore, we have
\begin{align*}
	&\norm{G(\zeta(t),x,0)} \\
	=&\norm{\zeta(t)F(x,0)} \\
	\leq & \norm{\zeta(t)}\abs{F(x,0)} \\
	\leq &\frac{L_1}{1-\lambda} \\
	=& L_1'.
\end{align*}

\subsubsection{Proof of Lemma \ref{lem:param_change}}
\label{subsubsec:proof_para_ch}
By Lemma \ref{lem:G-Lipschitz}, we have for ant $t$ that
\begin{align*}
	\norm{w(t)}-\norm{w(t-1)}\leq &\norm{w(t)-w(t-1)} \\
	= & \alpha \norm{G(\zeta_{t-1}, X(t-1), w(t-1)} \\
	\leq &\alpha L_1' (\norm{w(t-1)}+1).
\end{align*}
So $\norm{w(t)}+1 \leq (1+\alpha L_1')(\norm{w(t-1)}+1)$. Then we have for any $t\geq t_2$,
\begin{align*}
	\norm{w(t)}+1 \leq (1+\alpha L_1')^{t-t_2} (\norm{w(t_2)}+1).
\end{align*}
Therefore, we get
\begin{align*}
	&\norm{w(t_1)-w(t_2)} \\
	\leq & \sum_{t'=t_2}^{t_1-1} \norm{w(t'+1)-w(t')} \\
	\leq & \sum_{t'=t_2}^{t_1-1}\alpha L_1'(\norm{w(t')}+1) \\
	\leq & \sum_{t'=t_2}^{t_1-1}\alpha L_1'(1+\alpha L_1')^{t'-t_2} (\norm{w(t_2)}+1) \\
	=& \left[(1+\alpha L_1')^{t_1-t_2}-1\right](\norm{w(t_2)}+1).
\end{align*}
Notice that $(1+x)^p\leq 1+2px$ for any $p>0$ and $x\in [0,\frac{1}{2p}]$.
So for $\alpha(t_1-t_2)\leq \frac{1}{2L_1'}$, we have 
\begin{align*}
	(1+\alpha L_1')^{t_1-t_2} &\leq 1+2 L_1'\alpha(t_1-t_2).
\end{align*}
Therefore, we have $\norm{w(t_1)-w(t_2)}\leq 2 L_1'\alpha(t_1-t_2)(\norm{w(t_2)}+1)$.

Since $\alpha(t_1-t_2)\leq \frac{1}{4L_1'}$, we further have 
\begin{align*}
	& \norm{w(t_1)-w(t_2)} \\
	\leq & 2 L_1'\alpha(t_1-t_2)(\norm{w(t_2)}+1) \\
	\leq & 2 L_1'\alpha(t_1-t_2)(\norm{w(t_1)-w(t_2)}+\norm{w(t_1)}+1) \\
	\leq & \frac{1}{2} \norm{w(t_1)-w(t_2)} + 2 L_1'\alpha(t_1-t_2)(\norm{w(t_1)}+1).
\end{align*}
Rearrange the terms, and we have $\norm{w(t_1)-w(t_2)}
\leq 4 L_1'\alpha(t_1-t_2)(\norm{w(t_1)}+1)$.

\subsubsection{Proof of Lemma \ref{lem:bound_a2}}
\label{subsubsec:proof_a2}
We can decompose $(a_2)$ as 
\begin{align*}
	(a_2)=&\E_{\alpha}\left[(w(t)-w^*)^\top\Sp{G(\zeta(t),X(t),w(t))- \overline G(w(t))}\right] \\
	= & \E_{\alpha}\left[(w(t)-w(t-t_{\alpha}')^\top\Sp{G(\zeta(t),X(t),w(t))- \overline G(w(t))}\right] \tag{$b_1$}\\
	&+ (w(t-t_{\alpha}')-w^*)^\top\E_{\alpha}\left[G(\zeta(t),X(t),w(t-t_{\alpha}'))- \overline G(w(t-t_{\alpha}'))\right]\tag{$b_2$} \\
	&+ \E_{\alpha}\left[(w(t-t_{\alpha}')-w^*)^\top\Sp{G(\zeta(t),X(t),w(t))- G(\zeta(t),X(t),w(t-t_{\alpha}'))}\right]\tag{$b_3$} \\
	&+\E_{\alpha}\left[(w(t-t_{\alpha}')-w^*)^\top\Sp{\overline G(w(t-t_{\alpha}'))- \overline G(w(t))}\right]. \tag{$b_4$}
\end{align*}
For term $(b_1)$, we apply Lemma \ref{lem:param_change} with $t_1=t$ and $t_2=t-t_{\alpha}'$. Thus we get 
\begin{align}
	(b_1)&\leq  \E_{\alpha}\left[\norm{w(t)-w(t-t_{\alpha}')}\norm{G(\zeta(t),X(t),w(t)- \overline G(w(t)}\right]\nonumber \\
	& \leq \E_{\alpha}\left[\norm{w(t)-w(t-t_{\alpha}')}\Sp{\norm{G(\zeta(t),X(t),w(t)}+ \norm{\overline G(w(t)}}\right] \nonumber \\
	& \leq 4L_1'\alpha t_{\alpha}' \E_{\alpha}\left[(\norm{w(t)}+1)\Sp{\norm{G(\zeta(t),X(t),w(t)}+ \norm{\overline G(w(t)}}\right] \tag{Lemma \ref{lem:param_change}} \nonumber \\
	& \leq 8L_1'^2\alpha t_{\alpha}' \E_{\alpha}\left[(\norm{w(t)}+1)^2\right] \nonumber \\
	& \leq 8L_1'^2\alpha t_{\alpha}' \E_{\alpha}\left[(\norm{w(t)-w^*}+\norm{w^*}+1)^2\right] \nonumber \\
	& \leq 16L_1'^2\alpha t_{\alpha}' \E_{\alpha}\left[\norm{w(t)-w^*}^2+(\norm{w^*}+1)^2\right] \nonumber \\
	& = 16L_1'^2\alpha t_{\alpha}' \Sp{\E_{\alpha}\left[\norm{w(t)-w^*}^2\right]+(\norm{w^*}+1)^2}. \label{eq:b1}
\end{align}
Now we bound $(b_2)$. We have by Definition \ref{def:G_mix} that
\begin{align*}
	(b_2)&\leq \norm{w(t-t_{\alpha}')-w^*} \norm{\E_{\alpha}\left[(G(\zeta(t),X(t),w(t-t_{\alpha}'))- \overline G(w(t-t_{\alpha}')\right]} \\
	&\leq  L_1\alpha (\norm{w(t-t_{\alpha}')}+1)\norm{w(t-t_{\alpha}')-w^*} \\
	&=L_1\alpha \E_{\alpha}\Mp{(\norm{w(t-t_{\alpha}')}+1)\norm{w(t-t_{\alpha}')-w^*}}.
\end{align*}
We further bound $(b_2)$. By Lemma \ref{lem:param_change} and the fact that $\alpha t_{\alpha}'\leq \frac{1}{4L_1'}$, we have 
\begin{align}
	\norm{w(t)-w(t-t_{\alpha}')}\leq 4L_1'\alpha t_{\alpha}' (\norm{w(t)}+1)\leq \norm{w(t)}+1. \label{eq:para_change_cor}
\end{align}
Thus we have 
\begin{align*}
	&(\norm{w(t-t_{\alpha}')}+1)\norm{w(t-t_{\alpha}')-w^*} \\
	\leq &(\norm{w(t)-w(t-t_{\alpha}')}+\norm{w(t)-w^*}+\norm{w^*}+1)(\norm{w(t)-w(t-t_{\alpha}')}+\norm{w(t)-w^*}) \\
	\leq &(\norm{w(t)}+\norm{w(t)-w^*}+\norm{w^*}+2)(\norm{w(t)}+\norm{w(t)-w^*}+1) \\
	\leq & (2\norm{w(t)-w^*}+2\norm{w^*}+2)(2\norm{w(t)-w^*}+\norm{w^*}+1) \\
	\leq & 4(\norm{w(t)-w^*}+\norm{w^*}+1)^2 \\
	\leq & 8\Mp{\norm{w(t)-w^*}^2+(\norm{w^*}+1)^2}.
\end{align*}
So we can bound $(b_2)$ as 
\begin{align}
	(b_2)\leq 8L_1\alpha \E_{\alpha}\Mp{\norm{w(t)-w^*}^2+(\norm{w^*}+1)^2}. \label{eq:b2}
\end{align}

Finally we bound the sum of $(b_3)$ and $(b_4)$. We apply Lemma \ref{lem:G-Lipschitz} and get 
\begin{align}
	(b_3)+(b_4)\leq & 2L_1' \E_{\alpha}\left[\norm{w(t-t_{\alpha}')-w^*}\norm{w(t)-w(t-t_{\alpha}')}\right] \nonumber \\
	\leq & 8L_1'^2\alpha t_{\alpha}' \E_{\alpha}\Mp{\norm{w(t-t_{\alpha}')-w^*}(\norm{w(t)}+1)} \tag{Lemma \ref{lem:param_change}} \nonumber \\
	\leq & 8L_1'^2\alpha t_{\alpha}' \E_{\alpha}\Mp{(\norm{w(t)-w(t-t_{\alpha}')}+\norm{w(t)-w^*})(\norm{w(t)-w^*}+\norm{w^*}+1)} \nonumber \\
	\leq & 8L_1'^2\alpha t_{\alpha}' \E_{\alpha}\Mp{(\norm{w(t)}+\norm{w(t)-w^*}+1)(\norm{w(t)-w^*}+\norm{w^*}+1)} \tag{Eq. (\ref{eq:para_change_cor})} \nonumber \\
	\leq & 8L_1'^2\alpha t_{\alpha}' \E_{\alpha}\Mp{(2\norm{w(t)-w^*}+\norm{w^*}+1)(\norm{w(t)-w^*}+\norm{w^*}+1)} \nonumber \\
	\leq & 16L_1'^2\alpha t_{\alpha}' \E_{\alpha}\Mp{(\norm{w(t)-w^*}+\norm{w^*}+1)^2} \nonumber \\
	\leq & 32L_1'^2\alpha t_{\alpha}' \E_{\alpha}\Mp{\norm{w(t)-w^*}^2+(\norm{w^*}+1)^2}. \label{eq:b34}
\end{align}
Now we combine the bounds for $(b_1),(b_2),(b_3),(b_4)$ (Eq. (\ref{eq:b1}), (\ref{eq:b2}), (\ref{eq:b34})) and get 
\begin{align*}
	(a_2)=&(b_1)+(b_2)+(b_3)+(b_4) \\
	\leq &16L_1'^2\alpha t_{\alpha}' \Sp{\E_{\alpha}\left[\norm{w(t)-w^*}^2\right]+(\norm{w^*}+1)^2} \\
	&+8L_1\alpha \E_{\alpha}\Mp{\norm{w(t)-w^*}^2+(\norm{w^*}+1)^2} \\
	&+32L_1'^2\alpha t_{\alpha}' \E_{\alpha}\Mp{\norm{w(t)-w^*}^2+(\norm{w^*}+1)^2} \\
	\leq & 56L_1'^2 \alpha t_{\alpha}'\E_{\alpha}\Mp{\norm{w(t)-w^*}^2+(\norm{w^*}+1)^2}, 
\end{align*}
where the last inequality is derived by $1<L_1<L_1'$ (cf. Assumption \ref{assump:chen_assump}) and $t_{\alpha}'\geq 1$. 

\subsubsection{Proof of Lemma \ref{lem:bound_a}}
\label{subsubsec:proof_a}
By Assumption \ref{assump:hatC_lipschitz}, we have for any $t$
\begin{align*}
	\abs{\hat C(z_{\nkkc},w(t))-\hat C(z, w^*)}
	\leq & L_2 \norm{w(t)-w^*}.
\end{align*}
So we get for $\supepoch\geq t_{\alpha}'$ that
\begin{align*}
	&\E\left[\sup_z \abs{\hat C(z_{\nkkc},w(\supepoch))-\hat C(z_{\nkkc}, w^*)}^2\right] \\
	& \leq \E\left[L_2^2 \norm{w(\supepoch)-w^*}^2\right] \\
	&\leq L_2^2\left(c_1(1-\alpha c_0)^{\supepoch-t_\alpha'}+c_2\frac{\alpha t_{\alpha}'}{c_0}\right).
\end{align*}

\subsubsection{Proof of Lemma \ref{lemma:exp_decay_Q}}
\label{subsubsec:proof_exp_decay}
Denote by $\pi_{\kappa_r,t}$ the marginal distribution of $\sZ_{\nkkr}$ at time $t$ in the original Markov chain (state space $\sZ$). Let $\tilde \pi_{\kappa_r,t}$ denote the marginal state distribution at time $t$ in the sub-chain (state space $\sZ_{\nkkc}$.) 

Due to the local dependence of transition probability (cf. Eq. (\ref{eq:loc_z_dynamics})), $\pi_{\kappa_r,t}$ is only dependent on the initial states of agents in $N_{\agentk}^{\kappa_r+t}$, which is equal to $\tilde \pi_{t}^{\kappa_r}$ when $t\leq \kappa_c-\kappa_r$. Therefore we have for any $z\in \sZ$
\begin{align*}
	& \abs{\tilde C(z_\nkkc)-C(z)} \\
	\leq &\sum_{t=0}^{\infty}\abs{ \E\left[\gamma^t \supr(z_{\nkkr}(t))|z_{\nkkc}(0)
		= z_{\nkkc}\right]-\E\left[\gamma^t \supr( z_{\nkkr}(t))|z(0)
		=z\right]} \\
	=&\sum_{t=0}^{\infty}\abs{\gamma^t\mathop{\E}\limits_{z_{\nkkr}\sim \tilde \pi_{\kappa_r,t}} \supr(z_{\nkkr})
		-\gamma^t\mathop{\E}\limits_{z_{\nkkr}\sim \pi_{\kappa_r,t}} \supr(z_{\nkkr})} \\
	=&\sum_{t=\kappa_c-\kappa_r+1}^{\infty}\abs{\gamma^t\mathop{\E}\limits_{z_{\nkkr}\sim \tilde \pi_{\kappa_r,t}} \supr(z_{\nkkr})
		-\gamma^t\mathop{\E}\limits_{z_{\nkkr}\sim \pi_{\kappa_r,t}} \supr(z_{\nkkr})} \\
	\leq &\sum_{t=\kappa_c-\kappa_r+1}^{\infty}\gamma^t  \mathrm{TV}(\tilde \pi_{\kappa_r,t},\pi_{\kappa_r,t}) \\
	\leq & \frac{1}{1-\gamma}\gamma^{\kappa_c-\kappa_r+1}.
\end{align*}

Take $\sup_{z}$ on both sides above, we have
\begin{align*}
	\sup_z \abs{\tilde C(z_{\nkkc})-C(z)}\leq \frac{1}{1-\gamma}\gamma^{\kappa_c-\kappa_r+1}.
\end{align*}

\section{Proof of Theorem \ref{thm:critic_short}}\label{ap:critic}
In this section, we first derive the uniform properties of $\epsilon$-exploration policies, which are prerequisites for the critic error bound. Then we show that assumptions of Appendix \ref{sec:local_policy_eval} can be satisfied by Algorithm \ref{alg:LPES} under certain conditions. After that we restate Theorem \ref{thm:critic_short}. Finally, we give the proof of the theorem, which is based on the results of localized stochastic approximation (Appendix \ref{sec:local_policy_eval}). 

\subsection{Uniformity of $\epsilon$-Exploration Policies}\label{subsubsec:uni_eps}
In order to derive a uniform critic error bound for all policies, we need uniform properties, , such as convergence rate, exploration, for the critic sampling policies, i.e., $\epsilon$-exploration policy class $\Xi^{\epsilon}$. 
This can be done by applying  the results of \cite{zhang2022global} when Assumption \ref{assump:MC} holds. 

\begin{lemma}\label{lem:eps_ap_ir}
	For any policy $\hat \xi$ from $\Xi^{\epsilon}$ for $\epsilon>0$, the induced Markov chain $\{(s(t),a(t))\}$ is aperiodic and irreducible.
\end{lemma}
\begin{proof}[Proof of Lemma \ref{lem:eps_ap_ir}]
	The result is obvious by point 1. of \cite[Proposition 3]{zhang2022global} and the fact that $\hat \xi(a|s)\geq \prod_{i\in\sN} \frac{\epsilon}{\abs{\sA_i}}=\frac{\epsilon^n}{\abs{\sA}}>0$ for any $s,a$.
\end{proof}

For any policy $\hat \xi\in \Xi^{\epsilon}$, we use the notation $\pi^{\hat \xi}_{t}(s,a)$ to denote the probability of $(s,a)$ at time $t$ and use $\overline \pi^{\hat \xi}(s,a)$ for the probability of $(s,a)$ under stationary distribution. According to Lemma \ref{lem:eps_ap_ir}, for any $\hat \xi\in \Xi^{\epsilon}$, there exists $c_{\hat \xi}>1$ and $\rho_{\hat \xi}\in (0,1)$ such that 
\begin{align}\label{eq:mixing}
	\TV(\pi_t^{\hat \xi},\overline \pi^{\hat \xi})\leq c_{\hat \xi} \rho_{\hat \xi}^t. 
\end{align}
Let $\overline{c}:=\sup_{\hat \xi\in \Xi^{\epsilon}} c_{\hat \xi}$, $\overline{\rho}:=\sup_{\hat \xi\in \Xi^{\epsilon}} \rho_{\hat \xi}$. Besides, recall the definitions of $\pi_{\min}$ and $\underline{\lambda}$ in Subsection \ref{subsec:assumptions}. The uniformity of $\Xi^{\epsilon}$ can be shown by bounds of $\overline{\rho}$, $\pi_{\min}$ and $\underline{\lambda}$.

\begin{lemma}[Uniformity of $\epsilon$-exploration policy class]\label{lem:assump_MC_result}
	We have
	\begin{enumerate}
		\item $\overline{\rho}<1$. So $\pi_t^{\hat \xi}$ has a uniform convergence rate for all $\hat \xi\in \Xi^{\epsilon}$.
		\item $\pi_{\min}>0$. So each $(s_{\nikc},a_i)$ pair is visited with positive probability, which has a uniform positive lower bound for all $\epsilon$-exploration polices.
		\item $\underline{\lambda}>0$.
	\end{enumerate}
\end{lemma}

\begin{proof}[Proof of Lemma \ref{lem:assump_MC_result}]
	Fix any $\hat \xi\in \Xi^{\epsilon}$. For part 1, by \cite[Proposition 3]{zhang2022global}, for any $\hat \xi\in \Xi^{\epsilon}$, there exists $\hat c>0$ and $\rho_0<1$, such that $\sup_{\hat \xi\in \Xi^{\epsilon}}\max_{s\in\sS}\TV(\pi_{\sS,t}^{\hat \xi}(\cdot|s), \overline \pi_{\sS}^{\hat \xi})\leq \hat c\rho_0^t$ for any $t\geq 0$. Here $\pi_{\sS,t}^{\hat \xi}(s'|s)=\Pr_{\hat\xi}[s(t)=s'|s(0)=s]$ for any $s,s'\in \sS$, and $\overline\pi_{\sS}^{\hat \xi}\in \Delta(\sS)$ is the stationary distribution of $\sS$ under policy $\hat \xi$. Then
	\begin{align*}
		\TV(\pi_t^{\hat \xi},\overline \pi^{\hat \xi}) 
		=&\frac{1}{2}\sum_{s,a}\abs{\overline \pi_t^{\hat \xi}(s,a)-\overline \pi^{\hat \xi}(s,a)} \\
		=&\frac{1}{2}\sum_{s,a}\abs{\E_{s^0\sim\mu}\left[\pi_{\sS,t}^{\hat\xi}(s|s^0)\hat \xi(a|s)\right]-\overline\pi_{\sS}^{\hat\xi}(s)\hat\xi(a|s)} \\
		\leq &\frac{1}{2}\sum_{s,a}\E_{s^0\sim\mu}\left[\abs{\pi_{\sS,t}^{\hat\xi}(s|s^0)\hat \xi(a|s)-\overline\pi_{\sS}^{\hat\xi}(s)\hat\xi(a|s)}\right] \\
		=&\frac{1}{2}\E_{s^0\sim\mu}\left[\sum_{s,a}\abs{\pi_{\sS,t}^{\hat\xi}(s|s^0)-\overline\pi_{\sS}^{\hat\xi}(s)}\hat \xi(a|s)\right] \\
		=& \E_{s^0\sim\mu}\left[\TV(\pi_{\sS,t}^{\hat \xi}(\cdot|s), \overline \pi_{\sS}^{\hat \xi})\right] \\
		\leq & \hat c\rho_0^t.
	\end{align*}
	This shows that $\overline{\rho}\leq \rho_0<1$.
	
	As for part 2, by 3. of \cite[Proposition 3]{zhang2022global}, $\underline{\pi}:=\inf_{\xi\in \Xi^{\epsilon}}\min_{s\in\mathcal{S}}\overline \pi_{\sS}^{\hat \xi}(s)>0$. Thus for any $s_{\nikc}\in \sS_{\nikc},a_i\in \sA_i$ and any $\hat \xi\in \Xi^{\epsilon}$,
	\begin{align*}
		&\overline \pi^{\hat \xi} (s_{\nikc},a_i) \\
		= &\sum_{s_{-\nikc},a_{-i}}\overline \pi^{\hat \xi}(s_{\nikc},s_{-\nikc},a_i,a_{-i}) \\
		\geq &\underline{\pi}
	\end{align*}
	Thus $\pi_{\min}\geq \underline{\pi}>0$.
	
	Finally we show part 3.     For any square matrix $X$, denote by $\lambda_{\min}(X)$ the minimum eigenvalue of $X$. 
	By definition of $\pi_{\min}$ (see Subsection \ref{subsec:assumptions}), for any $\hat\xi\in \Xi^{\epsilon}$, $\lambda_{\min}(D^{\hat\xi})\geq \pi_{\min}$.
	
	Since $\Omega_i$ is full column rank, $\Omega_i^\top \Omega$ is positive definite, so $\lambda_{\min}(\Omega_i^\top \Omega_i)>0$ for any $i\in \sN$. 
	
	For any eigenvalue $\lambda$ of $\Omega_i^\top D^{\hat \xi}\Omega_i$, let $x\in \R^{d_i}$ be the corresponding eigenvector. Then
	\begin{align*}
		\lambda \norm{x}^2 &=\lambda x^\top x \\
		&=x^\top\Omega_i^\top D^{\hat \xi}\Omega_i x \\
		&\geq \pi_{\min} \norm{\Omega_i x}^2 \\
		&= \pi_{\min} x^T \Omega_i^\top \Omega_i x \\
		&= \pi_{\min} \lambda_{\min}(\Omega_i^\top \Omega_i) \norm{x}^2.
	\end{align*}
	Therefore $\lambda\geq \pi_{\min} \lambda_{\min}(\Omega_i^\top \Omega_i)$, and thus $\underline \lambda_i^{\hat \xi}\geq \pi_{\min} \lambda_{\min}(\Omega_i^\top \Omega_i)$. As a result, $\underline{\lambda}\geq \pi_{\min} \lambda_{\min}(\Omega_i^\top \Omega_i)>0$.
\end{proof}

Besides the uniform convergence rate of the distribution $(s(t),a(t))$, we also want to establish the uniform convergence of the marginal distribution $(s(t),a_i(t))$ for all $i\in \sN$. 

For any agent $i\in \sN$ and policy $\hat \xi\in \Xi^{\epsilon}$, we use the notation $\pi_{i,t}^{\hat \xi}(s,a_i)$ to denote the marginal probability of $(s,a_i)$ at time $t$ and use $\overline \pi_i^{\hat \xi}(s,a_i)$ for the marginal probability of $(s,a_i)$ under stationary distribution. The following lemma shows that the distribution $(s(t),a_i(t))$ converges as fast as the $(s,a)$, with a uniform convergence rate for all $\epsilon$-exploration policies. 

\begin{lemma}\label{lem:mixing_ak}
	The distribution of $(s(t),a_i(t))$ has uniform convergence rate
	\begin{align*}
		\TV(\pi^{\hat \xi}_{i,t},\overline{\pi}^{\hat \xi}_i)\leq \overline{c}\cdot\overline{\rho}^t.
	\end{align*}
	Here $\overline c$ and $\overline \rho$ are the uniform convergence rate of the distribution of $(s(t),a(t))$ (see Appendix \ref{subsubsec:uni_eps}).
\end{lemma}

\begin{proof}[Proof of Lemma \ref{lem:mixing_ak}]
	We have
	\begin{align*}
		& \TV(\pi^{\hat \xi}_{i,t},\overline{\pi}^{\hat \xi}_i) \\
		=&\frac{1}{2}\sum_{s,a_i}\abs{\pi^{\hat \xi}_{i,t}(s,a_i)-\overline{\pi}^{\hat \xi}_i} \\
		=&\frac{1}{2}\sum_{s,a_i}\abs{\sum_{a_{-i}}\pi^{\hat \xi}_{t}(s,a_i,a_{-i})-\overline{\pi}^{\hat \xi}(s,a_i,a_{-i})} \\
		\leq & \frac{1}{2}\sum_{s,a_i}\sum_{a_{-i}}\abs{\pi^{\hat \xi}_{t}(s,a_i,a_{-i})-\overline{\pi}^{\hat \xi}(s,a_i,a_{-i})} \\
		\leq &\overline{c}\cdot\overline{\rho}^t.
	\end{align*}
	Here the last inequality is due to Eq. (\ref{eq:mixing}), and $c_{\hat \xi}\leq \overline{c}$, $\rho_{\hat \xi}\leq \overline{\rho}$.
\end{proof}

\subsection{Verifying assumptions of Appendix \ref{sec:local_policy_eval}}
\label{subsec:verify_assumps_LPE}
In order to apply the results of localized stochastic approximation (Appendix \ref{sec:local_policy_eval}) , we verify the assumptions needed for Theorem \ref{thm:chen_main_var}.
First, we correspond the problem setting of NMPG to the Markov chain setting of localized stochastic approximation. 

We fix any policy $\hat \xi$,  and any agent $i\in \sN$ for the rest of this section. Then $\{(s(t),a_i(t)\}$ forms a Markov chain $\sM^{i,\hat \xi}$. 
Recall that $U_j^{\kappa}=N_j^{\kappa}/\{j\}$, which is the set of agents within $j$'s $\kappa$-hop neighborhood excluding $j$ it self. Construct Markov chain $\sM^{i,\hat \xi}=(\sN, \mathcal E, \sZ, \PR, \supr, \gamma, \mu')$, and choose $\agentk=i$. Here $\sN$, $\mathcal E$ and $\gamma$ have the same meaning as in the NMPG setting (Section \ref{sec:setting}), which are the set of agents, the edges of the graph and the discount factor, respectively. For the other elements, they are defined as
\begin{align*}
	&\sZ_j=\begin{cases}
		\sS_j & j\neq i \\
		\sS_j\times \sA_j &j=i
	\end{cases} \\
	&z_j=\begin{cases}
		s_j &j\neq i \\
		(s_j,a_j) &j=i
	\end{cases}\\
	&\PR_j(z_j'|z_{N_j})=\hat\xi_j(a_j'|s_j')\sP_j(s_j'|s_{N_j},a_j), \forall j\in \sN \\
	&\mu'(z)=\hat\xi(a|s)\mu(s) \\
	&\supr(z_{\nikr})
	=\overline{r}_i^{\hat \xi}(s_{\nikr},a_i)
	:=\sum_{a_{\uikr}}\hat\xi_{\uikr}(a_{\uikr}|s_{\uikr}) r_i(s_{\nikr},a_i,a_{\uikr}).
\end{align*}
Here $\overline{r}_i^{\hat \xi}(s_{\nikr},a_i)$ is the expected reward function with respect to $a_{\uikr}$. Furthermore, cost function $C(z)$ of $\sM^{i,\hat \xi}$ corresponds to 
\begin{align*}
	C(s,a_i)&=\sum_{t=0}^\infty \gamma^t {\E}_{\hat \xi}[\overline{r}_i^{\hat \xi}(s_{\nikr}(t),a_i(t)) | s(0)=s,a_i(0)=a_i] \\
	&=\sum_{t=0}^\infty \gamma^t {\E}_{\hat \xi}[r_i(s_{\nikr}(t),a_{\nikr}(t)) | s(0)=s,a_i(0)=a_i] \\
	&=\overline Q_i^{\hat \xi}(s,a_i),
\end{align*}
which is the averaged $Q$-function.

Now we represent localized TD($\lambda$) with linear function approximation in the form of generalized TD($\lambda$) (Algorithm \ref{alg:recursive_estimator}). 
Let 
\begin{align}
	&d=d_i,\ t_0=0 \nonumber \\
	&X(t)=(z_{\nikc}(t),z_{\nikc}(t+1)) \nonumber \\
	&F(z_{\nikc}(t),z_{\nikc}(t+1),w)=-\delta_i(t). \label{eq:F_expr}
\end{align}
(Recall the definition of $\delta_i(t)$ in Algorithm \ref{alg:LPES}.) Then Algorithm \ref{alg:recursive_estimator} analyzes the averaged $Q$-function of agent $i$.

Next we show that Assumptions \ref{assump:chain_irreducible}, \ref{assump:chen_assump}, \ref{assump:hatC_lipschitz} are satisfied by localized TD($\lambda$) under the conditions of Theorem \ref{thm:critic_short}.
\paragraph{Verify Assumption \ref{assump:chain_irreducible}.} This can be shown by the following lemma.
\begin{lemma}\label{lem:chain_irr_ak}
	For any policy $\hat\xi\in \Xi^{\epsilon}$, $\epsilon>0$, $i\in \sN$, the induced Markov chain $\{(s(t),a_i(t))\}$ is aperiodic and irreducible. 
\end{lemma}
\begin{proof}[Proof of Lemma \ref{lem:chain_irr_ak}]
	For any Markov chain with transition probability $\Gamma$ on some state space $\mathcal{X}$, we write $\Gamma^l(x'|x)=\Pr[x(t+l)=x'|x(t)=x]$, for any $t\in \N$, $x,x'\in \mathcal{X}$. 
	
	We fix $i\in \sN$. By Lemma $\ref{lem:assump_MC_result}$, the induced Markov chain $\{(s(t),a(t))\}$ is aperiodic and irreducible. We abuse the notation in this proof to use $\PR$ for the transition probability of $(s,a_i)$
	and use $\tilde \PR$ for that of $(s,a)$.
	
	For any 
	$s,s'\in \sS, a_i,a_i'\in \sA_i$, we randomly pick $a_{-i},a_{-i}'\in \sA_{-i}$.
	Since $\tilde \PR$ is irreducible, there exists $t>0$, such that 
	$\tilde \PR^t(s',a_i',a_{-i}'|s,a_i,a_{-i})>0$. We choose the smallest $t_0>0$ such that $\tilde \PR^{t_0}(s',a_i',a_{-i}'|s,a_i,a_{-i})>0$ for any $a_{-i},a_{-i}'\in \sA_{-i}$,
	so $\PR^t_{i}(s',a_i'|s,a_i,a_{-i})>0$. $\PR^t_{i}(s',a_i'|s,a_i,a_{-i})$ represents the marginal probability of $s',a_i'$ given previous state-action pair $(s,a_i,a_{-i})$. Then we have
	\begin{align*}
		{\PR}^{t_0}(s',a_i'|s,a_i)
		&=\sum_{\hat a_{-i}\in \sA_{-i}} 
		\hat\xi_{-i}(\hat a_{-i}|s_{-i}) \PR^{t_0}_{i}(s',a_i'|s,a_i,\hat a_{-i}) \\
		&\geq \min_{\hat a_{-i}\in \sA_{-i}} \PR^{t_0}_{i}(s',a_i'|s,a_i,\hat a_{-i}) \\
		&>0.
	\end{align*}
	Therefore, $\PR$ is irreducible.
	
	To show that $\PR$ is aperiodic, we assume that $\PR$ has period $T\geq 2$. Then for any 
	$t$ not divisible by $T$ and any $s,a_i$, $\PR^t(s,a_i|s,a_i)=0$. For any 
	$s,a$, since $t$ not divisible by $T$, we have
	\begin{align*}
		0=\PR^{t}(s,a_i|s,a_i)=\sum_{\hat a_{-i}\in \sA_{-i}} 
		\hat\xi_{-i}(\hat a_{-i}|s_{-i}) \PR^{t}_{i}(s,a_i|s,a_i,\hat a_{-i}).
	\end{align*}
	Thus we have $\PR^{t}_{i}(s,a_i|s,a_i,\hat a_{-i})=0$ for some 
	$\hat a_{-i}$ and thus $\PR^{t}_{i}(s,a_i,\hat a_{-i}|s,a_i,\hat a_{-i})=0$. This implies that the 
	period of $(s,a_i,\hat a_{-i})$ is at least $T\geq 2$, which contradicts the assumption that $\PR$ is aperiodic. (Notice that the period of any state is the same for an irreducible Markov chain.) Hence 
	${\PR}$ is aperiodic.
	
	In conclusion, ${\PR}$ is irreducible and aperiodic.
\end{proof}

\paragraph{Verify Assumption \ref{assump:chen_assump}.} We show that for any agent $i\in \sN$ and any policy $\hat \xi\in \Xi^{\epsilon}$, the assumption can be satisfied for $L_1=1+\gamma$ and $c_0=(1-\gamma)\underline \lambda$.

We first verify 1. of Assumption \ref{assump:chen_assump}. For any $w_1,w_2,x$, we have
\begin{align*}
	&\norm{F(x,w_1)-F(x,w_2)} \\
	= & \norm{\langle\phi_i(s_{\nikc}(t),a_i(t)), w_1-w_2\rangle-\gamma
		\langle\phi_i(s_{\nikc}(t+1),a_i(t+1)), w_1-w_2\rangle} \\
	\leq & \norm{\langle\phi_i(s_{\nikc}(t),a_i(t)), w_1-w_2\rangle}+\gamma
	\norm{\langle\phi_i(s_{\nikc}(t+1),a_i(t+1)), w_1-w_2\rangle}\\
	\leq &(1+\gamma)\|w_1-w_2\|.  
\end{align*}

For any $x$, notice that
\begin{align*}
	&\|F(x,0)\| \\
	=& \norm{\overline r_i^{\hat \xi}(s_{\nikr}(t),a_i(t))} \\
	\leq & 1 \\
	\leq & 1+\gamma.
\end{align*}

Now we prove 2. in Assumption \ref{assump:chen_assump}. 
We point out that our localized TD($\lambda$) can be easily reduced to the classical single-agent TD($\lambda$) algorithm, with $\sZ$ being the single-agent state space.
Then by \cite{tsitsiklis1997analysis}, $\overline G(w)$ has a unique zero $w^*$. Furthermore, then according to the proof of Lemma 9 in \cite{tsitsiklis1997analysis}, we have
\begin{align*}
	\langle w-w^*, \overline G(w)\rangle &\leq -(1-\gamma)\norm{\Omega_i w-\Omega_i w^*}_{D^{\hat \xi}}  \\
	&=-(1-\gamma) (w-w^*)^\top \Omega_i^\top D^{\hat \xi} \Omega_i (w-w^*) \\
	&\leq -(1-\gamma)\underline{\lambda}_{k}^{\hat \xi} \norm{w-w^*}^2 \\
	&\leq -(1-\gamma)\underline{\lambda} \norm{w-w^*}^2.
\end{align*}
We will sometimes add a subscript $i$ for $w^*$ and write $w_i^*$ to distinguish different agents.

\paragraph{Verify Assumption \ref{assump:hatC_lipschitz}.}

Approximate cost function $\hat C$ in the localized stochastic approximation problem corresponds to the approximate averaged $Q$-function of agent $i$, i.e., $\hat Q_i(s_{\nikc},a_i,w_i)=\langle \phi_i(s_{\nikc},a_i),w_i\rangle$. Since $\norm{\phi_i(s_{\nikc},a_i)}\leq 1$,
we have for any $w_1,w_2$
\begin{align*}
	&\abs{\hat C(z_{\nikc},w_1)-\hat C(z_{\nikc},w_2)} \\
	=& \abs{\langle \phi_i(s_{\nikc},a_i),w_1-w_2\rangle} \\
	\leq & \norm{\phi_i(s_{\nikc},a_i)} \norm{w_1-w_2} \\
	\leq &\norm{w_1-w_2}.
\end{align*}
So Assumption \ref{assump:hatC_lipschitz} can be satisfied with $L_2=1$.

\subsection{Restatement of Theorem \ref{thm:critic_short} }
Before giving the precise form of Theorem \ref{thm:critic_short}, 
We define two quantities needed for the theorem statement. The first quantity is the function approximation error.
Denote by $\sM_{\nikc}^{\hat \xi}$ the sub-chain of $\sM^{i,\hat \xi}$ (defined in Appendix \ref{subsec:verify_assumps_LPE}) with respect to agents in $\nikc$, and let $\tilde Q_i^{\hat \xi}(s_{\nikc},a_i)$ be the cost function of $\sM^{i,\hat \xi}$
at state $(s_{\nikc},a_i)$.
\begin{definition}
	For any policy $\hat \xi \in \Xi^{\epsilon}$,  and any agent $i\in \sN$, define the function approximation error as
	\begin{align*}
		\epsilon_{app}:=\sup_{\hat \xi\in \Xi^{\epsilon}}\sup_{i\in \sN} \inf_{w} \sup_{s,a_i} \abs{\hat Q_i(s_{\nikc},a_i,w)-\tilde Q_i^{\hat \xi}(s_{\nikc},a_i)}.
	\end{align*}
\end{definition}

Notice that we measure the representation power of the function approximation with respect to the sub-chain instead of the original MDP.

The second quantity we need is the uniform mixing time of the update term $-\delta_i(t)\zeta_i^{\kappa_c}(t)$ in Algorithm \ref{alg:LPES}, corresponding to function $G$ in Appendix \ref{sec:local_policy_eval}. 
Notice update term $-\delta_i(t)\zeta_i^{\kappa_c}(t)$ does not depend on policy $\hat \xi$, and there is a uniform decay rate of $(s(t),a_i(t)$ for any $i$ and any policy $\hat \xi\in \Xi^{\epsilon}$ (Lemma \ref{lem:mixing_ak}). Then we can apply Lemma \ref{lemma:chen_lm_2_2} for all policy $\hat \xi\in \Xi^{\epsilon}$ and get a uniform decay rate $c_{\text{TD}}:=c_g(\overline c, \overline \rho, \lambda, 0)$, $\rho_{\text{TD}}:=\rho_g(\overline \rho, \lambda)$. This indicates that there is a uniform ``mixing time of function $G$ with precision $\delta$'' (Definition \ref{def:G_mix}) for all policies $\hat \xi\in \Xi^{\epsilon}$ and all $\delta>0$ , which we denote by $t_{\delta}$. In addition, $t_{\delta}=O(\log\frac{1}{\delta})$.

We now give a stronger version of Theorem \ref{thm:critic_short}. We can derive Theorem \ref{thm:critic_short} by taking square root on both sides and apply Jenson's inequality on the right hand side. 
\begin{theorem}[Restatement of Theorem \ref{thm:critic_short}]\label{thm:critic_comb_restate}
	Recall the definition of uniform mixing time $t_\alpha$ above. 
	Choose step size $\alpha$ such that $\alpha t_{\alpha}\leq \min\Bp{\frac{1-\lambda}{4(1+\gamma)}, \frac{(1-\gamma)(1-\lambda)^2\underline{\lambda}}{114(1+\gamma)^2}}$.
	Then we have for any $K \geq t_\alpha$ and any $i\in \sN$
	\begin{align*}
		\epsilon_{critic}^2=&\max_{i\in\mathcal{N}}\sup_{\theta}\E\left[\max_{s,a_i} \abs{\hat Q_{i}(s_{\nikc},a_i,w_i(K))-\overline Q_i^{\theta}(s,a_i)}^2\right] \\
		\leq & 4\left[c_1^*(1-(1-\gamma)\underline{\lambda}\alpha)^{K-t_\alpha}+ c_2^*\frac{\alpha t_{\alpha}}{(1-\gamma)\underline{\lambda}}
		+\Sp{\frac{(1-\lambda\gamma)\epsilon_{app}}{\pi_{\min}(1-\gamma)}}^2\right. \\
		&+ \left.\Sp{\frac{\gamma^{\kappa_c-\kappa_r+1}}{1-\gamma}}^2 
		+\Sp{\frac{6n\epsilon}{(1-\gamma)^2}}^2\right].
	\end{align*}
	Here $c_1^*=(\max_{i\in \sN}\|w_i^*\|+1)^2$, $c_2^*=114\Sp{\frac{1+\gamma}{1-\lambda}}^2(\max_{i\in \sN}\norm{w_i^*}+1)^2$.
\end{theorem}

\subsection{Proof of Theorem \ref{thm:critic_comb_restate}}
Throughout the proof, we will fix agent $i\in \sN$ and policy $\xi^{\theta}$. 
For any policy parameter $\theta$, let $\hat \xi$ be the $\epsilon$-exploration policy of $\xi^{\theta}$. We can decompose the critic error as
\begin{align}
	&\E\left[\max_{s,a_i} \abs{\hat Q_i(s_{\nikc},a_i,w_i(K))-\overline Q_i^{\theta}(s,a_i)}^2\right] \nonumber \\
	\osi{\leq} &\E\left[\max_{s,a_i} (3+1)\Sp{\frac{1}{3}\abs{\hat Q_i(s_{\nikc},a_i,w_i(K))-\overline Q_i^{\hat \xi}(s,a_i)}^2+\abs{\overline Q_i^{\hat \xi}(s,a_i)-\overline Q_i^{\theta}(s,a_i)}^2}\right] \nonumber \\
	\osii{\leq} &(3+1)\left\{\E\left[\frac{1}{3}\max_{s,a_i} \abs{\hat Q_i(s_{\nikc},a_i,w_i(K))-\overline Q_i^{\hat \xi}(s,a_i)}^2\right] + \E\left[\max_{s,a_i} \abs{\overline Q_i^{\hat \xi}(s,a_i)-\overline Q_i^{\theta}(s,a_i)}^2\right]\right\}\nonumber \\
	\osiii{=}& 4\Bp{\frac{1}{3}\underbrace{\E\left[\max_{s,a_i} \abs{\hat Q_i(s_{\nikc},a_i,w_i(K))-\overline Q_i^{\hat \xi}(s,a_i)}^2\right]}_{(a)}+
		\underbrace{\max_{s,a_i} \abs{\overline Q_i^{\hat \xi}(s,a_i)-\overline Q_i^{\xi^{\theta}}(s,a_i)}^2}_{(b)}}. \label{eq:comb_critic_decompose}
\end{align}
Here $\ri$ uses Cauchy-Schwarz inequality, $\rii$ uses the fact that $\max_x\Bp{f(x)+g(x)}\leq \max_x f(x)+\max_x g(x)$ for any two functions $f(x),g(x)$.
For $\riii$, note that $Q_i^{\theta}(s,a_i)$ is the shorthand notation of $Q_i^{\xi^{\theta}}(s,a_i)$, and the term $(b)$ is irrelevant to the trajectory sampled. 

$(a)$ is the policy evaluation error in estimating the value function of policy $\hat \xi$, which can be derived using the result in Appendix \ref{sec:local_policy_eval}. $(b)$ is the difference in the value function between policy $\hat\xi$ and $\xi^{\theta}$, which can be bounded by a function of $\epsilon$.

\paragraph{Bound $(a)$.} 
With all assumptions in Appendix \ref{sec:local_policy_eval} satisfied (Appendix \ref{subsec:verify_assumps_LPE}), the following result is a direct application of Theorem \ref{thm:chen_main_var} in localized TD($\lambda$) with  linear function approximation:
\begin{corollary}\label{cor:td_raw}
	Choose step size $\alpha$ such that $\alpha t_{\alpha}\leq \min\Bp{\frac{1-\lambda}{4(1+\gamma)}, \frac{(1-\gamma)(1-\lambda)^2\underline{\lambda}}{114(1+\gamma)^2}}$.
	Then for $\agentk\geq t_\alpha$, we have
	\begin{align*}
		&\E\left[\max_{s,a_i} \abs{\hat Q_i(s_{\nikc},a_i,w_i(K))-\overline Q_i^{\hat \xi}(s,a_i)}^2\right] \\
		\leq & 3\Mp{c_1^*(1-(1-\gamma)\underline{\lambda}\alpha)^{K-t_\alpha}+c_2^*\frac{\alpha t_{\alpha}}{(1-\gamma)\underline{\lambda}}
			+ \epsilon_{i,red}^2 
			+ \Sp{\frac{\gamma^{\kappa_c-\kappa_r+1}}{1-\gamma}}^2}.
	\end{align*}
	Here 
	\begin{align*}
		c_1^*=\;&(\max_{i\in \sN}\|w_i^*\|+1)^2,\;c_2^*=114\Sp{\frac{1+\gamma}{1-\lambda}}^2(\max_{i\in \sN}\norm{w_i^*}+1)^2,\\
		\epsilon_{i,red}=\;&\sup_{s_{\nikc},a_i}\abs{\hat Q_i(s_{\nikc},a_i,w_i^*)-\tilde Q_i^{\hat \xi}(s_{\nikc},a_i)}.
	\end{align*}
\end{corollary}

For localized TD($\lambda$) with linear function approximation, we can associate $\epsilon_{i,red}$ with function approximation error $\epsilon_{app}$ in a similar way as \cite{tsitsiklis1997analysis}.

\begin{lemma}\label{lem:red_app}
	The reduction error of agent $i$' can be bounded by the function approximation error
	\begin{align*}
		\epsilon_{i,red}\leq \frac{1}{\pi_{\min}}\cdot\frac{1-\lambda\gamma}{1-\gamma}\epsilon_{app}.
	\end{align*}
\end{lemma}

\begin{proof}[Proof of Lemma \ref{lem:red_app}]
	Denote by $\overline\pi^{\hat \xi}_{\kappa_c}\in \Delta(\sS_{\nikc}\times\sA_i)$ the marginal stationary distribution of $(s_{\nikc},a_i)$. 
	Define
	\begin{align*}
		\tilde D=\mathop{\diag}\limits_{(s_{\nikc},a_i)\in \sS_{\nikc}\times\sA_i}\{\overline \pi^{\hat \xi}_{\kappa_c}(s_{\nikc},a_i)\}.
	\end{align*} 
	Then
	\begin{align*}
		\epsilon_{i,red} 
		&\leq \frac{1}{\min_{s_{\nikc},a_i} \{\overline \pi^{\hat \xi}_{\kappa_c}(s_{\nikc},a_i)\}}\norm{\hat Q_i(\cdot,\cdot,w^*)-\tilde Q_i^{\hat \xi}(\cdot,\cdot)}_{\tilde D} \\
		&\overset{(\romannumeral1)}{\leq} \frac{1}{\min_{s_{\nikc},a_i} \{\overline \pi^{\hat \xi}_{\kappa_c}(s_{\nikc},a_i)\}}\cdot\frac{1-\lambda\gamma}{1-\gamma}\inf_w \norm{\hat Q_i(\cdot,\cdot,w)-\tilde Q_i^{\hat \xi}(\cdot,\cdot)}_{\tilde D} \\
		&\leq \frac{1}{\min_{s_{\nikc},a_i} \{\overline \pi^{\hat \xi}_{\kappa_c}(s_{\nikc},a_i)\}}\cdot\frac{1-\lambda\gamma}{1-\gamma}\inf_w \sup_{s,a_i}\abs{\hat Q_i(s_{\nikc},a_i,w)-\tilde Q_i^{\hat \xi}(s,a_i)} \\
		&\leq \frac{1}{\min_{s_{\nikc},a_i} \{\overline \pi^{\hat \xi}_{\kappa_c}(s_{\nikc},a_i)\}}\cdot\frac{1-\lambda\gamma}{1-\gamma} \epsilon_{app} \\
		&\osii{\leq} \frac{1}{\pi_{\min}}\cdot\frac{1-\lambda\gamma}{1-\gamma} \epsilon_{app}. \\
	\end{align*}
	Here in $(\romannumeral1)$, we applied \cite[Lemma 6]{tsitsiklis1997analysis} to the sub-chain $\sM_{\nikc}^{\hat \xi}$, and in $\rii$, we used the fact that $\min_{s_{\nikc},a_i} \{\overline \pi^{\hat \xi}_{\kappa_c}(s_{\nikc},a_i)\geq \pi_{\min}$.
\end{proof}

Combining Lemma \ref{lem:red_app} and Corollary \ref{cor:td_raw}, we immediately have the bound of $(a)$.
\begin{corollary}\label{cor:critic_restate}
	Choose step size $\alpha$ such that $\alpha t_{\alpha}\leq \min\Bp{\frac{1-\lambda}{4(1+\gamma)}, \frac{(1-\gamma)(1-\lambda)^2\underline{\lambda}}{114(1+\gamma)^2}}$.
	Then for $K\geq t_\alpha$ and any $i\in \sN$, we have
	\begin{align*}
		(a)=&\E\left[\max_{s,a_i} \abs{\hat Q_i(s_{\nikc},a_i,w(K))-\overline Q_i^{\hat \xi}(s,a_i)}^2\right] \\
		\leq & 3\Mp{c_1^*(1-(1-\gamma)\underline{\lambda}\alpha)^{K-t_\alpha}+ c_2^*\frac{\alpha t_{\alpha}}{(1-\gamma)\underline{\lambda}} +\Sp{\frac{(1-\lambda\gamma)\epsilon_{app}}{\pi_{\min}(1-\gamma)}}^2
			+ \Sp{\frac{\gamma^{\kappa_c-\kappa_r+1}}{1-\gamma}}}^2.
	\end{align*}
	Here $c_1^*=(\max_{i\in \sN}\|w_i^*\|+1)^2$, $c_2^*=114\Sp{\frac{1+\gamma}{1-\lambda}}^2(\max_{i\in \sN}\norm{w_i^*}+1)^2$.
\end{corollary}

\paragraph{Bound $(b)$.} We first discuss the $l_1$-distance of policy $\xi^{\theta}$ and $\hat\xi$ with the two lemmas belows.

\begin{lemma}\label{lem:pol_1_norm_diff}
	For any agent $i$ and policy $\xi_i$, let $\hat \xi_i(a_i,s_i)=(1-\epsilon)\xi_i(a_i|s_i)+\epsilon \frac{1}{\abs{\sA_i}}$ for all $s_i,a_i$. Then 
	\begin{align*}
		\norm{\xi_i(\cdot|s_i)-\hat\xi_i(\cdot|s_i)}_1\leq 2\epsilon, \; \forall s_i.
	\end{align*}
\end{lemma}

\begin{proof}[Proof of Lemma \ref{lem:pol_1_norm_diff}]
	We have by definition of the $l_1$-norm that
	\begin{align*}
		\norm{\xi_i(\cdot|s_i)-\hat \xi_i(\cdot|s_i)}_1 
		=& \sum_{a_i}\abs{\xi_i(a_i|s_i)-\hat\xi_i(a_i|s_i)} \\
		=&\epsilon\sum_{a_i}\abs{\frac{1}{\abs{\sA_i}}-\xi_i(a_i|s_i)} \\
		\leq &\epsilon \sum_{a_i}\left(\frac{1}{\abs{\sA_i}}+\xi_i(a_i|s_i)\right) \\
		= &2\epsilon.
	\end{align*}
\end{proof}

Consider the difference of policy of a set of agents, we have the following result.
\begin{lemma}\label{lem:prod_pol_diff}
	For any set of agents $\sI\subseteq \sN$ and any policy $\xi_{\sI}$. Define $\hat\xi_i$ in the same way as in Lemma \ref{lem:pol_1_norm_diff} for any $i\in \sI$ and let $\hat \xi_{\sI}$ be the product policy. Then 
	\begin{align*}
		\norm{\xi_{\sI}(\cdot|s_{\sI})-\hat\xi_{\sI}(\cdot|s_{\sI})}_1\leq 2\abs{\sI}\epsilon, \; \forall s_{\sI}.
	\end{align*}
\end{lemma}

\begin{proof}[Proof of Lemma \ref{lem:prod_pol_diff}]
	By Lemma \ref{lem:prod_pol_diff_general} and Lemma \ref{lem:pol_1_norm_diff}, we immediately have
	\begin{align*}
		\norm{\xi_{\sI}(\cdot|s_{\sI})-\hat \xi_{\sI}(\cdot|s_{\sI})}_1 
		\leq &\sum_{i\in \sI} \norm{\xi_{i}(\cdot|s_{i})-\hat \xi_{i}(\cdot|s_{i})}_1 \\
		\leq &2\abs{\sI} \epsilon.
	\end{align*}
\end{proof}

Now we are ready to bound $(b)$.
\begin{lemma}\label{lem:Q_perf_diff}
	For any policy $\xi^{\theta}$, we have
	\begin{align*}
		(b)=\max_{s,a_i}\abs{\overline Q_i^{\hat\xi}(s,a_i)-\overline Q_i^{\xi^{\theta}}(s,a_i)}^2\leq \Sp{\frac{6n\epsilon}{(1-\gamma)^2}}^2.
	\end{align*}
	Here $\hat\xi$ is the $\epsilon$-exploration policy of policy $\xi^{\theta}$. Please refer to Line \ref{line:critic_sample_beg} of Algorithm \ref{alg:LPES} for the explicit definition of $\epsilon$-exploration policy.
\end{lemma}

\begin{proof}[Proof of Lemma \ref{lem:Q_perf_diff}]    
	Notice that for any policy $\xi$, we have
	\begin{align*}
		&Q_i^{\xi}(s,a_i) \\
		=&\sum_{a_{-k}}\xi_{-k}(a_{-k}|s_{-k})Q_i^{\xi}(s,a_i,a_{-k}) \\
		=& \sum_{a_{-k}}\xi_{-k}(a_{-k}|s_{-k}) \left(r_i(s_{\nikr},a_i,a_{\ukkr})+\sum_{s'}\sP(s'|s,a_i,a_{-k})V_i^{\xi}(s')\right).
	\end{align*}
	Then we have for any two policies $\xi$ and $\xi'$ that
	\begin{align*}
		&  \abs{\overline Q_i^{\xi}(s,a_i)-\overline Q_i^{\xi'}(s,a_i)} \\
		= & \left|\sum_{a_{-k}}\xi_{-k}(a_{-k}|s_{-k}) \left(r_i(s_{\nikr},a_i,a_{\ukkr})+\sum_{s'}\sP(s'|s,a_i,a_{-k})V_i^{\xi}(s')\right)\right. \\
		&-\left.\sum_{a_{-k}}\xi_{-k}'(a_{-k}|s_{-k}) \left(r_i(s_{\nikr},a_i,a_{\ukkr})+\sum_{s'}\sP(s'|s,a_i,a_{-k})V_i^{\xi'}(s')\right)\right| \\
		\leq & \abs{\sum_{a_{-k}} (\xi_{-k}(a_{-k}|s_{-k})-\xi'_{-k}(a_{-k}|s_{-k}))r_i(s_{\nikr},a_i,a_{\ukkr})} \\
		&+\abs{\sum_{a_{-k},s'}\sP(s'|s,a_i,a_{-k})\left[\xi_{-k}(a_{-k}|s_{-k})V_i^{\xi}(s')-\xi'_{-k}(a_{-k}|s_{-k})V_i^{\xi'}(s')\right]} \\
		\leq & \abs{\sum_{a_{-k}} (\xi_{-k}(a_{-k}|s_{-k})-\xi'_{-k}(a_{-k}|s_{-k}))r_i(s_{\nikr},a_i,a_{\ukkr})} \\
		&+\abs{\sum_{a_{-k},s'}\sP(s'|s,a_i,a_{-k})\left[\xi_{-k}(a_{-k}|s_{-k})-\xi'_{-k}(a_{-k}|s_{-k})\right]V_i^{\xi}(s')} \\
		&+\abs{\sum_{a_{-k},s'}\sP(s'|s,a_i,a_{-k})\xi_{-k}'(a_{-k}|s_{-k})\left[V_i^{\xi}(s')-V_i^{\xi'}(s')\right]} \\
		\leq & \sum_{a_{-k}} \abs{\xi_{-k}(a_{-k}|s_{-k})-\xi'_{-k}(a_{-k}|s_{-k})} \\
		&+\frac{1}{1-\gamma}\sum_{a_{-k}}\abs{\xi_{-k}(a_{-k}|s_{-k})-\xi'_{-k}(a_{-k}|s_{-k})}\sum_{s'}\sP(s'|s,a_i,a_{-k}) \\
		&+ \sup_{s'}\left[V_i^{\xi}(s')-V_i^{\xi'}(s')\right].
	\end{align*} 
	To bound the last term, we apply a variant of performance difference lemma (Lemma \ref{lem:V_PDT}), and we have for any state $s$ that
	\begin{align}
		\abs{V_i^{\xi}(s)-V_i^{\xi'}(s)} 
		=&\frac{1}{1-\gamma}\abs{\sum_{s',a}d_s^{\xi'}(s')(\xi'(a|s')-\xi(a|s')) Q_i^{\xi}(s',a)} \nonumber\\
		\leq &\frac{1}{(1-\gamma)^2}\abs{\sum_{s',a}d_s^{\xi'}(s')(\xi'(a|s')-\xi(a|s'))} \nonumber\\
		\leq &\frac{1}{(1-\gamma)^2}\sum_{s',a}d_s^{\xi'}(s')\abs{\xi'(a|s')-\xi(a|s')}\nonumber\\
		\leq &\frac{1}{(1-\gamma)^2}\sup_{s'}\norm{\xi'(\cdot|s')-\xi(\cdot|s')}_1. 
		\label{eq:V_PDT}
	\end{align}
	By Eq. (\ref{eq:V_PDT}), we can further bound $\abs{\overline Q_i^{\xi}(s,a_i)-\overline Q_i^{\xi'}(s,a_i)}$ as
	\begin{align*}
		\abs{\overline Q_i^{\xi}(s,a_i)-\overline Q_i^{\xi'}(s,a_i)}     
		\leq \;& \norm{\xi_{-k}(\cdot|s_{-k})-\xi'_{-k}(\cdot|s_{-k})}_1+\frac{1}{1-\gamma}\norm{\xi_{-k}(\cdot|s_{-k})-\xi'_{-k}(\cdot|s_{-k})}_1 \\
		&+\frac{1}{(1-\gamma)^2}\sup_{s''}\norm{\xi'(\cdot|s'')-\xi(\cdot|s'')}_1.
	\end{align*}
	
	In particular, by choosing $\xi\leftarrow \hat \xi$, $\xi'\leftarrow \xi^{\theta}$ and applying Lemma \ref{lem:prod_pol_diff}, we have 
	\begin{align*}
		\abs{\overline Q_i^{\hat\xi}(s,a_i)-\overline Q_i^{\xi^{\theta}}(s,a_i)}
		\leq & 2(n-1)\epsilon+\frac{1}{1-\gamma}\cdot 2(n-1)\epsilon+\frac{1}{(1-\gamma)^2}\cdot 2n\epsilon \\
		\leq &\frac{6n\epsilon}{(1-\gamma)^2}.
	\end{align*}
	Taking maximization on both sides w.r.t. $s$ and $a_i$, we get
	\begin{align*}
		\max_{s,a_i}\abs{\overline Q_i^{\hat\xi}(s,a_i)-\overline Q_i^{\xi^{\theta}}(s,a_i)}^2\leq \Sp{\frac{6n\epsilon}{(1-\gamma)^2}}^2.    
	\end{align*}
\end{proof}
Finally, apply Corollary \ref{cor:critic_restate} and Lemma \ref{lem:Q_perf_diff} to Eq. (\ref{eq:comb_critic_decompose}), and we complete the proof.

\section{Proof of Theorem \ref{thm:main}}
For any $i\in\mathcal{N}$ and $m\geq 0$, we have
\begin{align}
	&\Phi_i(\theta(m+1))-\Phi_i(\theta(m) \nonumber\\
	=\;&\left[\Phi_i(\theta(m+1))-\Phi_i(\theta_{\nik}(m+1),\theta_{-\nik}(m))\right] \nonumber\\
	&+ \left[ \Phi_i(\theta_{\nik}(m+1),\theta_{-\nik}(m)) - \Phi_i(\theta(m))\right]\label{eq1:decomposition2}.
\end{align}
Similar to the proof of Theorem \ref{thm:tab_softmax_Nash_regret}, using Assumption \ref{assump:pot_exp_decay} and the first term on the RHS of Eq. (\ref{eq1:decomposition2}) can be lower-bounded as 
\begin{align*}
	\Phi_i(\theta(m+1))-\Phi_i(\theta_{\nik}(m+1),\theta_{-\nik}(m))\geq -\frac{\sqrt{2}\nu(\kappa)\beta}{(1-\gamma)^2}.
\end{align*}
Now consider the second term. Denote $e_j(m)=\nabla_{\theta_j} J_j(\theta(m))-\Delta_j^T(m)$. Using the smoothness property (Lemma \ref{lem:value_smooth}) of the local potential functions and we have
\begin{align*}
	&\Phi_i(\theta_{\nik}(m+1),\theta_{-\nik}(m)) - \Phi_i(\theta(m)) \\
	\geq\;& \langle\nabla_{\theta_{\nik}}\Phi_i(\theta(m)),\theta_{\nik}(m+1)-\theta_{\nik}(m)\rangle - \frac{L(\kappa)}{2} \norm{\theta_{\nik}(m+1)-\theta_{\nik}(m)}^2 \\
	= \;&\sum_{j\in \nik}\left[\beta\langle \nabla_{\theta_j} J_j(\theta(m)),\Delta_j^T(m)\rangle-\frac{L(\kappa)\beta^2}{2} \norm{\Delta_j^T(m)}^2\right] \\
	=\; &\sum_{j\in \nik}\left[\beta\langle \nabla_{\theta_j} J_j(\theta(m)),\nabla_{\theta_j}J_j(\theta(m))+e_j(m)\rangle-\frac{L(\kappa)\beta^2}{2} \norm{\nabla_{\theta_j} J_j(\theta(m)+e_j(m)}^2\right]\\
	\geq \; &\sum_{j\in \nik}\left[(\beta-L(\kappa)\beta^2)\| \nabla_{\theta_j} J_j(\theta(m))\|^2+\beta\langle \nabla_{\theta_j}J_j(\theta(m)),e_j(m)\rangle-L(\kappa)\beta^2 \norm{e_j(m)}^2\right],
\end{align*}
where the last line follows from $(a+b)^2\leq 2(a^2+b^2)$ for any $a,b\in\mathbb{R}$. Using the previous two inequalities in Eq. (\ref{eq1:decomposition2}) and we have
\begin{align*}
	&\mathbb{E}[\Phi_i(\theta(m+1))\mid \mathcal{F}_m]-\mathbb{E}[\Phi_i(\theta(m)]\\
	\geq \;&-\frac{\sqrt{2}\nu(\kappa)\beta}{(1-\gamma)^2}+(\beta-L(\kappa)\beta^2)\sum_{j\in \nik}\| \nabla_{\theta_j} J_j(\theta(m))\|^2\\
	&+\beta\sum_{j\in \nik}\langle \nabla_{\theta_j}J_j(\theta(m)),\mathbb{E}[e_j(m)\mid \mathcal{F}_m]\rangle-L(\kappa)\beta^2 \sum_{j\in \nik}\mathbb{E}[\norm{e_j(m)}^2\mid \mathcal{F}_m]\\
	\geq \;&-\frac{\sqrt{2}\nu(\kappa)\beta}{(1-\gamma)^2}+(\beta-L(\kappa)\beta^2)\sum_{j\in \nik}\| \nabla_{\theta_j} J_j(\theta(m))\|^2\\
	&-\frac{\beta}{2}\sum_{j\in \nik}(\| \nabla_{\theta_j}J_j(\theta(m))\|^2+\|\mathbb{E}[e_j(m)\mid \mathcal{F}_m]\|^2)-L(\kappa)\beta^2 \sum_{j\in \nik}\mathbb{E}[\norm{e_j(m)}^2\mid \mathcal{F}_m]\\
	=\;&-\frac{\sqrt{2}\nu(\kappa)\beta}{(1-\gamma)^2}+\left(\frac{\beta}{2}-L(\kappa)\beta^2\right)\sum_{j\in \nik}\| \nabla_{\theta_j} J_j(\theta(m))\|^2\\
	&-\frac{\beta}{2}\sum_{j\in \nik}\|\mathbb{E}[e_j(m)\mid \mathcal{F}_m]\|^2-L(\kappa)\beta^2 \sum_{j\in \nik}\mathbb{E}[\norm{e_j(m)}^2\mid \mathcal{F}_m],
\end{align*}
where $\mathcal{F}_m$ represents the history up to the beginning of the $m$-th outer-loop iteration. Taking the total expectation on both sides of the previous inequality and we have
\begin{align*}
	&\mathbb{E}[\Phi_i(\theta(m+1))]-\mathbb{E}[\Phi_i(\theta(m)]\\
	\geq \;&-\frac{\sqrt{2}\nu(\kappa)\beta}{(1-\gamma)^2}+\left(\frac{\beta}{2}-L(\kappa)\beta^2\right)\sum_{j\in \nik}\mathbb{E}[\| \nabla_{\theta_j} J_j(\theta(m))\|^2]\\
	&-\frac{\beta}{2}\sum_{j\in \nik}\mathbb{E}[\|\mathbb{E}[e_j(m)\mid \mathcal{F}_m]\|^2]-L(\kappa)\beta^2 \sum_{j\in \nik}\mathbb{E}[\norm{e_j(m)}^2],
\end{align*}
which implies
\begin{align}
	&\frac{1}{M}\mathbb{E}[\Phi_i(\theta(M))]-\mathbb{E}[\Phi_i(\theta(0)]\nonumber\\
	\geq \;&-\frac{\sqrt{2}\nu(\kappa)\beta}{(1-\gamma)^2}+\frac{(\beta/2-L(\kappa)\beta^2)}{M}\sum_{m=0}^{M-1}\sum_{j\in \nik}\mathbb{E}[\| \nabla_{\theta_j} J_j(\theta(m))\|^2]\nonumber\\
	&-\frac{\beta}{2M}\sum_{m=0}^{M-1}\sum_{j\in \nik}\mathbb{E}[\|\mathbb{E}[e_j(m)\mid \mathcal{F}_m]\|^2]-\frac{L(\kappa)\beta^2}{M} \sum_{m=0}^{M-1}\sum_{j\in \nik}\mathbb{E}[\norm{e_j(m)}^2].\label{eq:before1}
\end{align}
Since $\Phi_{\min}\leq \Phi_i(\theta)\leq \Phi_{\max}$ for all $\theta$ and $\beta\leq \frac{1}{4L(\kappa)}$, after rearranging terms and we have
\begin{align*}
	&\frac{1}{M}\sum_{m=0}^{M-1}\sum_{j\in \nik}\mathbb{E}[\| \nabla_{\theta_j} J_j(\theta(m))\|^2]\\
	\leq \;&\frac{4(\Phi_{\max}-\Phi_{\min})}{\beta M}+\frac{4\sqrt{2}\nu(\kappa)}{(1-\gamma)^2}+\frac{2}{M}\sum_{m=0}^{M-1}\sum_{j\in \nik}\underbrace{\mathbb{E}[\|\mathbb{E}[e_j(m)\mid \mathcal{F}_m]\|^2]}_{\mathcal{T}_1}\\
	&+\frac{4L(\kappa)\beta}{M} \sum_{m=0}^{M-1}\sum_{j\in \nik}\underbrace{\mathbb{E}[\norm{e_j(m)}^2]}_{\mathcal{T}_2}.
\end{align*}
Next, we bound the terms $\mathcal{T}_1$ and $\mathcal{T}_2$ from above. To begin with, we decompose $e_i(m)$ in the following way:
\begin{align*}
	e_i(m)
	=\;&\nabla_{\theta_i}J_i(\theta(m))-\Delta_i^T(m)\\
	=\;&
	\sum_{k=0}^{\infty}\gamma^k \E\left[\nabla_{\theta_i}\log \xi_i^{\theta_i(m)}(a_i(k)|s_i(k))  \overline Q_{i}^{\theta(m)}(s(k),a_i(k))\right]\\
	&-\frac{1}{T}\sum_{t=0}^{T-1}\sum_{k=0}^{H-1}\gamma^k \nabla_{\theta_i}\log \xi_i^{\theta_i(m)}(a_i^t(k)|s_i^t(k)) \phi_i(s_{\mathcal{N}_i^{\kappa_c}}^t(k),a_i^t(k))^\top  w_i^m \\
	=\;&\sum_{k=H}^{\infty}\gamma^k \E\left[\nabla_{\theta_i}\log \xi_i^{\theta_i(m)}(a_i(k)|s_i(k))  \overline Q_{i}^{\theta(m)}(s(k),a_i(k))\right]\tag{$\mathcal{X}_1$}\\
	&+\sum_{k=0}^{H-1}\gamma^k \E\left[\nabla_{\theta_i}\log \xi_i^{\theta_i(m)}(a_i(k)|s_i(k))  \overline Q_{i}^{\theta(m)}(s(k),a_i(k))\right]\tag{$\mathcal{X}_2$}\\
	&-\frac{1}{T}\sum_{t=0}^{T-1}\sum_{k=0}^{H-1}\gamma^k \nabla_{\theta_i}\log \xi_i^{\theta_i(m)}(a_i^t(k)|s_i^t(k))\overline Q_{i}^{\theta(m)}(s^t(k),a_i^t(k))\tag{$\mathcal{X}_3$}\\
	&+\frac{1}{T}\sum_{t=0}^{T-1}\sum_{k=0}^{H-1}\gamma^k \nabla_{\theta_i}\log \xi_i^{\theta_i(m)}(a_i^t(k)|s_i^t(k))\\
	&\times (\overline Q_{i}^{\theta(m)}(s^t(k),a_i^t(k))-\phi_i(s_{\mathcal{N}_i^{\kappa_c}}^t(k),a_i^t(k))^\top  w_i^m).\tag{$\mathcal{X}_4$}
\end{align*}
It follows that
\begin{align*}
	\mathbb{E}[e_i(m)\mid \mathcal{F}_m]
	=\;&\mathcal{X}_1+\mathcal{X}_2+\mathbb{E}[\mathcal{X}_3\mid \mathcal{F}_m]+\mathbb{E}[\mathcal{X}_4\mid \mathcal{F}_m]\\
	=\;&\mathcal{X}_1+\mathbb{E}[\mathcal{X}_4\mid \mathcal{F}_m].
\end{align*}
Therefore, using Lemma \ref{lem:bound_grad_policy}, the definition of $\epsilon_{\text{critic}}$, and the fact that the averaged $Q$-function is bounded (in $\ell_\infty$-norm) by $\frac{1}{1-\gamma}$, and we have
\begin{align*}
	\|\mathbb{E}[e_i(m)\mid \mathcal{F}_m]\|\leq \;&\|\mathcal{X}_1\|+\|\mathbb{E}[\mathcal{X}_4\mid \mathcal{F}_m]\|\\
	\leq \;&\frac{\sqrt{2}\gamma^H}{(1-\gamma)^2}+\frac{\sqrt{2}\epsilon_{\text{critic}}}{1-\gamma},
\end{align*}
which implies that
\begin{align*}
	\mathcal{T}_1=\mathbb{E}[\|\mathbb{E}[e_j(m)\mid \mathcal{F}_m]\|^2]\leq \frac{4\gamma^{2H}}{(1-\gamma)^4}+\frac{4\epsilon^2_{\text{critic}}}{(1-\gamma)^2}.
\end{align*}
As for the term $\mathcal{T}_2$, similarly we have
\begin{align*}
	\mathcal{T}_2=\;&\mathbb{E}[\|e_j(m)\|^2]\\
	\leq \;&\mathbb{E}[\|\mathcal{X}_1+\mathcal{X}_2+\mathcal{X}_3+\mathcal{X}_4\|^2]\\
	\leq \;&3\mathbb{E}[\|\mathcal{X}_1\|^2+\|\mathcal{X}_2+\mathcal{X}_3\|^2+\|\mathcal{X}_4\|^2]\\
	\leq \;&\frac{6\gamma^{2H}}{(1-\gamma)^4}+\frac{24}{(1-\gamma)^2T}+\frac{6\epsilon^2_{\text{critic}}}{(1-\gamma)^2}.
\end{align*}
Substituting the upper bounds we obtained for the terms $\mathcal{T}_1$ and $\mathcal{T}_2$ in Eq. (\ref{eq:before1}) and we obtain
\begin{align*}
	&\frac{1}{M}\sum_{m=0}^{M-1}\sum_{j\in \nik}\mathbb{E}[\| \nabla_{\theta_j} J_j(\theta(m))\|^2]\\
	\leq \;&\frac{4(\Phi_{\max}-\Phi_{\min})}{\beta M}+\frac{4\sqrt{2}\nu(\kappa)}{(1-\gamma)^2}+\frac{8n(\kappa)\gamma^{2H}}{(1-\gamma)^4}+\frac{8n(\kappa)\epsilon^2_{\text{critic}}}{(1-\gamma)^2}\\
	&+24L(\kappa) n(\kappa)\beta\left(\frac{\gamma^{2H}}{(1-\gamma)^4}+\frac{\epsilon^2_{\text{critic}}}{(1-\gamma)^2}+\frac{4}{(1-\gamma)^2T}\right).
\end{align*}
Recall that Lemma \ref{lem:NE_bound_grad_Phi} implies
\begin{align*}
	\sum_{j\in \nik}\norm{\nabla_{\theta_{j}}J_j(\theta(m))}^2
	\geq \norm{\nabla_{\theta_{i}}J_i(\theta(m))}^2 
	\geq \frac{c^2}{\max_{j\in\sN} |\sA_j|D^2} \text{NE-Gap}_i(\theta(m))^2.
\end{align*} 
Therefore, by choosing $\kappa=\kappa_G$ and $\beta=\frac{1}{8L(\kappa)}$, we have
\begin{align*}
	&\frac{1}{M}\sum_{m=0}^{M-1}\mathbb{E}[\text{NE-Gap}_i(\theta(m))^2]\\
	\leq \;&\frac{\max_{j\in\sN} |\sA_j|D^2}{c^2}\bigg[\frac{4(\Phi_{\max}-\Phi_{\min})}{\beta M}+\frac{4\sqrt{2}\nu(\kappa)}{(1-\gamma)^2}+\frac{8n(\kappa)\gamma^{2H}}{(1-\gamma)^4}+\frac{8n(\kappa)\epsilon^2_{\text{critic}}}{(1-\gamma)^2}\\
	&+24L(\kappa) n(\kappa)\beta\left(\frac{\gamma^{2H}}{(1-\gamma)^4}+\frac{\epsilon^2_{\text{critic}}}{(1-\gamma)^2}+\frac{4}{(1-\gamma)^2T}\right)\bigg]\\
	\leq \;&\frac{12 n(\kappa_G)\max_{j\in\sN} |\sA_j|D^2}{c^2(1-\gamma)^2}\bigg[\frac{16(\Phi_{\max}-\Phi_{\min})}{ M(1-\gamma)}+\frac{\nu(\kappa_G)}{n(\kappa_G)}+\frac{\gamma^{2H}}{(1-\gamma)^2}+\epsilon^2_{\text{critic}}+\frac{1}{T}\bigg].
\end{align*}
Finally, using Jensen's inequality and we obtain
\begin{align*}
	&\mathbb{E}[\regret_i(M)]\\
	=\;&\frac{1}{M}\sum_{m=0}^{M-1}\mathbb{E}[\text{NE-Gap}_i(\theta(m))]\\
	\leq \;&\frac{4 n(\kappa_G)^{1/2}\max_{j\in\sN} \sqrt{|\sA_j|}}{c(1-\gamma)}\bigg[\frac{4(\Phi_{\max}-\Phi_{\min})^{1/2}}{ M^{1/2}(1-\gamma)^{1/2}}+\frac{\nu(\kappa_G)^{1/2}}{n(\kappa_G)^{1/2}}+\frac{\gamma^{H}}{1-\gamma}+\epsilon_{\text{critic}}+\frac{1}{T^{1/2}}\bigg].
\end{align*}
Since the RHS of the bound is not a function of $i$, we have the desired result.

\section{Supporting Results}
\subsection{NMPG is a strict generalization of MPG} \label{subsec:ex_general}
We present a simple example to show that NMPG is a strict generalization of the standard MPG. 

\noindent \textbf{Example:} Consider a networked multi-agent Markov game, denoted by $MG$,  with 4 agents $\mathcal N=\{1,2,3,4\}$. Each agent has local state space $\mathcal S_i=\{s_b,s_g\}$ and local action space $\mathcal A_i=\{a_b,a_g\}$ for all agent $i\in \mathcal N$. The underlying undirected graph connects edges between every ``neighboring'' agents, i.e., the set of edges is $\mathcal E = \Bp{(1,2),(2,3),(3,4)}$. 
\begin{itemize}
	\item Initial state: All agents starts at state $s_b$. That is, $s_i(0)=s_b$ for all $i\in \N$.
	\item Transition: $MG$ has deterministic transitions. For agent 1, the next state only depends on its own action: For all $t\geq 0$, $s_1(t+1)=s_g$ if $a_1(t)=a_g$ and $s_1(t+1)=s_b$ if $a_1(t)=a_b$. For agent $i\in \{2,3,4\}$, the next state only depends on state of "previous" agent, i.e., $s_{i}(t+1)=s_{i-1}(t)$, for all $t\geq 0$. 
	\item Reward: Each agent's reward is only dependent on its own state and action. The local reward of agent 1,2,3 is always 0. That is, $r_i(s_i,a_i) = 0$ for all $i\in \Bp{1,2,3}$, $s_i\in \sS_i$, $a_i\in \sA_i$. For agent 4, it receives reward 1 if it is at state $s_g$ and takes action $a_g$ while it receives reward 0 in all other cases. In other words, $r_4(s_g,a_g)=1$, $r_4(s_g,a_b)=0$, $r_4(s_b,a_g)=0$, $r_4(s_b,a_b)=0$.
\end{itemize}
Obviously, the expected return of agent 1,2 and 3 is always 0. As for agent 4, its state is solely dependent on the state of agent 1 (three timesteps before), which is completely determined by agent 1's policy. As a result, agent 4's expected return, or objective function, depends on local policies of \emph{both} agent 1 and 4. This observation is the key to showing that $MG$ is not an MPG while being a 1-NMPG, which is summarized as the theorem below. 

\begin{theorem}\label{thm:ex_MG}
	$MG$ is a 1-NMPG but not an MPG. 
\end{theorem}

\begin{proof}[Proof of Theorem \ref{thm:ex_MG}]
	We first compute the objective function $J_i(\xi)$ for all global policy $\xi$ and agent $i\in \sN$. Noticing that agents 1,2,3 always receive reward 0, we have $J_i(\xi)=0$ for all $i\in \Bp{1,2,3}$. As for agent 4, we have 
	\begin{align}
		J_4(\xi) 
		=& \sum_{t=0}^{\infty}\gamma^t \E_{\xi}\Mp{r_4(s_4(t),a_4(t)} \nonumber \\
		=& \sum_{t=0}^{\infty}\gamma^t \Pr[r_4(s_4(t),a_4(t)=1] \nonumber \\
		=& \sum_{t=0}^{\infty}\gamma^t \Pr[s_4(t)=s_g,a_4(t)=a_g] \nonumber \\
		\osi{=}& \sum_{t=4}^{\infty}\gamma^t \Pr[a_1(t-4)=a_g,a_4(t)=a_g] \label{eq:comp_f} \\
		\triangleq & f(\xi_1,\xi_4) \nonumber. 
	\end{align}
	Here $\ri$ uses the fact that $s_4(t)=s_b$ for $t=0,1,2,3$, and $s_4(t)=s_1(t-3)=s_g$ if and only if $a_1(t-4)=a_g$ for all $t\geq 4$.
	
	Next we show that $MG$ is a 1-NMPG. In fact, we can choose local potential functions $\Phi_1(\xi)=\Phi_2(\xi)=0$, $\Phi_3(\xi)=\Phi_4(\xi)=f(\xi_1,\xi_4)$ for all global policy $\xi$, and Eq. (\ref{eq:approx_MPG}) can be satisfied.
	
	Finally, we prove that $MG$ is not an MPG. Assume that there is a potential $\Phi(\xi)$ such that \begin{align}
		\label{eq:def_pot}
		J_i(\xi_i',\xi_{-i})-J_i({\xi_i,\xi_{-i}}) = \Phi(\xi_i',\xi_{-i})-\Phi(\xi_i,\xi_{-i})
	\end{align}
	for all $i\in \N$, $\xi_i,\xi_i'\in \Xi_i$, $\xi_{-i}\in \Xi_{-i}$. If we choose $i\in \{1,2,3\}$ in Eq. (\ref{eq:def_pot}), we can see that $\Phi(\xi_i',\xi_{-i})-\Phi(\xi_i,\xi_{-i})=0$ for all $\xi_i,\xi_i'\in \Xi_i$, $\xi_{-i}\in \Xi_{-i}$. Thus $\Phi$ is independent of $\xi_i$ for all $i\in \Bp{1,2,3}$. As a result, the potential function can be represented as $\Phi(\xi_4)$. Then we let $i=4$ in Eq. (\ref{eq:def_pot}), and we have for all $\xi_1',\xi_1\in \Xi_1$, $\xi_4'\in \Xi_4$ that
	\begin{align}
		\label{eq:test_j4}
		\Phi(\xi_4')-\Phi(\xi_4) = J_4(\xi_4',\xi_{-4})-J_4(\xi_4,\xi_{-4}) = f(\xi_1,\xi_4')-f(\xi_1,\xi_4).
	\end{align}
	To derive a contradiction, we take some special values of $\xi_1,\xi_1', \xi_4$. For any $i\in \N$, let $\xi_i^g$ be the policy that agent $i$ always takes $a_g$ and $\xi_i^b$ be the policy that agent $i$ always takes $a_b$. Then we can derive from Eq. (\ref{eq:comp_f}) that $f(\xi_1^g,\xi_4^g)=\frac{\gamma^4}{1-\gamma}$, $f(\xi_1^g,\xi_4^b)=f(\xi_1^b,\xi_4^g)=f(\xi_1^b,\xi_4^b)=0$. In Eq. (\ref{eq:test_j4}), let $\xi_4'=\xi_4^g$, $\xi_4=\xi_4^b$, $\xi_1\in \{\xi_1^b,\xi_1^g\}$, and we have 
	\begin{align*}
		\begin{cases}
			\Phi(\xi_4^g)-\Phi(\xi_4^b) = f(\xi_1^b,\xi_4^g)-f(\xi_1^b,\xi_4^b) = 0 \\
			\Phi(\xi_4^g)-\Phi(\xi_4^b) = f(\xi_1^g,\xi_4^g)-f(\xi_1^g,\xi_4^b) = \frac{\gamma^4}{1-\gamma},
		\end{cases}
	\end{align*}
	which leads to contradiction. As a result, $MG$ is not an MPG.
\end{proof}

$\kappa_G$-NMPG in general cannot be reduced to a standard MPG. Please note that in a $\kappa_G$-NMPG, for any agent $i$, any agent $j\in N_i^{\kappa_G}$ do share a potential function $\Phi_i$, but this potential is associated with the agent $i$. In the reviewer's example, the potential functions $\Phi_{i_1}, \Phi_{i_2},\cdots,\Phi_{i_{\kappa_G+1}}$ are associated with agent $i_1$, while the potential functions $\Phi_{i_L}, \Phi_{i_{L-1}}, \cdots,\Phi_{i_{L-\kappa_G}}$ are associated with agent $i_L$. Therefore, unless $L\leq 2\kappa_G-1$, $i_1$ and $i_L$ are not guaranteed to share a same potential function.

In addition, we can give a simple example that is a  for your reference.  The transition is deterministic.  The initial states of the 4 agents are all $s_b$. Then $J_1(\xi)=J_2(\xi)=J_3(\xi)=0$.  

$$
J_4(\xi)=\sum_{t=4}^{\infty}\gamma^t\Pr[a_1(t-4)=a_g]\Pr[a_4(t)=a_g]\triangleq f(\xi_1,\xi_4).
$$
The above example is a 1-NMPG, with $\Phi_1(\xi)=\Phi_2(\xi)=0$, $\Phi_3(\xi)=\Phi_4(\xi)=f(\xi_1,\xi_4)$, but it is not an MPG.

\subsection{Policy Gradient Theorem Variant}\label{ap:PG}
We prove Eq. (\ref{eq:PGT}), which is a variant of the policy gradient theorem \citep{sutton1999PGT}.
\begin{lemma}[Policy gradient theorem variant] \label{lem:PGT_MARL_main}
	\begin{align*}
		&\nabla_{\theta_i}J_i(\theta)
		=\sum_{t=0}^{\infty}\gamma^t \E_{\xi^{\theta}}\left[\nabla_{\theta_i}\log \xi_i^{\theta_i}(a_i(t)|s_i(t))\overline Q_{i}^{\theta}(s(t),a_i(t))\right].
	\end{align*}
\end{lemma}

\begin{proof}[Proof of Lemma \ref{lem:PGT_MARL_main}]
	By Lemma \ref{lem:PGT_MARL_compact}, we have
	\begin{align*}
		&\nabla_{\theta_i}J_i(\theta) \\
		=&\frac{1}{1-\gamma}\sum_{s,a_i}d^{\theta}(s)\nabla_{\theta_i}\xi_i^{\theta_i}(a_i|s_i)\overline Q_{i}^{\theta}(s,a_i) \\
		=& \sum_{s,a_i} \sum_{t=0}^{\infty}\gamma^t{\Pr}^{\xi^{\theta}}[s(t)=s|s(0)\sim \mu(\cdot)]\nabla_{\theta_i}\xi_i^{\theta_i}(a_i|s_i)\overline Q_{i}^{\theta}(s,a_i) \\
		=& \sum_{t=0}^{\infty} \sum_{s,a_i} \gamma^t{\Pr}^{\xi^{\theta}}[s(t)=s|s(0)\sim \mu(\cdot)]\xi_i^{\theta_i}(a_i|s_i)\nabla_{\theta_i}\log\xi_i^{\theta_i}(a_i|s_i)\overline Q_{i}^{\theta}(s,a_i) \\
		=&\sum_{t=0}^{\infty}\gamma^t \E_{\xi^{\theta}}\left[\nabla_{\theta_i}\log \xi_i^{\theta_i}(a_i(t)|s_i(t))\overline Q_{i}^{\theta}(s(t),a_i(t))\right].
	\end{align*}
	Here ${\Pr}^{\xi^{\theta}}[s(t)=s|s(0)\sim \mu(\cdot)]$ represents the probability that $s(t)=s$ given that the policy is $\xi^{\theta}$ and initial state $s(0)$ is sampled from distribution $\mu$.
\end{proof}

\subsection{Proof of Lemma \ref{le:truncated_Q}}\label{pf:le:decay}
When $\kappa_c \leq \kappa_r-1$, the conclusion is obvious by the fact that both the truncated averaged $Q$-function and the averaged $Q$-function are in the range of $[0,\frac{1}{1-\gamma}]$. Below we only consider the case that $\kappa_c \geq \kappa_r$.

Consider any agent $i$, global policy parameter $\theta$, and any truncated averaged $Q$-function $\overline Q_{i}^{\theta,\kappa_c}\in \mathcal Q_i^{\theta,\kappa_c}$. Then there exists $u_i\in \Delta(\mathcal{S}_{-\nikc})$, such that $\overline Q_i^{\theta,\kappa_c}(s_{\nikc},a_{i}) =\sum_{s_{-\nikc}} u_i(s_{-\nikc}) \overline Q_i^{\theta}(s_{\nikc},s_{-\nikc},a_i)$ for any $(s_\nikc, a_i)\in \sS_\nikc \times \sA_i$.  We have for any $(s,a_i)\in \sS\times \sA_i$ that
\begin{align}
	&\abs{\overline Q_{i}^{\theta,\kappa_c}(s_{\nikc},a_{i})-\overline Q_i^{\theta}(s,a_i)} \nonumber \\
	=&\abs{\sum_{s_{-\nikc}'}u_i(s_{-\nikc}') \left(\overline Q_i^{\theta}(s_{\nikc},s_{-\nikc}',a_i)-\overline Q_i^{\theta}(s_{\nikc},s_{-\nikc},a_i)\right)} \nonumber \\
	\leq &\sum_{s_{-\nikc}'}u_i(s_{-\nikc}') \abs{\overline Q_i^{\theta}(s_{\nikc},s_{-\nikc}',a_i)-\overline Q_i^{\theta}(s_{\nikc},s_{-\nikc},a_i)} \nonumber \\
	\leq &\sum_{s_{-\nikc}'}\sum_{a_{-i}}u_i(s_{-\nikc}') \xi_{-i}^{\theta_{-i}}(a_{-i}|s_{-i})\abs{ Q_i^{\theta}(s_{\nikc},s_{-\nikc}',a_i,a_{-i})-Q_i^{\theta}(s_{\nikc},s_{-\nikc},a_i,a_{-i})}. \label{eq:trunc_expand}
\end{align}
We now try to give a perturbation bound for agent $i$'s $Q$-function w.r.t. $s_{\nikc}$, the states of agents in $\kappa$-hop neighborhood. 
Notice that for any $s\in \sS,a\in \sA$,
\begin{align*}
	&Q_i^{\theta}(s,a) \\
	= & \sum_{t=0}^{\infty} \gamma^t \sum_{s(t),a(t)}{\Pr}^{\xi^\theta}(s(t)|s(0)=s,a(0)=a) \xi^{\theta}(a(t)|s(t))r_i(s_{\mathcal{N}_i^{\kappa_r}}(t),a_{\mathcal{N}_i^{\kappa_r}}(t) \\
	=& \sum_{t=0}^{\infty} \gamma^t \sum_{s_{\mathcal{N}_i^{\kappa_r}}(t),a_{\mathcal{N}_i^{\kappa_r}}(t)}{\Pr}^{\xi^\theta}(s_{\mathcal{N}_i^{\kappa_r}}(t),a_{\mathcal{N}_i^{\kappa_r}}(t)|s(0)=s,a(0)=a) r_i(s_{\mathcal{N}_i^{\kappa_r}}(t),a_{\mathcal{N}_i^{\kappa_r}}(t)) \\
	=& \sum_{t=0}^{\infty} \gamma^t \sum_{s_{\mathcal{N}_i^{\kappa_r}}(t),a_{\mathcal{N}_i^{\kappa_r}}(t)}{\Pr}^{\xi^\theta}(s_{\mathcal{N}_i^{\kappa_r}}(t),a_{\mathcal{N}_i^{\kappa_r}}(t)|s(0)=s,a(0)=a) r_i(s_{\mathcal{N}_i^{\kappa_r}}(t),a_{\mathcal{N}_i^{\kappa_r}}(t))
\end{align*}
for simplicity of notation. 

For any fixed $i\in \sN$, $s_{\nikc}\in \sS_{\nikc}$, $s_{-\nikc},s_{-\nikc}'\in \sS_{-\nikc}$,$a\in \sA$ and policy $\xi^{\theta}$, let 
\begin{align*}
	&\pi_{t}^{\kappa_r}(s_{\mathcal{N}_i^{\kappa_r}}(t),a_{\mathcal{N}_i^{\kappa_r}}(t))={\Pr}^{\xi^\theta}(s_{\mathcal{N}_i^{\kappa_r}}(t),a_{\mathcal{N}_i^{\kappa_r}}(t)|s(0)=(s_{\nikc},s_{-\nikc}),a(0)=a) \\
	&\tilde \pi_{t}^{\kappa_r}(s_{\mathcal{N}_i^{\kappa_r}}(t),a_{\mathcal{N}_i^{\kappa_r}}(t))={\Pr}^{\xi^\theta}(s_{\mathcal{N}_i^{\kappa_r}}(t),a_{\mathcal{N}_i^{\kappa_r}}(t)|s(0)=(s_{\nikc},s_{-\nikc}'),a(0)=a).
\end{align*}
Due to the local dependence of network and localized policy structure, $\pi_{t}^{\kappa_r}$ is only dependent on the initial states and actions of agents in $\mathcal{N}_i^{\kappa_r+t}$, which is equal to $\tilde \pi_{t}^{\kappa_r}$ when $t\leq \kappa_c-\kappa_r$. Therefore we have
\begin{align*}
	&\abs{ Q_i^{\theta}(s_{\nikc},s_{-\nikc}',a_i,a_{-i})-Q_i^{\theta}(s_{\nikc},s_{-\nikc},a_i,a_{-i})} \\
	=& \left|\sum_{t=0}^{\infty} \gamma^t \sum_{s_{\mathcal{N}_i^{\kappa_r}}(t),a_{\mathcal{N}_i^{\kappa_r}}(t)} \left(\xi_{t}^{\kappa_r}(s_{\mathcal{N}_i^{\kappa_r}}(t),a_{\mathcal{N}_i^{\kappa_r}}(t))-\tilde \pi_{t}^{\kappa_r}(s_{\mathcal{N}_i^{\kappa_r}}(t),a_{\mathcal{N}_i^{\kappa_r}}(t))\right)\right.\\
	&\left.\times r_i(s_{\mathcal{N}_i^{\kappa_r}}(t),a_{\mathcal{N}_i^{\kappa_r}}(t))\right| \\
	\leq & \sum_{t=0}^{\infty} \gamma^t \sum_{s_{\mathcal{N}_i^{\kappa_r}}(t),a_{\mathcal{N}_i^{\kappa_r}}(t)} \abs{\xi_{t}^{\kappa_r}(s_{\mathcal{N}_i^{\kappa_r}}(t),a_{\mathcal{N}_i^{\kappa_r}}(t))-\tilde \pi_{t}^{\kappa_r}(s_{\mathcal{N}_i^{\kappa_r}}(t),a_{\mathcal{N}_i^{\kappa_r}}(t))}\\
	&\times r_i(s_{\mathcal{N}_i^{\kappa_r}}(t),a_{\mathcal{N}_i^{\kappa_r}}(t)) \\
	\osi{\leq} & \sum_{t=0}^{\infty} \gamma^t \norm{\pi_{t}^{\kappa_r}-\tilde \pi_{t}^{\kappa_r}}_1 \\
	= & \sum_{t=\kappa_c-\kappa_r+1}^{\infty} \gamma^t \norm{\pi_{t}^{\kappa_r}-\tilde \pi_{t}^{\kappa_r}}_1 \\
	\leq & 2\sum_{t=\kappa_c-\kappa_r+1}^{\infty} \gamma^t \\
	\leq & \frac{2}{1-\gamma}\gamma^{\kappa_c-\kappa_r+1}.
\end{align*}
Here $\ri$ is by $r_i(s_{\nikr},a_{\nikr})\leq 1$ for any $i\in \sN$ and any $s_{\nikr}$, $a_{\nikr}$. Plug into Eq. (\ref{eq:trunc_expand}), and we have 
\begin{align*}
	& \abs{\overline Q_{i}^{\theta,\kappa_c}(s_{\nikc},a_{i})-\overline Q_i^{\theta}(s,a_i)} \\
	\leq & \sum_{s_{-\nikc}'}\sum_{a_{-i}}u_i(s_{-\nikc}') \xi_{-i}^{\theta_{-i}}(a_{-i}|s_{-i}) \frac{2}{1-\gamma}\gamma^{\kappa_c-\kappa_r+1} \\
	= & \frac{2}{1-\gamma}\gamma^{\kappa_c-\kappa_r+1}.
\end{align*}
Take $\max$ over $s, a_i$, $\sup$ over $\overline Q_{i}^{\theta,\kappa_c}$, and we complete the proof.

\subsection{Averaged Nash Regret}\label{ap:connection}
The relationship between the averaged Nash regret defined in this work and the Nash regret in \cite{pmlr-v162-ding22b} is shown in the following lemma. Recall that we denote $n=|\mathcal{N}|$.
\begin{lemma}\label{lem:avg_NR_bound}
	Given any positive integer $M$, the following inequality holds for any sequence of policies $\{\xi(0),\xi(1),\cdots,\xi(M-1)\}$:
	\begin{align*}
		\frac{1}{n}\text{Nash-Regret}(M)\leq \regret(M)\leq \text{Nash-Regret}(M).
	\end{align*}
\end{lemma}
\begin{proof}[Proof of Lemma \ref{lem:avg_NR_bound}]
	By definition of the averaged Nash Regret (cf. Definition \ref{def:avg_NR}) and the Nash Regret (cf. Eq. (\ref{eq:NR})), we have
	\begin{align*}
		\frac{1}{n}\text{Nash-Regret}(M)=\;&\frac{1}{n}\frac{1}{M}\sum_{m=0}^{M-1}\max_{i\in\mathcal{N}}\text{NE-Gap}_i(\xi(m))\\
		\leq \;&\frac{1}{n}\sum_{i=1}^n\frac{1}{M}\sum_{m=0}^{M-1}\text{NE-Gap}_i(\xi(m))\\
		\leq \;&\max_{i\in\mathcal{N}}\frac{1}{M}\sum_{m=0}^{M-1}\text{NE-Gap}_i(\xi(m))\\
		=\;&\regret(M)\\
		\leq \;&\frac{1}{M}\sum_{m=0}^{M-1}\max_{i\in\mathcal{N}}\text{NE-Gap}_i(\xi(m))\tag{Jensen's inequality}\\
		=\;&\text{Nash-Regret}(M).
	\end{align*}
\end{proof}

\subsection{Decay of Local Potential Functions}\label{subsec:supp_assump_decay}
We derive $\nu(\kappa)=O(\sum_{j\in -\nik}\abs{\sA_{j}})$ given a mild assumption, 
which guarantees existence of stage potential.
\begin{assumption}\label{assump:stage_pot}
	For any agent $i\in \sN$, there exists stage potential function $\varphi_i:\sS\times \sA\rightarrow [0,\overline\varphi]$, such that 
	\begin{align}
		\Phi_i(\theta)=\E_{\theta}\left[\sum_{t=0}^{\infty}\gamma^t\varphi_i(s(t),a(t))\right].
	\end{align}
\end{assumption}
This assumption is common in recent MPG literature \citep{zhang2022logBarrierSoftmax}.

\begin{lemma}\label{lem:nu_poly}
	With Assumption \ref{assump:stage_pot} satisfied, we have
	\begin{align}
		\abs{\Phi_i(\theta_{\nik},\theta_{-\nik}')-\Phi_i(\theta_{\nik},\theta_{-\nik})}\leq \frac{\sqrt{2}\overline{\varphi}}{(1-\gamma)^2}
		\sum_{j\in -\nik}\abs{\sA_j}\max_{j\in -\nik}\norm{\theta_j'-\theta_j}.
	\end{align}
\end{lemma}

\begin{proof}[Proof of Lemma \ref{lem:nu_poly}]
	Similar to $Q$-function and averaged $Q$-function, we can define ``$Q$-potential'' function and averaged ``$Q$-potential'' function as 
	\begin{align*}
		&Q\Phi_i^{\theta}(s,a)=\E_{\theta}\left[\sum_{t=0}^{\infty}\varphi_i(s(t),a(t))|s(0)=s,a(0)=a\right] \\
		&\overline{Q\Phi}_i^{\theta}(s,a_{-\nik})=\sum_{a_{\nik}}\xi_{\nik}^{\theta_{\nik}}(a_{\nik}|s_{\nik}) Q\Phi_i(s,a_{\nik},a_{-\nik}).
	\end{align*}
	In Lemma \ref{lem:PDT}, replace the objective function $J_i$ with potential function $\Phi_i$, treat agents in $-\nik$ as one agent, and we have
	\begin{align*}
		&\abs{\Phi_i(\theta_{\nik},\theta_{-\nik}')-\Phi_i(\theta_{\nik},\theta_{-\nik})} \\
		=&\frac{1}{1-\gamma}\sum_{s,a_{-\nik}}d^{\theta'}(s)\abs{\xi_{-\nik}^{\theta_{-\nik}'}(a_{-\nik}|s_{-\nik})-\xi_{-\nik}^{\theta_{-\nik}}(a_{-\nik}|s_{-\nik})}\overline {Q\Phi}_i^{\theta}(s,a_{-\nik}) \\
		\osi{\leq} &\frac{\overline{\varphi}}{(1-\gamma)^2}
		\sum_{s,a_{-\nik}}d^{\theta'}(s)\abs{\xi_{-\nik}^{\theta_{-\nik}'}(a_{-\nik}|s_{-\nik})-\xi_{-\nik}^{\theta_{-\nik}}(a_{-\nik}|s_{-\nik})}.
	\end{align*}
	Here $\ri$ uses the fact that $\overline {Q\Phi}_i^{\theta}(s,a_{-\nik})\leq \frac{\overline \varphi}{1-\gamma}$.
	By Lagrange mean value theorem, for any $j$,
	\begin{align*}
		&\abs{\xi_j^{\theta_j'}(a_j|s_j)-\xi_j^{\theta_j}(a_j|s_j)} \\
		\leq &\sup_{t\in [0,1], \hat\theta_j=t\theta_j'+(1-t)\theta_j} \norm{\nabla_{\theta_j}\xi_j^{\hat \theta_j}(a_j|s_j)}\norm{\theta_j'-\theta_j} \\
		\leq &\sqrt{2}\norm{\theta_j'-\theta_j}
	\end{align*}
	By Lemma \ref{lem:prod_pol_diff_general}, we derive that
	\begin{align*}
		&\sum_{a_{-\nik}}\abs{\xi_{-\nik}^{\theta_{-\nik}'}(a_{-\nik}|s_{-\nik})-\xi_{-\nik}^{\theta_{-\nik}}(a_{-\nik}|s_{-\nik})} \\
		\leq & \sum_{j\in -\nik} \norm{\xi_{j}^{\theta'_{j}}(\cdot|s_{j})-\xi_{j}^{\theta_{j}}(\cdot|s_{j})}_1 \\
		= & \sum_{j\in -\nik}\sum_{a_i}\abs{\xi_{j}^{\theta'_{j}}(a_{j}|s_{j})-\xi_{j}^{\theta_{j}}(a_{j}|s_{j})} \\
		\leq &\sqrt{2}\sum_{j\in -\nik}\abs{\sA_j}\norm{\theta_j'-\theta_j} \\
		\leq &\sqrt{2}\sum_{j\in -\nik}\abs{\sA_j}\max_{j\in -\nik}\norm{\theta_j'-\theta_j}.
	\end{align*}
	Therefore, we arrive at the conclusion
	\begin{align*}
		&\abs{\Phi_i(\theta_{\nik},\theta_{-\nik}')-\Phi_i(\theta_{\nik},\theta_{-\nik})} \\
		\leq &\frac{\overline{\varphi}}{(1-\gamma)^2}
		\sum_{s}d^{\theta'}(s)\sqrt{2}\sum_{j\in -\nik}\abs{\sA_j}\max_{j\in -\nik}\norm{\theta_j'-\theta_j} \\
		=&\frac{\sqrt{2}\overline{\varphi}}{(1-\gamma)^2}
		\sum_{j\in -\nik}\abs{\sA_j}\max_{j\in -\nik}\norm{\theta_j'-\theta_j}.
	\end{align*}
\end{proof}

\subsection{Boundedness of Local Potential Functions}\label{ap:boundedness_Phi}
\begin{lemma}\label{lem:pot_diff}
	For any agent $i\in \sN$, let $\xi_{\nikg}, \xi'_{\nikg} \in \Xi_{\nikg}$ and $\xi_{-\nikg} \in \Xi_{-\nikg}$ be arbitrary. Then we have
	\begin{align*}
		\Phi_i(\xi'_{\nikg},\xi_{-\nikg})-\Phi_i(\xi_{\nikg}, \xi_{-\nikg})\leq \frac{\abs{\nikg}}{1-\gamma}.
	\end{align*}
\end{lemma}

\begin{proof}[Proof of Lemma \ref{lem:pot_diff}]
	Suppose that $\nikg=\{i_1,i_2,\cdots,i_{k}\}$, where $k=\abs{\nikg}$. For any $r\in \{1,2,\cdots,k+1\}$, denote $\tilde{\xi}_{\nikg}^r=(\xi_{i_1},\cdots,\xi_{i_{r-1}},\xi'_{i_r},\cdots,\xi'_{i_{k}})$. Note that $\tilde{\xi}_{\nikg}^1=\xi_{\nikg}'$ and $\tilde{\xi}_{\nikg}^{k+1}=\xi_{\nikg}$. Then, we have by Definition \ref{def:local_MPG} that
	\begin{align*}
		\Phi_i(\xi'_{\nikg},\xi_{-\nikg})-\Phi_i(\xi_{\nikg}, \xi_{-\nikg}) 
		=\;& \sum_{r=1}^{k} \left[\Phi_i(\tilde{\xi}_{\nikg}^{r}, \xi_{\-\nikg})-\Phi_i(\tilde{\xi}_{\nikg}^{r+1}, \xi_{\-\nikg})\right] \\
		=\;& \sum_{r=1}^k\Mp{J_{i_r}(\tilde{\xi}_{\nikg}^r, \xi_{\-\nikg})-J_{i_r}(\tilde{\xi}_{\nikg}^{r+1}, \xi_{\-\nikg})} \\
		\leq \;& \sum_{r=1}^kJ_{i_r}(\tilde{\xi}_{\nikg}^r, \xi_{\-\nikg}) \\
		\leq \;& \frac{k}{1-\gamma}.
	\end{align*}
\end{proof}

\begin{lemma}\label{lem:bound_pot}
	For an arbitrary NMPG, there exist a set of local potential functions $\{\hat{\Phi}_i\}_{i\in \sN}$ and $\Phi_{\min},\Phi_{\max}>0$ satisfying $0\leq \Phi_{\max}-\Phi_{\min}\leq \frac{2n(\kappa_G)}{1-\gamma}$ such that $\Phi_{\min}\leq \hat{\Phi}_i(\xi)\leq \Phi_{\max}$ for all $i\in\mathcal{N}$ and $\xi\in \Xi$.
\end{lemma}
\begin{proof}[Proof of Lemma \ref{lem:bound_pot}]
	Let $\{\Phi_i\}_{i\in\mathcal{N}}$ be a set of local potential functions, and let $\Bar{\xi}\in\Xi$ be an arbitrary policy. Define $\{\hat{\Phi}_i\}_{i\in\mathcal{N}}$ as 
	\begin{align*}
		\hat{\Phi}_i(\xi)=\Phi_i(\xi)-\Phi_i(\Bar{\xi})+\frac{n(\kappa_G)}{1-\gamma}+1
	\end{align*}
	for all $i\in\mathcal{N}$ and $\xi\in\Xi$. It can be easily verified that $\{\hat{\Phi}_i\}_{i\in\mathcal{N}}$ is also a set of local potential functions. Now, for any $i\in\mathcal{N}$ and $\xi\in\Xi$, we have by Lemma \ref{lem:pot_diff} that
	\begin{align*}
		\hat{\Phi}_i(\xi)=\;& \hat{\Phi}_i(\xi)-\hat{\Phi}_i(\Bar{\xi})+\frac{n(\kappa_G)}{1-\gamma}+1\leq \frac{2n(\kappa_G)}{1-\gamma}+1,\\
		\hat{\Phi}_i(\xi)=\;&-\left(\Phi_i(\Bar{\xi})-\Phi_i(\xi)-\frac{n(\kappa_G)}{1-\gamma}-1\right)\geq 1.
	\end{align*}
	Therefore, we have
	\begin{align*}
		1\leq \Phi_i(\xi)\leq \frac{2n(\kappa_G)}{1-\gamma}+1,\quad \forall\;i\in\mathcal{N},\xi\in\Xi.
	\end{align*}
	The result follows by letting $\Phi_{\min}=1$ and $\Phi_{\max}=\frac{2n(\kappa_G)}{1-\gamma}+1$.
\end{proof}

\subsection{Other Technical Lemmas}

\begin{lemma}\label{lem:policy_grad}
	Under softmax parameterization with weights $\{\theta_i\}_{i\in\mathcal{N}}$, 
	the derivative of the policy $\xi^\theta$ is given by
	\begin{align*}
		\frac{\partial\xi_i^{\theta_i}(a_i|s_i)}{\partial \theta_{i,s_i',a_i'}}
		=\xi_i^{\theta_i}(a_i|s_i)\mathds{1}\{s_i'=s_i\}\left(\mathds{1}\{a_i'=a_i\}-\xi_i^{\theta_i}(a_i'|s_i)\right)
	\end{align*}
	for all $i\in\mathcal{N}$, $s_i,s_i'\in\mathcal{S}_i$, $a_i,a_i'\in\mathcal{A}_i$, and $\theta_i\in\mathbb{R}^{|\mathcal{S}_i||\mathcal{A}_i|}$.
\end{lemma}

\begin{proof}[Proof of Lemma \ref{lem:policy_grad}]
	When $s_i'\neq s_i$, since $\xi_i^{\theta_i}(a_i|s_i)$ is not a function of $\theta_{i,s_i',a_i'}$, the result clearly holds. We next consider the case where $s_i'=s_i$. 
	
	Given an arbitrary positive integer $d$, let $f:\mathbb{R}^d\mapsto\mathbb{R}^d$ be the softmax operator defined as
	\begin{align*}
		[f(x)](\ell)=\frac{\exp(x_\ell)}{\sum_{\ell'=1}^d\exp(x_{\ell'})},\quad \forall\;\ell\in \{1,2,\cdots,d\}.
	\end{align*}
	It was shown in \cite[Proposition 2]{gao2017properties} that $\nabla f(x)=\text{diag}(f(x))-f(x)f(x)^\top$.
	Therefore, we have for any $a_i'\in\mathcal{A}_i$ that
	\begin{align*}
		\frac{\partial\xi_i^{\theta_i}(a_i|s_i)}{\partial \theta_{i,s_i,a_i'}}=\xi_i^{\theta_i}(a_i|s_i)\mathds{1}\{a_i'=a_i\}-\xi_i^{\theta_i}(a_i|s_i)\xi_i^{\theta_i}(a_i'|s_i).
	\end{align*}
	The proof is complete.
\end{proof}

\begin{lemma}\label{lem:bound_grad_policy}
	For any softmax policy $\xi^\theta$, we have
	\begin{align*}
		\|\nabla_{\theta_i} \xi_i^{\theta_i}(a_i|s_i)\|\leq \sqrt 2\xi_i^{\theta_i}(a_i|s_i)\leq \sqrt 2,\quad \forall\;i,s_i,a_i,\theta_i. 
	\end{align*}
\end{lemma}

\begin{proof}[Proof of Lemma \ref{lem:bound_grad_policy}]
	By Lemma \ref{lem:policy_grad}, we have
	\begin{align*}
		\|\nabla_{\theta_i} \xi_i^{\theta_i}(a_i|s_i)\|
		=\; & \xi_i^{\theta_i}(a_i|s_i)\left[(1-\xi_i^{\theta_i}(a_i|s_i))^2+\sum_{a_i'\neq a_i} \xi_i^{\theta_i}(a_i'|s_i)^2\right]^{1/2} \\
		\leq\; & \xi_i^{\theta_i}(a_i|s_i)\left[(1-\xi_i^{\theta_i}(a_i|s_i))+\sum_{a_i'\neq a_i} \xi_i^{\theta_i}(a_i'|s_i)\right]^{1/2}  \\
		\leq \; & \sqrt{2}\xi_i^{\theta_i}(a_i|s_i).
	\end{align*}
\end{proof}

\begin{lemma}[Multi-Agent Policy Gradient Theorem] \label{lem:PGT_MARL_compact}
	It holds for all $i\in\mathcal{N}$ and $\theta\in\mathbb{R}^{|\mathcal{S}||\mathcal{A}|} $ that
	\begin{align*}
		\nabla_{\theta_i}J_i(\theta)=\frac{1}{1-\gamma}\sum_{s,a_i}d^{\theta}(s)\nabla_{\theta_i}\xi_i^{\theta_i}(a_i|s_i)\overline Q_{i}^{\theta}(s,a_i).
	\end{align*}
\end{lemma}
\begin{proof}[Proof of Lemma \ref{lem:PGT_MARL_compact}]
	Using the policy gradient theorem \cite[Theorem 1]{sutton1999PGT} and we have for any $i\in\mathcal{N}$ that
	\begin{align*}
		\nabla_{\theta_i}J_i(\theta) 
		=\;&\frac{1}{1-\gamma}\sum_{s,a}d^{\theta}(s)\nabla_{\theta_i}\xi^{\theta}(a|s) Q_{i}^{\theta}(s,a) \\
		=\;&\frac{1}{1-\gamma}\sum_{s}\sum_{a_i}\sum_{a_{-i}}d^{\theta}(s)\nabla_{\theta_i}\xi^{\theta}(a_i,a_{-i}|s) Q_{i}^{\theta}(s,a_i,a_{-i}) \\
		=\;&\frac{1}{1-\gamma}\sum_{s}\sum_{a_i}\sum_{a_{-i}}d^{\theta}(s)\nabla_{\theta_i}[\xi^{\theta_i}(a_i|s_i)\xi^{\theta_{-i}}(a_{-i}|s_{-i})] Q_{i}^{\theta}(s,a_i,a_{-i}) \\
		=\;&\frac{1}{1-\gamma}\sum_{s}\sum_{a_i}d^{\theta}(s)\nabla_{\theta_i}\xi^{\theta_i}(a_i|s_i)\sum_{a_{-i}}\xi^{\theta_{-i}}(a_{-i}|s_{-i}) Q_{i}^{\theta}(s,a_i,a_{-i}) \\
		=\;&\frac{1}{1-\gamma}\sum_{s,a_i}d^{\theta}(s)\nabla_{\theta_i}\xi_i^{\theta_i}(a_i|s_i)\overline Q_{i}^{\theta}(s,a_i),
	\end{align*}
	where the last line follows from the definition of the averaged $Q$-function.
\end{proof}

\begin{lemma}\label{lem:bound_grad}
	It holds for all $i\in\mathcal{N}$ and $\theta\in\mathbb{R}^{|\mathcal{S}||\mathcal{A}|} $ that
	\begin{align*}
		\left\|\nabla_{\theta_i}J_i(\theta)\right\|\leq \frac{\sqrt 2}{(1-\gamma)^2}.
	\end{align*}
\end{lemma}
\begin{proof}[Proof of Lemma \ref{lem:bound_grad}]  
	Using Lemma \ref{lem:bound_grad_policy} and Lemma \ref{lem:PGT_MARL_compact}, and we have
	\begin{align*}
		\left\|\nabla_{\theta_i}J_i(\theta)\right\|
		=\;& \frac{1}{1-\gamma}\norm{\sum_{s,a_i}d^{\theta}(s)\nabla_{\theta_i}\xi_i^{\theta_i}(a_i|s_i)\overline Q_{i}^{\theta}(s,a_i)} \\
		\leq \;& \frac{1}{1-\gamma}\sum_{s,a_i}d^\theta(s) \left|\overline{Q}_{i}^{\theta}(s,a_i)\right|\norm{\nabla_{\theta_i}\xi_i^{\theta_i}(a_i|s_i)} \\
		\leq \;& \frac{\sqrt{2}}{(1-\gamma)^2}\sum_{s,a_i}d^\theta(s) \xi_i^{\theta_i}(a_i|s_i) \\
		=\;& \frac{\sqrt 2}{(1-\gamma)^2}.
	\end{align*}
\end{proof}

\begin{lemma}\label{cor:softmax_v_grad}
	The following inequality holds for all $i$ and $\theta$:
	\begin{align*}
		\frac{\partial J_i(\theta)}{\partial \theta_{i,s_i',a_i'}}=\frac{1}{1-\gamma}\sum_{s_{-i}}d^{\theta}(s_i',s_{-i})\xi_i^{\theta_i}(a_i'|s_i')\overline A_{i}^{\theta}(s_i',s_{-i},a_i').
	\end{align*}
\end{lemma}
\begin{proof}[Proof of Lemma \ref{cor:softmax_v_grad}]
	Using Lemma \ref{lem:policy_grad} and Lemma \ref{lem:PGT_MARL_compact}, and we have
	\begin{align*}
		\frac{\partial J_i(\theta)}{\partial \theta_{i,s_i',a_i'}}
		=\;& \frac{1}{1-\gamma}\sum_{s,a_i}d^{\theta}(s)\frac{\partial \xi_i^{\theta_i}(a_i|s_i)}{\partial \theta_{i,s_i',a_i'}}\overline Q_{i}^{\theta}(s,a_i)\\
		=\;& \frac{1}{1-\gamma}\sum_{s,a_i}d^{\theta}(s)
		\xi_i^{\theta_i}(a_i|s_i)\mathds{1}\{s_i'=s_i\}\left(\mathds{1}\{a_i'=a_i\}-\xi_i^{\theta_i}(a_i'|s_i)\right)
		\overline Q_{i}^{\theta}(s,a_i)\\
		=\;& \frac{1}{1-\gamma}\sum_{s_{-i},a_i}d^{\theta}(s_i',s_{-i})
		\xi_i^{\theta_i}(a_i|s_i')\left(\mathds{1}\{a_i'=a_i\}-\xi_i^{\theta_i}(a_i'|s_i')\right)
		\overline Q_{i}^{\theta}(s_i',s_{-i},a_i)\\
		=\;& \frac{1}{1-\gamma}\sum_{s_{-i}}d^{\theta}(s_i',s_{-i})\xi_i^{\theta_i}(a_i'|s_i')
		\left( \overline Q_{i}^{\theta}(s_i',s_{-i},a_i')-\sum_{a_i}\xi_i^{\theta_i}(a_i|s_i')\overline Q_{i}^{\theta}(s_i',s_{-i},a_i)\right)
		\\
		=\;& \frac{1}{1-\gamma}\sum_{s_{-i}}d^{\theta}(s_i',s_{-i})\xi_i^{\theta_i}(a_i'|s_i')
		\left( \overline Q_{i}^{\theta}(s_i',s_{-i},a_i')-V_{i}^{\theta}(s_i',s_{-i})\right)
		\\
		=\;&\frac{1}{1-\gamma}\sum_{s_{-i}}d^{\theta}(s_i',s_{-i})\xi_i^{\theta_i}(a_i'|s_i')\overline A_{i}^{\theta}(s_i',s_{-i},a_i'),
	\end{align*}
	where the last line follows from the definition of the averaged advantage function.
\end{proof}

\begin{lemma}\label{lem:V_PDT}
	The following inequality holds for all $\xi,\xi'\in\Xi$, $i\in\mathcal{N}$, and $s\in\mathcal{S}$:
	\begin{align*}
		V_i^{\xi'}(s)-V_i^{\xi}(s)=\frac{1}{1-\gamma}\sum_{s',a}d_s^{\xi'}(s')(\xi'(a|s')-\xi(a|s'))Q_i^{\xi}(s',a). 
	\end{align*}    
\end{lemma}
\begin{proof}[Proof of Lemma \ref{lem:V_PDT}]
	Using the performance difference lemma in the single agent setting \citep[Lemma 2]{agarwal2021theory}, and we have
	\begin{align*}
		V_i^{\xi'}(s)-V_i^{\xi}(s)
		=\;&\frac{1}{1-\gamma}\sum_{s',a}d_s^{\xi'}(s')\xi'(a|s')A_i^{\xi}(s',a) \\
		=\;&\frac{1}{1-\gamma}\sum_{s',a}d_s^{\xi'}(s')\xi'(a|s')(Q_i^{\xi}(s',a)-V_i^{\xi}(s')) \\
		=\;&\frac{1}{1-\gamma}\left(\sum_{s',a}d_s^{\xi'}(s')\xi'(a|s')Q_i^{\xi}(s',a)-\sum_{s'}d_s^{\xi'}(s')V_i^{\xi}(s')\right) \\
		=\;&\frac{1}{1-\gamma}\left(\sum_{s',a}d_s^{\xi'}(s')\xi'(a|s')Q_i^{\xi}(s',a)-\sum_{s',a}d_s^{\xi'}(s)\xi(a|s')Q_i^{\xi}(s',a)\right) \\
		=\;&\frac{1}{1-\gamma}\sum_{s',a}d_s^{\xi'}(s')(\xi'(a|s')-\xi(a|s')) Q_i^{\xi}(s',a).    
	\end{align*}
\end{proof}

\begin{lemma}\label{lem:PDT}
	It holds for any $i\in\mathcal{N}$, $\theta=(\theta_i,\theta_{-i})$, and $\theta'=(\theta_i',\theta_{-i})$ that
	\begin{align*}
		J_i(\theta')-J_i(\theta) =\frac{1}{1-\gamma}\sum_{s,a_i}d^{\theta'}(s)(\xi_i^{\theta_i'}(a_i|s_i)-\xi_i^{\theta_i}(a_i|s_i))\overline Q_i^{\theta}(s,a_i).   
	\end{align*}
\end{lemma}
\begin{proof}[Proof of Lemma \ref{lem:PDT}]
	Using the performance difference lemma in the single agent setting \citep[Lemma 2]{agarwal2021theory}, and we have
	\begin{align*}
		J_i(\theta')-J_i(\theta) 
		=\;&\frac{1}{1-\gamma}\sum_{s,a}d^{\theta'}(s)\xi^{\theta'}(a|s)A_i^{\theta}(s,a) \\
		=\;&\frac{1}{1-\gamma}\sum_{s,a}d^{\theta'}(s)\xi^{\theta'}(a|s)(Q_i^{\theta}(s,a)-V_i^{\theta}(s)) \\
		=\;&\frac{1}{1-\gamma}\left(\sum_{s,a_i}\sum_{a_{-i}}d^{\theta'}(s)\xi_i^{\theta_i'}(a_i|s_i)\xi_{-i}^{\theta_{-i}}(a_{-i}|s_{-i})Q_i^{\theta}(s,a)-\sum_{s}d^{\theta'}(s)V_i^{\theta}(s)\right) \\
		=\;&\frac{1}{1-\gamma}\left(\sum_{s,a_i}d^{\theta'}(s)\xi_i^{\theta_i'}(a_i|s_i)\overline Q_i^{\theta}(s,a_i)-\sum_{s}d^{\theta'}(s)V_i^{\theta}(s)\right) \\
		=\;&\frac{1}{1-\gamma}\left(\sum_{s,a_i}d^{\theta'}(s)\xi_i^{\theta_i'}(a_i|s_i)\overline Q_i^{\theta}(s,a_i)-\sum_{s,a_i}d^{\theta'}(s)\xi_i^{\theta_i}(a_i|s_i)\overline Q_i^{\theta}(s,a_i)\right) \\
		=\;&\frac{1}{1-\gamma}\sum_{s,a_i}d^{\theta'}(s)(\xi_i^{\theta_i'}(a_i|s_i)-\xi_i^{\theta_i}(a_i|s_i))\overline Q_i^{\theta}(s,a_i).    
	\end{align*}
\end{proof}

\begin{lemma}\label{lem:prod_pol_diff_general}
	The following inequality holds for any $\sI\subseteq \sN$ and any $\xi_{\sI}, \xi'_{\sI}\in \Xi_{\sI}$:
	\begin{align*}
		\norm{\xi_{\sI}(\cdot|s_{\sI})-\xi'_{\sI}(\cdot|s_{\sI})}_1\leq \sum_{i\in \sI} \norm{\xi_{i}(\cdot|s_{i})-\xi'_{i}(\cdot|s_{i})}_1, \quad  \forall\; s_{\sI}\in\mathcal{S}_{\mathcal{I}}.
	\end{align*}
\end{lemma}
\begin{proof}[Proof of Lemma \ref{lem:prod_pol_diff_general}]
	The result follows by applying \cite[Lemma 3.4.3]{durrett2019probability}.
\end{proof}

\begin{lemma}[Property of NMPG]\label{lem:NMPG}
	In an NMPG, consider any $i\in \sN$ and any policy parameter $\theta$. Then the following equality holds for any $j\in \nikg$:
	\begin{align*}
		\nabla_{\theta_j} J_j(\theta)=\nabla_{\theta_j}\Phi_i(\theta).
	\end{align*}
\end{lemma}
\begin{proof}[Proof of Lemma \ref{lem:NMPG}]
	The proof essentially follows from \cite{leo2021convergeMPG}. Using the definition of NMPG (cf. Definition \ref{def:local_MPG}), we have for any $\theta_j,\theta_j'$, and $\theta_{-j}$ that
	\begin{align*}
		J_j(\theta_j',\theta_{-j})- \Phi_i(\theta_j',\theta_{-j})=J_j({\theta_j,\theta_{-j}}) -\Phi_i(\theta_j,\theta_{-j}).  
	\end{align*}
	Thus $J_j({\theta_j,\theta_{-j}}) -\Phi_i(\theta_j,\theta_{-j})$ is independent of $\theta_j$. Let $J_j({\theta_j,\theta_{-j}}) -\Phi_i(\theta_j,\theta_{-j})=U_j(\theta_{-j})$. Taking gradient with respect to $\theta_j$ on both sides, and we have 
	\begin{align*}
		\nabla_{\theta_j} J_j(\theta_j,\theta_{-j})=\nabla_{\theta_j}\Phi_i(\theta_j,\theta_{-j}).
	\end{align*}
\end{proof}

\end{document}